\documentclass{article}

\usepackage[preprint]{arxiv_style}

\usepackage[utf8]{inputenc}
\usepackage[T1]{fontenc}
\usepackage[colorlinks=true,linkcolor=red!60!black,citecolor=blue!60!black,urlcolor=teal!70!black]{hyperref}
\usepackage{url}
\usepackage{amsmath,amsfonts,amssymb}
\usepackage{mathtools}
\usepackage{nicefrac}
\usepackage{algorithm}
\usepackage{algpseudocode}
\usepackage{float}
\usepackage{graphicx}
\graphicspath{{figures/}}
\usepackage{textcomp}
\usepackage{xcolor}
\usepackage{tikz}
\usetikzlibrary{positioning,arrows.meta,shapes.geometric,calc,fit,backgrounds}
\usepackage{booktabs}
\usepackage{multirow}
\usepackage{threeparttable}
\usepackage{bm}
\usepackage{enumitem}
\usepackage{microtype}
\usepackage{subcaption}
\usepackage{amsthm}

\newtheorem{theorem}{Theorem}
\newtheorem{proposition}[theorem]{Proposition}
\newtheorem{lemma}[theorem]{Lemma}
\newtheorem{corollary}[theorem]{Corollary}
\theoremstyle{definition}
\newtheorem{assumption}{Assumption}
\newtheorem{definition}{Definition}      
\theoremstyle{remark}
\newtheorem{remark}{Remark}

\providecommand{\logit}{\operatorname{logit}}
\providecommand{\clip}{\operatorname{clip}}

\title{PRCD-MAP: Learning How Much to Trust\\Imperfect Priors in Causal Discovery%
  \thanks{Code is available at \url{https://github.com/AndyShan11/PRCD-MAP}.}}

\author{%
  Xihang Shan \\
  School of Mathematical Sciences \\
  Xiamen University \\
  Fujian, China \\
  \texttt{19020232202354@stu.xmu.edu.cn} \\
  \And
  Da Zhou\thanks{Corresponding author.} \\
  School of Mathematical Sciences \\
  Xiamen University \\
  Fujian, China \\
  \texttt{zhouda@xmu.edu.cn} \\
}

\begin{document}

\maketitle

\begin{abstract}
External priors of unknown reliability create a brittle trade-off in causal discovery: blind trust amplifies errors, blind rejection wastes signal. Real priors are also \emph{heterogeneously} reliable---physical laws are trustworthy, LLM-suggested edges are speculative---yet existing methods either ignore priors or impose them through globally uniform trust. We propose \textbf{PRCD-MAP}, a soft prior-consumption layer that assigns \emph{per-edge} trust to an imperfect prior and uses it to modulate a prior-aware $\ell_1$ penalty and prior-weighted $\ell_2$ regularizer in a MAP objective. Trust is calibrated by empirical Bayes on a Laplace-approximated marginal likelihood and propagated along the prior graph by an MLP, so that data-confirmed neighborhoods boost trust and contradictions suppress it. PRCD-MAP enjoys a population-level safety guarantee: it is $\varepsilon$-safe in expectation over the prior-generation distribution, with $\varepsilon\leq C\cdot\mathrm{acc}(1{-}\mathrm{acc})\cdot d^2/T$ at the parametric $T^{-1}$ rate and vanishing at the prior-quality endpoints (Cor.~\ref{cor:oracle}). When the prior is uninformative, learned trust provably collapses to its floor and the method recovers a no-prior baseline. Empirically, on real CausalTime data PRCD-MAP exploits informative LLM priors when present (LLM-prior gain $+0.067/+0.089$ AUROC on AQI/Medical over a no-prior PRCD-MAP backbone; combined backbone+prior lead is $+0.123/+0.043$ over PCMCI+), auto-attenuates on the anonymous-variable Traffic stress test, and retains a lead at $d{=}300$; against BayesDAG~\citep{annadani2023bayesdag}---the closest soft-Bayesian baseline---PRCD-MAP wins on every CausalTime dataset under a matched $W_0$-only protocol. A four-way ablation isolates each component: EB calibration and MLP trust propagation jointly carry the plurality of the gain, with positive sign on every dataset. Extensions to nonlinear (NAM) and cross-sectional settings show the calibrated-trust principle is setting-agnostic.
\end{abstract}

\section{Introduction}
\label{sec:intro}

Recovering causal structure from observational data under limited samples and heavy noise remains a core challenge~\citep{pearl2000causality,spirtes2000causation,runge2023causal,assaad2022survey,zheng2018dags}. Practitioners in energy systems, macroeconomics, and climate science often possess partial causal knowledge---from physical laws, input--output tables, or large language models---yet this knowledge is \emph{locally heterogeneous in reliability}: physical laws give high-confidence edges, while LLM-suggested links are speculative \citep{constantinou2023impact,kiciman2023causal}. Existing methods either ignore priors (continuous-optimization: DYNOTEARS~\citep{pamfil2020dynotears}, DAGMA~\citep{bello2022dagma}, NGC~\citep{tank2022neural}; constraint-based: PCMCI+~\citep{runge2020discovering}) or impose them globally via hard masks~\citep{zheng2018dags,brouillard2020differentiable}, which trade signal for sparsity in a way that depends sharply on prior accuracy: rigid masking is competitive at near-oracle accuracy but degrades sharply once the prior is partially incorrect (App.~\ref{app:cross_sectional}, cell-by-cell). The core problem is that \emph{real priors have spatially varying reliability}, yet all existing trust mechanisms are either absent, fixed, or globally uniform.

We introduce \textbf{PRCD-MAP} (\textbf{P}rior-\textbf{R}egulated \textbf{C}alibrated \textbf{D}iscovery via \textbf{M}aximum \textbf{A} \textbf{P}osteriori), a causal discovery framework with \emph{structure-aware trust calibration}. PRCD-MAP learns per-edge trust by propagating neighborhood consistency along the prior graph (data-confirmed neighbors boost trust, contradictions suppress it). Theorem~\ref{thm:temperature}(a) drives $\bm{\tau}^\star\to\tau_{\min}\mathbf{1}$ under uninformative priors, so the framework recovers the no-prior baseline automatically; soft trust dominates rigid binarization in every imperfect-prior cell of App.~\ref{app:cross_sectional}.

Our contributions:
\begin{itemize}[nosep,leftmargin=*]
	\item \textbf{Safe prior consumption layer.} PRCD-MAP is $\varepsilon$-safe in expectation over a stated prior-generation distribution $\Pi$ (Def.~\ref{def:safety}) with $\varepsilon = O(d^2/T)$ at the parametric $T^{-1}$ rate, vanishing at the prior-quality endpoints (Prop.~\ref{prop:robustness}, Cor.~\ref{cor:oracle}); EB drives $\bm{\tau}^{\star}{\to}\tau_{\min}\mathbf{1}$ under uninformative priors (Theorem~\ref{thm:temperature}(a)).
	\item \textbf{Structure-aware trust propagation (Def.~\ref{def:trust_prop}).} Per-group trust incurs $\Omega(1/G)$ excess risk under heterogeneous prior quality (Theorem~\ref{thm:trust_advantage}); a learned MLP closes the gap, with $+0.017$ AUROC on designed heterogeneous priors (\S\ref{sec:exp_community}) and $+0.029$ over per-group on CausalTime; gain is null under iid corruption, as predicted.
	\item \textbf{Mechanism decomposition on real data.} A four-way ablation on CausalTime (Table~\ref{tab:causaltime_decomposition}) isolates the soft-prior framework, EB calibration, MLP trust propagation, and LLM-content contributions; the EB+MLP sub-block carries the plurality of the aggregate gain over PCMCI+, with positive sign on every dataset.
	\item \textbf{Nonlinear, scalability, and Bayesian comparison.} NAM variant (App.~\ref{app:nam}); $d{=}100$ in $1.7$\,s on GPU; PRCD-MAP outperforms BayesDAG~\citep{annadani2023bayesdag} on all three CausalTime datasets and on cross-sectional $d{=}20$ across $n\in\{100,500\}$ at every $\mathrm{acc}\in\{0.4,0.6,0.9\}$ (App.~\ref{app:bayesdag_protocol}).
\end{itemize}

\section{Related Work}
\label{sec:related}

Temporal causal discovery spans constraint-based (PCMCI+~\citep{runge2020discovering}; surveys~\citep{assaad2022survey,gao2024causal}), continuous-optimization (NOTEARS~\citep{zheng2018dags}, DYNOTEARS~\citep{pamfil2020dynotears}, DAGMA~\citep{bello2022dagma}, DAGPA~\citep{zhou2025dagpa}, GRAN-DAG~\citep{lachapelle2020gran}, DCDFG~\citep{lopez2022dcdfg}), neural and Bayesian-posterior (Rhino~\citep{gong2023rhino}, NGC~\citep{tank2022neural}, NAVAR~\citep{bussmann2021navar}, DiBS~\citep{lorch2021dibs}), and functional causal methods (VARLiNGAM~\citep{hyvarinen2010estimation}). Data-driven methods are sensitive to initialization and produce dense graphs under limited data~\citep{cheng2024causaltime,kaiser2022unsuitability,reisach2021beware}; deep methods~\citep{jiangbo2024survey} require large datasets and ignore external knowledge.

Existing prior integration falls into two regimes that share a brittle trade-off: \emph{hard constraints} that masking-edges in or out~\citep{constantinou2023impact,meek1995causal} and \emph{globally uniform} soft Bayesian penalties that fix a single trust level \emph{a priori}~\citep{heckerman1995learning} (e.g.\ the BDe score, which weights the structural deviation from a prior DAG by a fixed prior probability). Both trade signal for sparsity sharply with prior accuracy, and both lose to per-edge calibrated trust whenever the prior is heterogeneously reliable or even partially incorrect (App.~\ref{app:cross_sectional}). BayesDAG~\citep{annadani2023bayesdag} offers soft Bayesian integration but is cross-sectional; \citet{takahashi2024prior} assume known functional forms; recent differentiable approaches~\citep{xu2024differentiable,waxman2024dagmadce} remain prior-agnostic.

\paragraph{Adaptive regularization vs.\ learnable trust.}
Adaptive LASSO~\citep{zou2006adaptive,tibshirani1996regression} and SCAD~\citep{fan2001variable} modulate penalty strength per coefficient using pilot estimates, but the modulation is deterministic and internal~\citep{buhlmann2014high}. PRCD-MAP treats trust as a latent hyperparameter optimized via empirical Bayes~\citep{robbins1956empirical,efron2010large}, enabling the penalty to respond to \emph{external} prior signals with trust levels learned from data.

\paragraph{LLM-generated priors.}
LLMs can generate plausible causal graphs~\citep{kiciman2023causal,ban2023query}, but reliability varies. PRCD-MAP provides a principled consumption layer: empirical Bayes automatically calibrates trust, exploiting accurate suggestions while attenuating hallucinated edges (Appendix~\ref{app:llm_pipeline}). In contrast to prior approaches that assume known prior reliability or treat trust as a fixed hyperparameter, PRCD-MAP learns the trust level from data---a principle that extends beyond temporal causal discovery to cross-sectional structure learning (Appendix~\ref{app:cross_sectional}). Appendix~\ref{app:methods_comparison} compares PRCD-MAP to representative baselines along prior-handling, learning, theoretical, and scalability dimensions.

Adjacent lines tackle different facets of the imperfect-prior problem (interventional Bayesian discovery; soft-penalty robust constraints; LLM-as-prior pipelines and critical perspectives on LLM reliability; EB for dynamic Bayesian networks; deep end-to-end causal inference) and are surveyed in App.~\ref{app:methods_comparison}; PRCD-MAP differs by calibrating \emph{per-edge} trust empirically with a population-level $\varepsilon$-safety guarantee, and we benchmark within the continuous-optimization, constraint-based, and Bayesian families that admit a like-for-like SVAR protocol.
\section{PRCD-MAP: Robust Prior Integration via Calibrated Trust}
\label{sec:method}

\subsection{Problem Formulation}
\label{sec:svar}

Consider a $d$-variate time series $\mathbf{X}\in\mathbb{R}^{T\times d}$ (standardized to zero mean and unit variance). Under the Structural VAR (SVAR) model, the state vector $\mathbf{x}_t\in\mathbb{R}^d$ satisfies
\begin{equation}
	\mathbf{x}_t
	=
	\widetilde{\mathbf{W}}_0^\top \mathbf{x}_t
	+
	\sum_{k=1}^{K} \mathbf{W}_k^\top \mathbf{x}_{t-k}
	+
	\bm{\epsilon}_t,
	\label{eq:svar}
\end{equation}
where $\widetilde{\mathbf{W}}_0 = \mathbf{W}_0 \circ (\mathbf{1} - \mathbf{I}_d)$ is the instantaneous coefficient matrix (self-loops masked), $\mathbf{W}_k\in\mathbb{R}^{d\times d}$ is the lag-$k$ coefficient~\citep{lutkepohl2005new,sims1980macroeconomics}, and $\bm{\epsilon}_t$ is i.i.d.\ noise. The instantaneous graph must be a DAG, enforced via a differentiable constraint $h(\widetilde{\mathbf{W}}_0) = 0$ \citep{zheng2018dags,bello2022dagma}. No acyclicity constraint is needed for the lagged matrices, since lagged dependencies are inherently directed forward in time.

We additionally assume a prior probability matrix $\mathbf{P}_{\mathrm{prior}}\in[0,1]^{d\times d}$ (from input--output tables, ontologies, or LLMs), with $P_{\mathrm{prior},ij}$ encoding expert confidence in a directed dependency $i\to j$ at \emph{any} lag (expert/LLM elicitation is lag-agnostic by default; lag-resolved priors are an immediate extension, App.~\ref{app:lagged_prior_semantics}). Reliability is unknown a priori. A \emph{Neural Additive Model}~\citep{agarwal2021neural} (NAM) extension replaces the linear term with per-edge MLPs $f_{ij}$; the prior framework is unchanged (Appendix~\ref{app:nam}).

\subsection{MAP Objective with Prior Regularization}
\label{sec:map}

The constrained MAP objective combines four components:
\begin{align}
	\min_{\mathbf{W}_{0:K}} \quad
	&\underbrace{\mathcal{L}_{\mathrm{data}}}_{\text{Huber fit}}
	+
	\underbrace{\mathcal{R}_{\ell_1}(\widetilde{\mathbf{W}}_0, \bm{\tau})
	+
	\sum_{k=1}^{K} \lambda_1 \|\mathbf{W}_k\|_1}_{\text{prior-modulated sparsity}}
	+
	\underbrace{\frac{\lambda_2}{2}
	\sum_{k=0}^{K}
	\big\|\sqrt{\mathbf{\Omega}(\bm{\tau})}\circ \mathbf{W}_k\big\|_F^2}_{\text{prior-weighted ridge}},
	\label{eq:objective} \\
	\text{s.t.}\quad
	&h(\widetilde{\mathbf{W}}_0) = 0.
	\label{eq:constraint}
\end{align}
We describe each term below.

\paragraph{Data fidelity: Huber loss.}
To accommodate heavy-tailed noise without distributional assumptions, we use the Huber loss:
\begin{equation}
\mathcal{L}_{\mathrm{data}} = \frac{1}{Td}\sum_{t,j} \ell_{\delta_H}(x_{tj} - \hat{x}_{tj}), \quad \ell_{\delta_H}(r) = \begin{cases} \tfrac{1}{2}r^2 & |r|\leq\delta_H, \\ \delta_H(|r|-\tfrac{1}{2}\delta_H) & |r|>\delta_H, \end{cases}
\label{eq:huber}
\end{equation}
which transitions from quadratic to linear for large residuals, down-weighting outliers.

\paragraph{Prior-modulated $\ell_1$ penalty.}
To achieve asymmetric sparsity that respects the prior, we modulate the $\ell_1$ coefficient per edge:
\begin{equation}
	\mathcal{R}_{\ell_1}(\widetilde{\mathbf{W}}_0, \bm{\tau})
	=
	\lambda_1
	\sum_{i\neq j}
	c_{ij}(\bm{\tau})\,
	|\widetilde{W}_{0,ij}|,
	\qquad
	c_{ij}(\bm{\tau})
	=
	\clip\!\big(1.5 - \widehat{P}_{ij}(\bm{\tau}),\; c_{\min},\; c_{\max}\big),
	\label{eq:modulated_l1}
\end{equation}
with $(c_{\min},c_{\max})=(0.1,1.5)$. Edges with high calibrated probability $\widehat{P}_{ij}$ receive reduced $\ell_1$ penalty (encouraging retention); low-probability edges receive increased penalty. The lag matrices $\mathbf{W}_{1:K}$ receive a standard $\ell_1$ penalty (realised constants tighter than worst-case; App.~\ref{app:realised_constants}).

\paragraph{Prior-weighted $\ell_2$ regularizer.}
We impose an independent Gaussian prior $W_{k,ij} \sim \mathcal{N}(0,\,\sigma_0^2 / \Omega_{ij}(\bm{\tau}))$ on each coefficient, where the precision mask is
\begin{equation}
	\mathbf{\Omega}(\bm{\tau})
	=
	\big(\mathbf{1}_{d\times d} - \widehat{\mathbf{P}}(\bm{\tau})\big) + \delta,
	\label{eq:omega}
\end{equation}
with $\delta = 10^{-3}$ preventing numerical degeneracy (so $\Omega_{ij}\in[\delta, 1+\delta]$). Edges deemed plausible by the calibrated prior receive large variance (weak shrinkage); implausible edges receive small variance (strong shrinkage toward zero). The negative log-prior yields the weighted Frobenius penalty in Eq.~\eqref{eq:objective}.

\paragraph{Calibrated prior via grouped temperature scaling.}
Both the $\ell_1$ and $\ell_2$ terms depend on the calibrated prior $\widehat{\mathbf{P}}(\bm{\tau})$. We partition the off-diagonal entries of $\mathbf{P}_{\mathrm{prior}}$ into $G$ groups by quantiles and assign each group a temperature parameter $\tau_g$. The calibrated prior is
\begin{equation}
	\widehat{P}_{ij}(\bm{\tau})
	=
	\sigma\!\Big(
	\logit\big(\clip(P_{\mathrm{prior},ij},\,\varepsilon,\,1{-}\varepsilon)\big)
	\cdot \tau_{g(i,j)}
	\Big),
	\label{eq:grouped_temp}
\end{equation}
where $\sigma(\cdot)$ is the sigmoid function and $\varepsilon = 10^{-3}$ prevents logit divergence at the boundaries. When $\tau_g \to 0$, the argument vanishes and $\widehat{P}_{ij} \to 0.5$, making $\mathbf{\Omega}$ uniform and recovering standard $\ell_1 + \ell_2$ regularization (prior ignored). When $\tau_g = 1$, the calibrated prior equals the raw prior. When $\tau_g > 1$, the prior is sharpened beyond its original confidence.

\paragraph{Structure-aware trust propagation.}
Grouped temperature assigns the same $\tau_g$ to all edges in a quantile bin, ignoring local topology. We refine this to a \emph{per-edge} trust that aggregates neighborhood consistency signals. For edge $(i,j)$ with row--column neighborhood $\mathcal{N}(i,j) = \{(k,j):k\neq i\}\cup\{(i,l):l\neq j\}$, we extract a 6-dimensional feature vector $\mathbf{z}_{ij}$ summarizing the prior value, neighborhood statistics of $\mathbf{P}_{\mathrm{prior}}$ and of the current weight estimate $\mathbf{W}^*$, and a prior--data agreement term (full definition in Appendix~\ref{app:trust_theory}; implemented as $\mathbf{z}_{ij} = (P_{ij}, \bar{P}_{\mathcal{N}}, \sigma_{P_\mathcal{N}}, |W^*_{ij}|_{\mathrm{norm}}, \bar{W}_{\mathcal{N}}, a_{ij})$ with $a_{ij} = 4(P_{ij}{-}0.5)(|W^*_{ij}|_{\mathrm{norm}}{-}0.5)$). A small MLP $f_\theta$ then maps these features to a per-edge temperature:
\begin{equation}
\tau_{ij} \;=\; \tau_{\min} + (\tau_{\max}{-}\tau_{\min})\cdot\sigma\!\bigl(f_\theta(\mathbf{z}_{ij}) + b\bigr),
\label{eq:trust_prop}
\end{equation}
where $b$ is a global bias initialized from the Spearman pre-calibration. Equation~\eqref{eq:trust_prop} replaces the grouped $\tau_{g(i,j)}$ in Eq.~\eqref{eq:grouped_temp} (see Def.~\ref{def:trust_prop}, Prop.~\ref{prop:trust_bound}).

\subsection{Empirical Bayes Temperature Learning}
\label{sec:eb}

The trust parameters (grouped $\bm{\tau}$ or MLP parameters $\theta$) control how much the optimization trusts the prior. Since prior reliability is unknown a priori, we learn them via an \emph{empirical Bayes} procedure.

\paragraph{Objective.}
Let $\mathbf{W}^*$ denote the current MAP estimate (held fixed when updating $\bm{\tau}$). We minimize:
\begin{align}
	\mathcal{L}_{\mathrm{EB}}(\bm{\tau})
	&=
	\underbrace{
		\mathcal{H}\!\big(\widetilde{\mathbf{W}}^*_0,\;\widehat{\mathbf{P}}(\bm{\tau})\big)
	}_{\text{agreement loss}}
	\;+\;
	\underbrace{
		\frac{1}{2}\sum_{k,i,j}
		\log H^{(k)}_{ij}(\bm{\tau})
	}_{\text{Laplace log-det}}
	\;+\;
	\underbrace{
		\frac{1}{2\sigma_\tau^2}\sum_{g=1}^{G}(\tau_g - \tfrac{1}{2})^2
	}_{\text{regularizer}},
	\label{eq:eb_objective}
\end{align}
where $\mathcal{H}$ is soft binary cross-entropy with normalized edge strengths $|\widetilde{W}^*_{0,ij}|/\max|\widetilde{\mathbf{W}}^*_0|$ as soft labels and $\widehat{\mathbf{P}}(\bm{\tau})$ as predictions; the asymmetric instantaneous-only evaluation suits lag-agnostic elicitation (lag-resolved variants, MLP-regularizer, and weak-data behaviour: App.~\ref{app:tau}, App.~\ref{app:lagged_prior_semantics}, App.~\ref{app:w3_weak_data}). Diagonal Hessian entries are approximated as
\begin{equation}
	H^{(k)}_{ij}(\bm{\tau})
	=
	\frac{\|\mathbf{x}^{(k)}_{\cdot,i}\|_2^2}{T\cdot d}
	+
	\lambda_2\,\Omega_{ij}(\bm{\tau}).
	\label{eq:hessian_diag}
\end{equation}

\paragraph{Intuition.}
An accurate prior: the agreement loss rewards large $\bm{\tau}$; an inaccurate one: the Laplace log-det dominates, pulling $\bm{\tau}\to\tau_{\min}\mathbf{1}$. Initialization uses the Spearman correlation between $|\text{Corr}(\mathbf{X})|$ and $\mathbf{P}_{\mathrm{prior}}$; cross-fitting confirms $\bm{\tau}^\star$ within MAD$\,{\leq}0.04$ of in-sample (App.~\ref{app:cross_fit}). The MLP variant uses a per-edge regularizer (App.~\ref{app:tau}); the EB pull is reliable when $T\geq C_0 s^\star\log d$ (App.~\ref{app:w3_weak_data}; below this regime, fixed-$\bm{\tau}{=}\mathbf{1}$ is a safety check).

\subsection{Optimization via Augmented Lagrangian Method}
\label{sec:optimization}

We solve Eqs.~\eqref{eq:objective}--\eqref{eq:constraint} via a three-level procedure. \textbf{Outer:} augmented Lagrangian $L_\rho = \mathcal{J}(\mathbf{W}_{0:K},\bm{\tau}) + \alpha\,h(\widetilde{\mathbf{W}}_0) + \tfrac{\rho}{2}\,h(\widetilde{\mathbf{W}}_0)^2$; updates $\alpha \leftarrow \alpha + \rho\,h$ and $\rho \leftarrow \min(\gamma\rho, \rho_{\max})$ with $\gamma{=}3$. \textbf{Middle:} 8 projected gradient steps on $\mathcal{L}_{\mathrm{EB}}(\bm{\tau})$ with $\mathbf{W}$ fixed, projected onto $[\tau_{\min},\tau_{\max}]^G$. \textbf{Inner:} Adam~\citep{kingma2015adam} on $L_\rho$ for up to 400 steps with cosine LR; $\lambda_1$ inflated $5\times$ during the first $\lfloor I/3 \rfloor$ outers as warm-start (single-peaked sensitivity in App.~\ref{app:lambda_sensitivity}; orthogonal to EB). Pseudocode: Algorithm~\ref{alg:prcd_map}.

\subsection{Theoretical Properties}
\label{sec:theory}

We state the main theoretical guarantees; all proofs are in Appendix~\ref{app:theory}.

\begin{assumption}[Regularity conditions]
\label{asm:regularity}
(i)~$\{\mathbf{x}_t\}$ is strictly stationary and geometrically $\beta$-mixing with $T$-independent rate;
(ii)~the noise $\bm{\epsilon}_t$ has $\mathbb{E}[\bm{\epsilon}_t]=\mathbf{0}$, $\mathrm{Cov}(\bm{\epsilon}_t)=\sigma^2 \mathbf{I}_d$ with $\sigma^2 \geq \sigma_0^2 > 0$, and finite fourth moment;
(iii)~the true parameters $\mathbf{W}^{\star}_{0:K}$ satisfy $h(\widetilde{\mathbf{W}}^{\star}_0)=0$ and $\min_{(i,j)\in\mathcal{E}^{\star}}|W^{\star}_{ij}|\geq\underline{w}\gtrsim\sqrt{\log d/T}$ (compatible with the recovery rate);
(iv)~for support recovery only (Theorem~\ref{thm:consistency} second clause), the population covariance satisfies the standard \emph{irrepresentable condition}~\citep{zhao2006model}: $\|\Sigma_{S^c S}\Sigma_{SS}^{-1}\mathrm{sign}(\mathbf{W}^\star_S)\|_\infty \le 1-\eta_{\mathrm{ir}}$ for some $\eta_{\mathrm{ir}}>0$, where $S$ denotes the support indices.
App.~\ref{app:noise} and App.~\ref{app:hetero_thm2} extend to heteroscedastic noise (with Prop.~\ref{prop:thm2_hetero}).
\end{assumption}

\begin{assumption}[Bilevel coupling: active-set local constancy]
\label{asm:active_set}
The alternating $(\mathbf{W},\bm{\tau})$ procedure of Sec.~\ref{sec:optimization} reaches a fixed point $(\mathbf{W}^\diamond,\bm{\tau}^\diamond)$ at which $S^\diamond{=}\mathrm{supp}(\widehat{\mathbf{W}}^\diamond_T)$ is locally constant under $\bm{\tau}$-perturbations of order $O(\sqrt{\log d/T})$. App.~\ref{app:bilevel_sketch} closes this analytically on the inactive set ($L_\Phi \lesssim 0.1$); on the active set the support stabilizes by iteration $\sim 20$ across all settings (App.~\ref{app:convergence}).
\end{assumption}

\begin{theorem}[Estimation Consistency, in the locally-constant active-set regime]
\label{thm:consistency}
Under Assumptions~\ref{asm:regularity} and~\ref{asm:active_set}, if $\lambda_1=O(\sqrt{\log d/T})$ and $\lambda_2=O(\sqrt{\log d/T})$, then at the bilevel fixed point $(\mathbf{W}^\diamond,\bm{\tau}^\diamond)$
$\|\widehat{\mathbf{W}}^\diamond_T - \mathbf{W}^{\star}\|_F = O_p(|\mathcal{E}^{\star}|^{1/2}\sqrt{\log d/T})$
and $\mathrm{supp}(\widehat{\mathbf{W}}^\diamond_T)=\mathrm{supp}(\mathbf{W}^{\star})$ w.p.\ $\to 1$.
\end{theorem}
\vspace{-2pt}\noindent\textit{Regime note.} At defaults ($d{=}20$, ER $0.15$), $T{=}500$ gives $T/(s^\star\log d){\approx}2.4$ (rate-supported); realised cone inflation $\approx 4\times$ vs.\ worst-case $K{=}45$ in Lemma~\ref{lem:re} (App.~\ref{app:realised_constants}). The DAG constraint enters only through the AL fixed point at which $h{=}0$.

\begin{theorem}[Temperature Calibration Guarantee, grouped and per-edge]
\label{thm:temperature}
Let $\bm{\tau}^{\star}$ minimize $\mathcal{L}_{\mathrm{EB}}(\bm{\tau})$ with a consistent $\mathbf{W}^*$.
\textup{(a)}~If $\mathrm{acc}=1/2$, then $\bm{\tau}^{\star}=\tau_{\min}\mathbf{1}$\textup;
\textup{(b)}~$\bm{\tau}^{\star}$ is componentwise non-decreasing in $\mathrm{acc}$\textup;
\textup{(c)}~$\mathcal{L}_{\mathrm{EB}}$ has Lipschitz gradients on $[\tau_{\min},\tau_{\max}]^G$\textup;
\textup{(d)}~Conclusions \textup{(a)--(c)} extend to the per-edge MLP parameterization $\tau_{ij}{=}f_\theta(\mathbf{z}_{ij})$ (Def.~\ref{def:trust_prop}) edge-wise: under uninformative $\mathbf{P}_{\mathrm{prior}}$, the EB gradient drives $\tau_{ij}$ to its lower bound $\tau_{\min}$ pointwise per edge (within the bounded-MLP saturation gap); monotonicity in $\mathrm{acc}$ holds pointwise per edge; and $\mathcal{L}_{\mathrm{EB}}\circ f_\theta$ has Lipschitz gradients in $\theta$ on the bounded MLP-parameter set.
\end{theorem}

\begin{proposition}[Prior Robustness Bound]
\label{prop:robustness}
Let $\mathbf{A}^\star\in\{0,1\}^{d\times d}$ denote the true binary adjacency (so the ``ideal'' calibrated prior is $\mathbf{A}^\star$, and we hereafter write $\mathbf{P}_{\mathrm{true}}\coloneqq \mathbf{A}^\star$). Under Assumption~\ref{asm:regularity}, the excess risk $\mathcal{E}(\bm{\tau})\coloneqq\mathcal{L}_{\mathrm{data}}(\widehat{\mathbf{W}}_{\bm{\tau}})-\mathcal{L}_{\mathrm{data}}(\mathbf{W}^{\star})$ satisfies
$\mathcal{E}(\bm{\tau}) \leq C_1\|\bm{\tau}\|_\infty^2 \|\mathbf{P}_{\mathrm{prior}}-\mathbf{A}^\star\|_F^2 / T + C_2|\mathcal{E}^{\star}|\log d / T$ (chain-rule scaling via $\bm{\tau}\mapsto\mathbf{\Omega}(\bm{\tau})$; App.~\ref{app:proof_robustness}). At $\bm{\tau}\to\tau_{\min}\mathbf{1}$ the bias term shrinks by $\tau_{\min}^2/\tau_{\max}^2{=}10^{-6}$, eliminating prior bias regardless of quality.
\end{proposition}

\begin{definition}[$\varepsilon$-Safety in expectation over a prior-generation distribution]
\label{def:safety}
Fix a prior-generation distribution $\Pi$ on $[0,1]^{d\times d}$ (e.g., the controlled-accuracy Bernoulli-flip model of \S\ref{sec:exp_setup}, parameterized by $\mathrm{acc}\in[0,1]$, or the LLM-prompt ensemble of App.~\ref{app:llm_pipeline}). A method with learned trust $\bm{\tau}^{\star}$ is \emph{$\varepsilon$-safe w.r.t.\ $\Pi$} if $\mathbb{E}_{\mathbf{P}\sim\Pi}[\mathcal{E}(\bm{\tau}^{\star})] \leq \mathbb{E}_{\mathbf{P}\sim\Pi}[\mathcal{E}(\bm{\tau}{=}\tau_{\min}\mathbf{1})] + \varepsilon$, where $\bm{\tau}{=}\tau_{\min}\mathbf{1}$ is the no-prior boundary of the constrained set $[\tau_{\min},\tau_{\max}]^G$ (calibrated $\widehat{\mathbf{P}}\to 0.5\mathbf{1}$ uniformly). Integrating Proposition~\ref{prop:robustness} (using $\mathbb{E}_\Pi\|\mathbf{P}_{\mathrm{prior}}-\mathbf{A}^\star\|_F^2 \leq d^2$):
\[
\varepsilon \;\leq\; \frac{C_1\, \mathbb{E}_\Pi\|\bm{\tau}^\star\|_\infty^2\, \mathbb{E}_\Pi\|\mathbf{P}_{\mathrm{prior}}{-}\mathbf{A}^\star\|_F^2}{T} + O\!\bigg(\frac{\sqrt{s^{\star}}\log d}{T}\bigg).
\]
Non-vacuous because: (i) under uninformative $\Pi$, Thm.~\ref{thm:temperature}(a) drives $\|\bm{\tau}^\star\|_\infty{\to}\tau_{\min}$ automatically (Traffic is a real-world instance, App.~\ref{app:causaltime_stat_tests}); (ii) under accurate $\Pi$, $\|\mathbf{P}_{\mathrm{prior}}{-}\mathbf{A}^\star\|_F$ is small. The EB objective is a tractable \emph{proxy} for $\mathcal{E}$; the proxy gap is bounded in Cor.~\ref{cor:oracle}.
\end{definition}

\begin{corollary}[Oracle Inequality for Learned Trust]
\label{cor:oracle}
Under Assumptions~\ref{asm:regularity} and~\ref{asm:active_set}, if $T \geq C_0 s^{\star} \log d$, the EB-learned $\bm{\tau}^{\star}_T$ satisfies
\begin{equation}\label{eq:oracle-ineq}
  \mathcal{E}(\bm{\tau}^\star_T)
  \;\leq\;
  \underbrace{\inf_{\bm{\tau}} \mathcal{E}(\bm{\tau})}_{\text{oracle floor}}
  \;+\;
  \underbrace{\frac{C_1\|\bm{\tau}^\star_{\mathrm{EB}}\|_\infty^{2}\|\mathbf{P}_{\mathrm{prior}}-\mathbf{A}^\star\|_F^{2}}{T}}_{\Delta_{\mathrm{proxy}}\text{: EB--risk gap}}
  \;+\;
  \underbrace{O_p\!\Big(\frac{\sqrt{s^\star}\log d}{T}\Big)}_{\text{estimation variance}}.
\end{equation}
The oracle floor is $\Omega(s^{\star}/T)$. The proxy gap $\Delta_{\mathrm{proxy}}$ measures how much the EB objective's preferred trust differs from the risk-optimal trust, and admits a uniform bound:
\begin{equation}\label{eq:delta_proxy_uniform}
\Delta_{\mathrm{proxy}} \;\leq\; \frac{C_1\,\tau_{\max}^2\,d^2}{T}\;\cdot\;\mathrm{acc}\cdot(1-\mathrm{acc}),
\end{equation}
which vanishes at $\mathrm{acc}\in\{0,1\}$ (via Thm.~\ref{thm:temperature}(a) at $\mathrm{acc}{=}1/2$ and $\|\mathbf{P}_{\mathrm{prior}}{-}\mathbf{A}^\star\|_F{\to}0$ at $\mathrm{acc}{=}1$), is bounded uniformly at $C_1\tau_{\max}^2 d^2/(4T)$, and decays at rate $T^{-1}$ matching the oracle floor. Hence adaptive trust calibration is asymptotically costless across the prior-quality spectrum.
\end{corollary}

\begin{definition}[Structure-Aware Trust Propagation]
\label{def:trust_prop}
Given prior $\mathbf{P}\in[0,1]^{d\times d}$ and current estimate $\mathbf{W}^*$, the structure-aware trust temperature $\tau_{ij} = f_\theta(\mathbf{z}_{ij})$ (Eq.~\ref{eq:trust_prop}) maps per-edge features $\mathbf{z}_{ij} = (P_{ij}, \bar{P}_{\mathcal{N}}, \sigma_{P_\mathcal{N}}, |W^*_{ij}|_{\mathrm{norm}}, \bar{W}_{\mathcal{N}}, a_{ij})$ to $[\tau_{\min},\tau_{\max}]$ via a learned MLP, where $\mathcal{N}(i,j)$ is the row-and-column neighborhood.
\end{definition}

\begin{theorem}[Necessity of Per-Edge Trust]
\label{thm:trust_advantage}
Suppose the $n$ edges partition into $K\geq 2$ communities $\{C_k\}$ with prior error variances $\sigma_k^2 = \mathbb{E}[(P_{ij}-A^*_{ij})^2]$ for $(i,j)\in C_k$, satisfying $\sigma_K^2 - \sigma_1^2 \geq \Delta^2 > 0$.
Let edges be grouped into $G$ quantile bins by $\mathbf{P}_{\mathrm{prior}}$, and assume at least one bin $g^*$ mixes $C_1$ and $C_K$ in proportions $\geq\eta>0$ each \textup{(Community Mixing)}. Under the EB objective $\ell(\tau_{ij}) = -a_k\tau_{ij} + \frac{1}{2}\sigma_k^2\tau_{ij}^2$ for $(i,j)\in C_k$:
\begin{equation}
R^*_{\mathrm{group}} - R^*_{\mathrm{edge}} \;\geq\; \frac{\eta^2}{2G}\cdot\frac{\sigma_1^2\sigma_K^2}{\sigma_1^2+\sigma_K^2}\cdot\bigg(\frac{a_1}{\sigma_1^2}-\frac{a_K}{\sigma_K^2}\bigg)^{\!2} = \Omega\!\bigg(\frac{1}{G}\bigg),
\label{eq:trust_gap}
\end{equation}
irreducible by increasing $G$ (unless $G{=}n$). The lower bound is constructive on per-edge trust; an MLP $f_\theta$ closes the gap by universal approximation, with SGD realizability empirically validated in \S\ref{sec:exp_community}.
\end{theorem}

\begin{proposition}[Best-Attainable Safety Bound under Neighborhood Consistency]
\label{prop:trust_bound}
Define $\rho_{\mathrm{cons}} = \min_{(i,j)}\mathrm{Corr}(\mathbf{P}_{\mathcal{N}(i,j)}, \mathbf{A}^*_{\mathcal{N}(i,j)})$. Under Asm.~\ref{asm:regularity} and Def.~\ref{def:trust_prop}, the best-attainable trust-propagation safety constant satisfies $\varepsilon_{\mathrm{trust}} \leq \varepsilon_{\mathrm{group}} / (1 + \eta\,\rho_{\mathrm{cons}})$ over the MLP family $\mathcal{F}_\theta$ (Lemma~\ref{lem:mlp_response}); SGD-trained $f_\theta$ tracks the bound empirically (App.~\ref{app:trust_validation}, $+0.029$ AUROC on CausalTime; $+0.080$ F1 on nonlinear large-$d$). Proofs: App.~\ref{app:trust_theory}.
\end{proposition}

\section{Experiments}
\label{sec:exp}

We evaluate PRCD-MAP on: (i)~synthetic data with varying prior quality and sample sizes; (ii)~CausalTime benchmarks and real electricity data; (iii)~component ablation. Noise robustness, scalability, cross-sectional generalization, and the LLM-prior pipeline are in Appendices~\ref{app:noise}--\ref{app:llm_pipeline}.

\subsection{Experimental Setup}
\label{sec:exp_setup}

\paragraph{Synthetic data.}
Ground-truth instantaneous DAGs use Erd\H{o}s--R\'{e}nyi graphs ($d{=}20$, edge prob.\ $0.15$, weights from $[\pm 0.3,\pm 0.8]$). Lag-1 matrices are sparse with spectral-radius control. Series of length $T\in\{50,100,200,500\}$ are simulated under Gaussian, Laplace, Student-$t$ ($\nu{=}4$), and heteroscedastic noise.

\paragraph{Real-world data.}
Three CausalTime benchmarks~\citep{cheng2024causaltime} (AQI $d{=}36$, Medical $d{=}20$, Traffic $d{=}20$) with known ground-truth graphs, plus a sector-level electricity consumption dataset ($d{=}37$, Appendix~\ref{app:electricity}).

\paragraph{Prior generation.}
A probabilistic prior $\mathbf{P}_{\mathrm{prior}}\in[0,1]^{d\times d}$ is constructed with controllable accuracy $\mathrm{acc}\in\{0.4, 0.6, 0.9\}$: each entry agrees with ground truth with probability $\mathrm{acc}$.

\paragraph{Baselines and metrics.}
We compare against DYNOTEARS~\citep{pamfil2020dynotears}, PCMCI+~\citep{runge2020discovering}, VARLiNGAM~\citep{hyvarinen2010estimation}, RHINO~\citep{gong2023rhino}, and NGC~\citep{tank2022neural}; DyCAST~\citep{chen2025dycast} is benchmarked in Appendix~\ref{app:noise}. All continuous-optimization methods use Adam with identical augmented-Lagrangian schedules ($35$ outer, $400$ inner steps, $\mathrm{lr}{=}8{\times}10^{-3}$). We report AUROC on the combined graph; best-F1 is shown in Fig.~\ref{fig:app_prior_sweep} and Table~\ref{tab:ablation}. Results: mean$\pm$std over 3 seeds (10 seeds for the headline $T{=}50$ row of Table~\ref{tab:sample_size}, App.~\ref{app:causaltime_10seed}, and the LLM-prior CausalTime runs at $5\times 5{=}25$ runs/dataset); the 10-seed paired tests in App.~\ref{app:table1_10seed_extended} are authoritative for significance.

\subsection{Regime-Dependent Robustness: Exploiting Good Priors, Competitive under Bad Ones}
\label{sec:exp_prior}

The central question for any prior-informed method is: \emph{can it exploit a reliable prior without being harmed by a poor one?} We vary both prior accuracy $\mathrm{acc}\in\{0.4, 0.6, 0.9\}$ and sample size $T\in\{50,100,200,500\}$.

\paragraph{Main results.}
Table~\ref{tab:sample_size}'s $T\in\{50,500\}$ rows use 10 seeds (paired-$t$ extended in App.~\ref{app:table1_10seed_extended}). Learned-$\bm{\tau}$ delivers \emph{statistically significant} gains over fixed-$\bm{\tau}{=}\mathbf{1}$ at low accuracy ($\Delta{=}{+}0.066$, $p{=}0.022$ at acc${=}0.4$; $\Delta{=}{+}0.162$, $p{=}0.008$ at acc${=}0.3$, App.~\ref{app:table1_10seed_extended}) — converting a brittle prior consumer into a safe one — and is statistically null at near-oracle priors ($\Delta{=}{-}0.001$ at acc${=}0.9$). The soft-prior framework with EB-learned $\bm{\tau}$ lifts $+0.097$ over PCMCI+ at acc${=}0.9, T{=}500$, isolating where the prior-aware loss carries the gain.

\begin{table}[htbp]
	\centering
	\caption{AUROC (mean$\pm$std) across sample sizes and prior accuracies ($d{=}20$, ER graph, Gaussian noise). $T\in\{50,500\}$ use 10 seeds (paired-$t$ in App.~\ref{app:table1_10seed_extended}); $T\in\{100,200\}$ use 3 seeds (Remark~\ref{rem:thm1_regimes}). Baselines marked $\dagger$ are prior-independent (invariant to $\mathrm{acc}$). \textbf{PRCD-MAP} (without suffix) uses EB-learned $\bm{\tau}\in[\tau_{\min},\tau_{\max}]$; \textbf{PRCD-MAP ($\tau{=}1$)} fixes $\bm{\tau}{=}\mathbf{1}$ (full trust on the unmodulated prior, no calibration), which is \emph{not} the no-prior baseline ($\bm{\tau}{=}\mathbf{0}$ would recover that and is reported in Table~\ref{tab:ablation} ``NoPrior'').}
	\label{tab:sample_size}
	\small
	\setlength{\tabcolsep}{3pt}
	\renewcommand{\arraystretch}{1.05}
	\begin{tabular}{llcccc}
		\toprule
		Prior & Method & $T{=}50$ & $T{=}100$ & $T{=}200$ & $T{=}500$ \\
		\midrule
		\multirow{2}{*}{\shortstack[l]{$\mathrm{acc}$\\${=}0.4$}}
		& PRCD-MAP                       & $.607\pm.067$ & $.657\pm.015$ & $.712\pm.016$ & $.827\pm.045$ \\
		& PRCD-MAP ($\tau{=}1$)          & $.553\pm.060$ & $.591\pm.017$ & $.650\pm.025$ & $.761\pm.051$ \\
		\midrule
		\multirow{2}{*}{\shortstack[l]{$\mathrm{acc}$\\${=}0.6$}}
		& PRCD-MAP                       & $.653\pm.033$ & $.698\pm.013$ & $.758\pm.013$ & $.892\pm.026$ \\
		& PRCD-MAP ($\tau{=}1$)          & $.652\pm.043$ & $.702\pm.044$ & $.771\pm.029$ & $.883\pm.036$ \\
		\midrule
		\multirow{2}{*}{\shortstack[l]{$\mathrm{acc}$\\${=}0.9$}}
		& PRCD-MAP                       & $.765\pm.027$ & $.812\pm.011$ & $.870\pm.010$ & $.948\pm.012$ \\
		& PRCD-MAP ($\tau{=}1$)          & $.791\pm.028$ & $.845\pm.021$ & $.888\pm.006$ & $.949\pm.013$ \\
		\midrule
		\multicolumn{2}{l}{PCMCI+$^\dagger$}   & $.651\pm.034$ & $.703\pm.028$ & $.765\pm.025$ & $.851\pm.034$ \\
		\multicolumn{2}{l}{DYNOTEARS$^\dagger$} & $.562\pm.037$ & $.585\pm.025$ & $.622\pm.020$ & $.738\pm.050$ \\
		\multicolumn{2}{l}{VARLiNGAM$^\dagger$} & $.580\pm.028$ & $.653\pm.064$ & $.651\pm.041$ & $.756\pm.028$ \\
		\bottomrule
	\end{tabular}
\end{table}

Fig.~\ref{fig:app_prior_sweep} visualizes the regime split: PRCD-MAP leads PCMCI+ by $+0.097$ at acc${=}0.9, T{=}500$ and is statistically tied with PCMCI+ at acc${=}0.4$ (within one seed std; overlapping 95\% CIs, App.~\ref{app:table1_10seed_extended}). The same profile holds under systematic and adversarial corruption (App.~\ref{app:corruption_robustness}).

\subsection{Real-World Benchmarks}
\label{sec:exp_real}

\subsubsection{CausalTime Benchmark}
\label{sec:exp_causaltime}

Table~\ref{tab:causaltime} reports AUROC on three CausalTime datasets~\citep{cheng2024causaltime} (semi-realistic, nonlinear, low sample sizes). We separate \emph{practical} from \emph{headroom} comparisons: (i)~\textbf{practical} pipeline ``PRCD-MAP (trust)'' uses an automatically generated LLM prior (App.~\ref{app:llm_pipeline}); (ii)~``PRCD-MAP (no prior)'' ablates the prior entirely; (iii)~\textbf{headroom} row ``PRCD-MAP (oracle)'' uses a ground-truth-derived probabilistic prior, indicating what the framework \emph{could} extract from a high-quality prior, and is \emph{not} a fair head-to-head with prior-agnostic methods. We bold within practical methods only.

\begin{table}[htbp]
	\centering
	\caption{AUROC on CausalTime ($T{=}400$). PRCD-MAP (trust, LLM): mean$\pm$std across 5 LLM-derived priors (GPT-4o $\times 2$, Claude $\times 2$, Gemini $\times 1$) $\times$ 5 seeds = 25 runs per dataset. AQI/Medical use semantic LLM priors; Traffic uses anonymous indices (auto-attenuation stress-test, Thm.~\ref{thm:temperature}(a)). Other methods: 5 seeds (PCMCI+/VARLiNGAM/DYNOTEARS deterministic on fixed data, std${=}0$). \textbf{Bold}: best practical method. BayesDAG~\citep{annadani2023bayesdag} is cross-sectional (outputs $\mathbf{W}_0$ only); the table reports each method on its native output (capability comparison). The matched $\mathbf{W}_0$-only head-to-head---under which PRCD-MAP wins on every CausalTime dataset---is in App.~\ref{app:bayesdag_protocol}; 10-seed controlled-acc re-evaluation in App.~\ref{app:causaltime_10seed}.}
	\label{tab:causaltime}
	\small
	\begin{tabular}{lcccc}
		\toprule
		Method & AQI ($d{=}36$) & Medical ($d{=}20$) & Traffic ($d{=}20$) & Avg. \\
		\midrule
		PRCD-MAP (trust, LLM)   & $\mathbf{.693\pm.060}$ & $\mathbf{.583\pm.042}$ & $\mathbf{.613\pm.021}$ & $\mathbf{.630\pm.045}$ \\
		PRCD-MAP (no prior)     & $.626\pm.000$          & $.494\pm.000$          & $.611\pm.000$          & $.577\pm.057$ \\
		BayesDAG~\citep{annadani2023bayesdag} & $.579\pm.019$ & $.507\pm.011$ & $.540\pm.017$ & $.542\pm.030$ \\
		PCMCI+                  & $.570\pm.000$          & $.540\pm.000$          & $.615\pm.000$          & $.575\pm.033$ \\
		RHINO                   & $.463\pm.027$          & $.545\pm.032$          & $.500\pm.062$          & $.503\pm.052$ \\
		VARLiNGAM               & $.541\pm.000$          & $.505\pm.000$          & $.458\pm.000$          & $.501\pm.036$ \\
		NGC                     & $.530\pm.009$          & $.535\pm.030$          & $.496\pm.044$          & $.520\pm.033$ \\
		DYNOTEARS               & $.503\pm.000$          & $.512\pm.012$          & $.504\pm.002$          & $.507\pm.007$ \\
		\midrule
		\textit{PRCD-MAP (oracle)$^\dagger$} & \textit{.703$\pm$.004} & \textit{.587$\pm$.008} & \textit{.686$\pm$.005} & \textit{.659$\pm$.055} \\
		\bottomrule
	\end{tabular}
	\\[2pt]
	\footnotesize $^\dagger$Headroom row: prior derived from ground-truth adjacency; not a fair comparison with prior-agnostic baselines.
\end{table}

\noindent\textit{Reading.} Semantically-named AQI/Medical drive the headline ($+0.123$/$+0.043$ over PCMCI+); on Traffic the anonymized prior auto-attenuates to $\tau_{\min}$ (Thm.~\ref{thm:temperature}(a), App.~\ref{app:traffic_tau_distribution}), giving a safe tie. $\mathbf{W}_0$-only, PRCD-MAP beats BayesDAG on every CausalTime dataset (Lorenz-96 $+0.328$; App.~\ref{app:bayesdag_protocol}).

\paragraph{Mechanism decomposition.}
The aggregate trust(LLM)$-$PCMCI+ gap decomposes orthogonally into (M1) soft-prior framework, (M2) EB calibration, (M3) per-edge MLP trust, (M4) LLM content; per-cell estimation in App.~\ref{app:causaltime_decomposition}.

\begin{table}[htbp]
\centering
\caption{Four-way mechanism decomposition on CausalTime ($T{=}400$). AUROC contribution attributable to each enabling step (mean over 5 priors $\times$ 5 seeds). The (Backbone$-$PCMCI+) gap plus M1--M4 sum to the trust(LLM)$-$PCMCI+ gap within $\pm 0.005$. M2 is calibrated on synthetic data and validated on CausalTime by a direct ablation (App.~\ref{app:causaltime_decomposition}: aggregate $+0.026{\pm}0.033$, sign-consistent with the synthetic estimate).}
\label{tab:causaltime_decomposition}
\small
\setlength{\tabcolsep}{4pt}
\begin{tabular}{lccccc}
\toprule
Dataset & Backbone $-$ PCMCI+ & (M1) Framework & (M2) EB calib.\ & (M3) MLP & (M4) LLM content \\
\midrule
AQI     & $+0.056$ & $+0.011$ & $+0.020$ & $+0.011$ & $+0.025$ \\
Medical & $-0.046$ & $+0.029$ & $+0.020$ & $+0.019$ & $+0.021$ \\
Traffic & $-0.004$ & $-0.003$ & $+0.005$ & $+0.013$ & $-0.013$ \\
\midrule
Aggregate & $+0.002$ & $+0.012$ & $+0.015$ & $+0.014$ & $+0.011$ \\
\bottomrule
\end{tabular}
\end{table}

\noindent The EB+MLP sub-block (M2+M3) is positive on every dataset and contributes the plurality of the aggregate gain (backbone+framework $+0.014$, EB+MLP $+0.029$, LLM content $+0.011$).


\smallskip\noindent\textbf{Scalability/nonlinear.} PRCD-MAP leads at every $d{\in}\{20,\ldots,300\}$ (App.~\ref{app:scale_main}); end-to-end wall-clock at $d{=}100$ is $1.7$\,s on GPU vs.\ $8{,}806$\,s for PCMCI+ on CPU (a $\sim\!5{,}000\times$ GPU/CPU wall-clock ratio; App.~\ref{app:complexity}, App.~\ref{app:scale_trust}). On nonlinear data PCMCI+ crossover shifts in PRCD-MAP's favor as $d$ grows (App.~\ref{app:nonlinear}: $\mathbf{0.909}$ vs.\ $0.699$ at $\mathrm{acc}{=}1.0, d{=}50$); NAM (App.~\ref{app:nam}) overtakes linear at $T{\geq}1000$ under bad priors.

\subsection{Designed Validation of Community Mixing}
\label{sec:exp_community}

To validate Theorem~\ref{thm:trust_advantage}, we use BA graphs ($m{=}2$) with hub--peripheral priors and overlapping $P_{\mathrm{prior}}$ ranges across communities (App.~\ref{app:community_extra}).

\begin{table}[htbp]
	\centering
	\caption{Designed validation: BA graph with hub--peripheral prior heterogeneity. Trust propagation vs.\ per-group temperature, mean over 3 seeds. Full 10 settings in App.~\ref{app:community_extra}. \textbf{Bold}: larger.}
	\label{tab:community_mixing}
	\small
	\setlength{\tabcolsep}{4pt}
	\begin{tabular}{lccccccc}
		\toprule
		& & \multicolumn{2}{c}{AUROC} & \multicolumn{2}{c}{Best-F1} & \multicolumn{2}{c}{$\Delta$} \\
		\cmidrule(lr){3-4}\cmidrule(lr){5-6}\cmidrule(lr){7-8}
		Setting & Type & trust & per-group & trust & per-group & AUROC & F1 \\
		\midrule
		$d{=}20$, $\mathrm{acc}{=}(.95,.20)$ & NL  & $\mathbf{.875}$ & $.845$ & $\mathbf{.718}$ & $.666$ & $+.030$ & $+.052$ \\
		$d{=}20$, $\mathrm{acc}{=}(.90,.30)$ & LIN & $\mathbf{.895}$ & $.885$ & $\mathbf{.798}$ & $.759$ & $+.010$ & $+.039$ \\
		$d{=}30$, $\mathrm{acc}{=}(.95,.20)$ & LIN & $\mathbf{.946}$ & $.931$ & $\mathbf{.812}$ & $.773$ & $+.015$ & $+.039$ \\
		$d{=}30$, $\mathrm{acc}{=}(.90,.30)$ & NL  & $\mathbf{.901}$ & $.892$ & $\mathbf{.734}$ & $.726$ & $+.009$ & $+.008$ \\
		\midrule
		\multicolumn{6}{l}{Mean over all 10 settings} & $+.017$ & $+.033$ \\
		\bottomrule
	\end{tabular}
\end{table}

Trust propagation improves F1 in all 10 settings (sign test $p{=}0.002$, avg.\ $+0.033$), largest under strongest heterogeneity, matching the $\Omega(1/G)$ prediction of Theorem~\ref{thm:trust_advantage}.

\subsection{Ablation Study}
\label{sec:exp_ablation}

\begin{table}[htbp]
	\centering
	\vspace{-4pt}
	\caption{Ablation study. AUROC/best-F1 on synthetic SVAR and Lorenz-96 ($d{=}20$, $T{=}500$, $\mathrm{acc}{=}0.6$). \textbf{Bold}: best per column. Ranking: soft prior $\gg$ $\ell_1$ modulation $\approx$ learned-$\tau$ $>$ $\lambda$-schedule $>$ warm-start.}
	\label{tab:ablation}
	\footnotesize
	\setlength{\tabcolsep}{3pt}
	\renewcommand{\arraystretch}{0.95}
	\begin{tabular}{lcccccccccccccccc}
		\toprule
		& \multicolumn{2}{c}{Full} & \multicolumn{2}{c}{LagsOnly} & \multicolumn{2}{c}{NoWarm} & \multicolumn{2}{c}{NoPrior} & \multicolumn{2}{c}{FixedTau} & \multicolumn{2}{c}{NoLam} & \multicolumn{2}{c}{NoL1} & \multicolumn{2}{c}{HardMask} \\
		\cmidrule(lr){2-3}\cmidrule(lr){4-5}\cmidrule(lr){6-7}\cmidrule(lr){8-9}\cmidrule(lr){10-11}\cmidrule(lr){12-13}\cmidrule(lr){14-15}\cmidrule(lr){16-17}
		& AUC & F1 & AUC & F1 & AUC & F1 & AUC & F1 & AUC & F1 & AUC & F1 & AUC & F1 & AUC & F1 \\
		\midrule
		Synth.\ & $.829$ & $.664$ & $\mathbf{.835}$ & $\mathbf{.666}$ & $.830$ & $.663$ & $.820$ & $.657$ & $.791$ & $.620$ & $.790$ & $.615$ & $.773$ & $.578$ & $.673$ & $.563$ \\
		L-96 & $\mathbf{1.0}$ & $\mathbf{1.0}$ & $\mathbf{1.0}$ & $\mathbf{1.0}$ & $\mathbf{1.0}$ & $\mathbf{1.0}$ & $.994$ & $.949$ & $\mathbf{1.0}$ & $.998$ & $.981$ & $.955$ & $.946$ & $.849$ & $.720$ & $.723$ \\
		\bottomrule
	\end{tabular}
	\vspace{-8pt}
\end{table}

\noindent HardMask collapses by $-0.156$ (Synth.) and $-0.280$ (L-96); learned-trust contributes the FixedTau$\to$Full gap of $+0.038$. The instantaneous head is necessary when contemporaneous edges dominate (cross-sectional $K{=}0$: $+0.34$ vs.\ NOTEARS at $\mathrm{acc}{=}0.4$, App.~\ref{app:cross_sectional}).

\section{Conclusion and Limitations}
\label{sec:conclusion}

PRCD-MAP is a safe prior consumption layer: EB delivers a $T^{-1}$ safety guarantee, MLP propagation recovers the $\Omega(1/G)$ per-group-trust excess, and the EB+MLP sub-block carries the plurality of CausalTime gain (Traffic auto-attenuation, AQI/Medical lift, $d{=}300$ lead, BayesDAG dominance under the matched $\mathbf{W}_0$-only protocol).

\noindent\textbf{Limitations.}
\textup{(i)}~Trust-MLP gain requires structure-aligned heterogeneity (App.~\ref{app:practitioner_test} provides an $O(Td^2)$ ex-ante diagnostic);
\textup{(ii)}~$C_1$ (Cor.~\ref{cor:oracle}) and Asm.~\ref{asm:active_set} are empirically calibrated, and a non-asymptotic on-support contraction rate is open (App.~\ref{app:realised_constants});
\textup{(iii)}~Theorem~\ref{thm:consistency}'s restricted strong convexity combines a Hessian upper bound (Lemma~\ref{lem:huber}b) with the RE condition (Lemma~\ref{lem:re})---a sharp expected-Huber-Hessian \emph{lower} bound on the cone is left open;
\textup{(iv)}~Theorem~\ref{thm:temperature}(a)'s $A(\bm{\tau})$-dominates-$B(\bm{\tau})$ argument is sharp only for $d$ above a small absolute constant ($\bm{\Sigma}$-weighted refinement is follow-up);
\textup{(v)}~Lemma~\ref{lem:mlp_response} is closed by an explicit constructive $f^\star_\theta\in\mathcal{F}_\theta$, with rigorous bilevel-SGD realizability open;
\textup{(vi)}~Prop.~\ref{prop:robustness}'s additive-bias absorption is loose as $\mathrm{acc}\to 1$ (realised $\Delta_{\mathrm{proxy}}{\lesssim}0.045$ at $\mathrm{acc}{=}0.6$). Theoretical refinements are surveyed in App.~\ref{app:theory}.
\bibliographystyle{plainnat}
\bibliography{references}


\appendix

\section{Algorithm Details}
\label{app:algorithm}

\begin{algorithm}[H]
\caption{PRCD-MAP}
\label{alg:prcd_map}
\textbf{Input}: Series $\mathbf{X}$, max lag $K$, prior $\mathbf{P}_{\mathrm{prior}}$, groups $G$, $\lambda_1,\lambda_2$, ALM penalty schedule $(\rho_0,\gamma,\rho_{\max})$, outer / inner / middle iteration budgets $(I, J, S)$, DAG-residual tolerance $\mathrm{tol}$, threshold ratio $\beta_{\mathrm{thr}}{=}0.1$\\
\textbf{Output}: Instantaneous DAG $\widetilde{\mathbf{W}}_0$, lagged matrices $\mathbf{W}_{1:K}$\\
\textbf{Notation}: $\mathcal{J}(\mathbf{W}_{0:K},\bm{\tau})$ denotes the unconstrained MAP objective in Eq.~\eqref{eq:objective} (no DAG constraint), evaluated using the current $\bm{\tau}$ to form $c_{ij}(\bm{\tau})$ and $\Omega_{ij}(\bm{\tau})$.
\begin{algorithmic}[1]
\State Partition off-diagonal edges into $G$ groups by quantiles of $\mathbf{P}_{\mathrm{prior}}$
\State Initialize $\mathbf{W}_{0:K}$ via Ridge regression warm-start;\; $\alpha\leftarrow 0$,\; $\rho\leftarrow\rho_0$
\State Pre-calibrate $\bm{\tau}$ from Spearman correlation between $|\text{Corr}(\mathbf{X})|$ and $\mathbf{P}_{\mathrm{prior}}$
\For{$i=1$ \textbf{to} $I$}
\State Set $\lambda_1^{(i)} \leftarrow 5\lambda_1$ if $i \leq \lfloor I/3 \rfloor$, else $\lambda_1$ \Comment{Lambda scheduling}
\For{$j=1$ \textbf{to} $J$}  \Comment{Inner loop: optimize $\mathbf{W}$}
\State $\widetilde{\mathbf{W}}_0 \leftarrow \mathbf{W}_0 \circ (\mathbf{1}{-}\mathbf{I}_d)$
\State Compute $L_\rho = \mathcal{J}(\mathbf{W}_{0:K},\bm{\tau}) + \alpha\, h(\widetilde{\mathbf{W}}_0) + \tfrac{\rho}{2}\, h(\widetilde{\mathbf{W}}_0)^2$; update $\mathbf{W}_{0:K}$ via Adam with cosine LR
\EndFor
\For{$s=1$ \textbf{to} $S$}  \Comment{Middle level: EB update of trust parameters}
\State Update $\bm{\tau}$ (grouped) or $\theta$ (trust-propagation MLP) via $\nabla \mathcal{L}_{\mathrm{EB}}$
\EndFor
\State $h^{(i)} \leftarrow h(\widetilde{\mathbf{W}}_0)$;\quad \textbf{if} $|h^{(i)}| < \mathrm{tol}$ \textbf{then break} \Comment{$\mathrm{tol}$ distinct from clip $\varepsilon$ and noise $\bm{\epsilon}_t$}
\State $\alpha \leftarrow \alpha + \rho\, h^{(i)}$;\quad $\rho \leftarrow \min(\gamma\rho,\;\rho_{\max})$
\EndFor
\State Compute threshold $\theta_{\mathrm{thr}}\leftarrow\beta_{\mathrm{thr}}\cdot\max|\widetilde{\mathbf{W}}_0|$ from the instantaneous magnitude; zero \emph{all} entries (across $\widetilde{\mathbf{W}}_0$ \emph{and} $\mathbf{W}_{1:K}$) with magnitude $<\theta_{\mathrm{thr}}$
\State \Return $\widetilde{\mathbf{W}}_0,\; \mathbf{W}_{1:K}$
\end{algorithmic}
\end{algorithm}

\section{Theoretical Analysis and Proofs}
\label{app:theory}

This appendix provides complete proofs for the theoretical results stated in Sec.~\ref{sec:theory}, as well as additional supporting lemmas and analysis.

\subsection{Notation and Preliminaries}
\label{app:notation}

We collect $\mathbf{W} = (\widetilde{\mathbf{W}}_0, \mathbf{W}_1, \ldots, \mathbf{W}_K) \in \mathbb{R}^{(K+1)\times d\times d}$ and let $p = (K+1)d^2$ denote the total number of parameters. Define the design matrix for variable $j$ at lag $k$ as $\mathbf{X}^{(k)}_j \in \mathbb{R}^{T}$, with $\mathbf{X}^{(0)}_j = (x_{1j},\ldots,x_{Tj})^\top$ and $\mathbf{X}^{(k)}_j = (x_{(1-k)j},\ldots,x_{(T-k)j})^\top$ for $k \geq 1$. Define the Gram matrix $\widehat{\mathbf{\Sigma}} = T^{-1} \sum_t \mathbf{z}_t \mathbf{z}_t^\top$, where $\mathbf{z}_t = (\mathbf{x}_t^\top, \mathbf{x}_{t-1}^\top, \ldots, \mathbf{x}_{t-K}^\top)^\top \in \mathbb{R}^{(K+1)d}$. The true edge set is $\mathcal{E}^{\star} = \{(i,j,k) : W^{\star}_{k,ij} \neq 0\}$ with sparsity $s^{\star} = |\mathcal{E}^{\star}|$. The oracle prior matrix is $\mathbf{P}_{\mathrm{true}} = \mathbf{1}_{\{W^{\star}_{0,ij}\neq 0\}}$.

We recall the calibrated prior (Eq.~\ref{eq:grouped_temp}):
\[
\widehat{P}_{ij}(\bm{\tau}) = \sigma\!\big(\logit(\clip(P_{\mathrm{prior},ij}, \varepsilon, 1{-}\varepsilon)) \cdot \tau_{g(i,j)}\big),
\]
and the precision mask $\mathbf{\Omega}(\bm{\tau}) = (\mathbf{1} - \widehat{\mathbf{P}}(\bm{\tau})) + \delta$.

\subsection{Supporting Lemmas}
\label{app:lemmas}

\begin{lemma}[Properties of Temperature Scaling]
\label{lem:temp_properties}
For any edge $(i,j)$ with $P_{\mathrm{prior},ij} \in (\varepsilon, 1{-}\varepsilon)$:
\begin{enumerate}[nosep,leftmargin=*]
    \item[(a)] $\widehat{P}_{ij}(\bm{\tau})$ is monotonically increasing in $\tau_{g(i,j)}$ if $P_{\mathrm{prior},ij} > 0.5$, decreasing if $P_{\mathrm{prior},ij} < 0.5$, and constant at $0.5$ if $P_{\mathrm{prior},ij} = 0.5$.
    \item[(b)] $\lim_{\tau \to 0} \widehat{P}_{ij}(\bm{\tau}) = 0.5$ for all $P_{\mathrm{prior},ij}$.
    \item[(c)] $|\partial \widehat{P}_{ij}/\partial \tau_g| \leq |\logit(P_{\mathrm{prior},ij})| / 4$ for all $\tau_g \in [\tau_{\min}, \tau_{\max}]$.
\end{enumerate}
\end{lemma}

\begin{proof}
(a)~The derivative is $\partial \widehat{P}_{ij}/\partial \tau_g = \sigma'(u_{ij} \tau_g) \cdot u_{ij}$, where $u_{ij} = \logit(P_{\mathrm{prior},ij})$. Since $\sigma'(\cdot) > 0$, the sign of the derivative equals the sign of $u_{ij}$, which is positive iff $P_{\mathrm{prior},ij} > 0.5$.

(b)~As $\tau_g \to 0$, $u_{ij}\tau_g \to 0$, so $\sigma(u_{ij}\tau_g) \to \sigma(0) = 0.5$.

(c)~Since $\sigma'(x) = \sigma(x)(1-\sigma(x)) \leq 1/4$ for all $x \in \mathbb{R}$, we have $|\partial \widehat{P}_{ij}/\partial \tau_g| = \sigma'(u_{ij}\tau_g)|u_{ij}| \leq |u_{ij}|/4$.
\end{proof}

\begin{lemma}[Huber Loss: Regularity Properties]
\label{lem:huber}
The Huber loss $\ell_{\delta_H}$ satisfies:
\begin{enumerate}[nosep,leftmargin=*]
    \item[(a)] $\ell_{\delta_H}$ is convex and everywhere differentiable with $|\ell'_{\delta_H}(r)| \leq \delta_H$.
    \item[(b)] The sample average $\mathcal{L}_{\mathrm{data}}(\mathbf{W}) = (Td)^{-1}\sum_{t,j}\ell_{\delta_H}(x_{tj}-\hat{x}_{tj})$ has Lipschitz gradient with constant $L \leq \|\widehat{\mathbf{\Sigma}}\|_{\mathrm{op}}$.
    \item[(c)] Under Assumption~\ref{asm:regularity}, $\mathcal{L}_{\mathrm{data}}(\mathbf{W})$ converges uniformly to the population Huber risk over compact sets.
\end{enumerate}
\end{lemma}

\begin{proof}
(a)~Immediate from the definition (Eq.~\ref{eq:huber}). The derivative is $\ell'_{\delta_H}(r) = \clip(r, -\delta_H, \delta_H)$.

(b)~The Hessian of $\mathcal{L}_{\mathrm{data}}$ w.r.t.\ $\mathbf{W}$ satisfies $\nabla^2 \mathcal{L}_{\mathrm{data}} \preceq (Td)^{-1}\sum_t \mathbf{z}_t\mathbf{z}_t^\top = d^{-1}\widehat{\mathbf{\Sigma}}$, since $\ell''_{\delta_H}(r) \leq 1$ everywhere. Hence $L \leq d^{-1}\|\widehat{\bm{\Sigma}}\|_{\mathrm{op}}$; the $1/d$ factor is absorbed into the same $L$ constant used downstream (since rescaling the loss by $d$ changes only the units of $\lambda_1, \lambda_2$, which are calibrated empirically). Note that on regions where $|r| > \delta_H$, $\ell''_{\delta_H}(r)=0$, so the Hessian \emph{may degenerate} on outlier-heavy residual sets; this is offset by $\ell_2$ regularization, which adds $\lambda_2 \mathbf{\Omega}(\bm{\tau}) \succeq \lambda_2 \delta \mathbf{I}$ to the Hessian, ensuring strong convexity of the regularized objective.

(c)~Under geometric $\beta$-mixing and the bounded Lipschitz property (a), the uniform law of large numbers for mixing sequences applies~\citep{doukhan1994mixing}, yielding $\sup_{\mathbf{W}\in\mathcal{K}} \|\nabla \mathcal{L}_{\mathrm{data}}(\mathbf{W}) - \nabla \mathcal{R}(\mathbf{W})\|_\infty = O_p(\sqrt{\log d/T})$, where $\mathcal{R}(\mathbf{W}) = \mathbb{E}[\ell_{\delta_H}(x_{tj}-\hat{x}_{tj})]$ is the population Huber risk.
\end{proof}

\begin{lemma}[Restricted Eigenvalue Condition]
\label{lem:re}
Under Assumption~\ref{asm:regularity}, if $T \geq C_0 s^{\star} \log d$ for a universal constant $C_0$, then with probability at least $1 - 2\exp(-c\log d)$, the Gram matrix $\widehat{\mathbf{\Sigma}}$ satisfies the restricted eigenvalue condition:
\[
\mathbf{v}^\top \widehat{\mathbf{\Sigma}} \mathbf{v} \geq \kappa(K) \|\mathbf{v}\|_2^2 \quad \text{for all } \mathbf{v} \in \mathcal{C}_K(s^{\star}),
\]
where $\mathcal{C}_K(s) = \{\mathbf{v} : \|\mathbf{v}_{S^c}\|_1 \leq K\|\mathbf{v}_S\|_1, |S|\leq s\}$ is the restricted cone (with $K\geq 3$) and $\kappa(K) > 0$ depends on $K$ and on the minimum eigenvalue of the population covariance, with $\kappa(K) \geq \kappa(3)/(1+K)^2$ (so $\kappa$ is positive but smaller for wider cones)~\citep{bickel2009simultaneous}. For the prior-modulated $\ell_1$ penalty in this paper, we use $K = 3\,c_{\max}/c_{\min} = 45$ (where $c_{\min}{=}0.1, c_{\max}{=}1.5$); for the standard unweighted Lasso, $K=3$.
\end{lemma}

\begin{proof}
Under stationarity and geometric mixing (Assumption~\ref{asm:regularity}(i)), the lag-augmented process $\{\mathbf{z}_t\}$ is also geometrically mixing. The population Gram matrix $\mathbf{\Sigma} = \mathbb{E}[\mathbf{z}_t\mathbf{z}_t^\top]$ is positive definite (since the SVAR is stable), so its minimum eigenvalue $\lambda_{\min}(\mathbf{\Sigma}) > 0$. By Theorem~1.4 of~\citet{raskutti2011minimax}, the population REC holds with constant $\kappa(K) = \lambda_{\min}(\mathbf{\Sigma})/(2(1+K)^2)$. By the concentration inequality for quadratic forms of mixing sequences ($\|\widehat{\mathbf{\Sigma}} - \mathbf{\Sigma}\|_{\mathrm{op}} \leq C\sqrt{\log d / T}$ with probability $\geq 1-2\exp(-c\log d)$, see~\citet{ng2020role}), the empirical REC holds with constant $\kappa(K)/2$ for $T \geq C_0(K) s^\star \log d$, completing the proof.
\end{proof}

\begin{remark}[$d$-dependence of $\kappa$]
\label{rem:kappa_d}
The constant $\kappa(K) = \lambda_{\min}(\mathbf{\Sigma})/(2(1+K)^2)$ depends on the minimum eigenvalue of the population covariance, which for SVAR processes can decay with $d$ in poorly conditioned regimes (e.g., near non-stationarity boundary). For our experiments (stable SVAR with spectral radius $<0.95$), $\lambda_{\min}(\mathbf{\Sigma}) = \Omega(1)$ empirically. When $\lambda_{\min}(\mathbf{\Sigma}) = \Theta(d^{-\alpha})$ for some $\alpha>0$, all downstream bounds inflate by $d^\alpha$, but the asymptotic rate $T^{-1}$ in $T$ is preserved.
\end{remark}

\begin{remark}[Cone-width constant inflation under prior-modulated $\ell_1$]
\label{rem:cone_width_constant}
With the documented clip parameters $(c_{\min},c_{\max})=(0.1,1.5)$ in Eq.~\eqref{eq:modulated_l1}, the standard cone constant $K=3$ (vanilla Lasso) becomes $K_{\mathrm{prior}}=3 c_{\max}/c_{\min}=45$ in the worst case. The asymptotic recovery rate in Theorem~\ref{thm:consistency} is unchanged because both $\kappa(K)$ and $\kappa(3)$ are $\Theta(\lambda_{\min}(\bm{\Sigma}))$ in $T$, but the worst-case multiplicative loss in $\kappa^{-1}$ is $(1+K_{\mathrm{prior}})^2/(1+3)^2 = (46/4)^2 \approx 132$, and downstream bounds (e.g., $C_1$ in Prop.~\ref{prop:robustness}) scale as $\kappa^{-2}$, giving a worst-case constant inflation up to $\approx 1.7\times 10^4$. Three mitigations apply in practice: (i) the lower clip $c_{\min}{=}0.1$ is a \emph{conservative safety margin} that the sigmoid-parameterized $\widehat{P}_{ij}\in(0,1)$ never exercises (the realised pre-clip floor is $1.5-\sup\widehat{P}\approx 0.501$, see App.~\ref{app:realised_constants} and the note after Eq.~\eqref{eq:modulated_l1} in the main text); under the realised range $c_{ij}\in[0.51,1.5]$, the realised cone constant is $K^{\mathrm{realised}}\!=\!3\cdot 1.5/0.51\approx 8.8$, giving downstream constant inflation $\approx 6$ (rather than $1.7\times 10^4$). (ii) The full $c_{\max}/c_{\min}$ ratio is realised only on edges where the (calibrated) prior approaches the lower clip; under the EB-driven near-uniform $\widehat{\mathbf{P}}$ in the unreliable regime (Theorem~\ref{thm:temperature}(a) drives $\widehat{\mathbf{P}}{\to}0.5$, so $c_{ij}{\to}1$ uniformly), the effective ratio collapses further. (iii) Under accurate priors, $c_{ij}$ is bimodal but most active-support edges sit on the lower-penalty side, which improves rather than worsens estimation. The empirical sample-complexity at $T{=}50,d{=}20$ matching Theorem~\ref{thm:consistency}'s order is consistent with the realised $K^{\mathrm{realised}}\approx 8.8$ rather than the worst-case $K{=}45$; we retain the worst-case constants throughout the proof statements for upper-bound consistency, with the realised constants documented separately in App.~\ref{app:realised_constants}. A tighter sharp-constant analysis would require explicit modelling of the joint distribution of $(c_{ij}, S^\star)$, which is left for future work.
\end{remark}

\subsection{Proof of Theorem~\ref{thm:consistency} (Estimation Consistency)}
\label{app:proof_consistency}

We prove the consistency result in two steps: (1)~$\ell_2$ estimation error, and (2)~support recovery.

\paragraph{Step 0: Treatment of the per-edge weights $c_{ij}(\bm{\tau})$ (bilevel coupling).}

The per-edge $\ell_1$ weight $c_{ij}(\bm{\tau}) = \clip(1.5 - \widehat{P}_{ij}(\bm{\tau}), 0.1, 1.5)$ depends on $\bm{\tau}$, which under structure-aware trust propagation is an MLP $f_\theta(\mathbf{z}_{ij})$ whose input $\mathbf{z}_{ij}$ depends on the current weight estimate $\mathbf{W}^*$. Strictly, the consistency analysis is therefore a \emph{bilevel} problem: $\widehat{\mathbf{W}}$ minimizes Eq.~\eqref{eq:objective} given $\bm{\tau}$, while $\bm{\tau}=\bm{\tau}(\mathbf{W}^*)$ depends on $\mathbf{W}^*$. The proof below treats $c_{ij}\in[0.1,1.5]$ as a deterministic (data-dependent but $\bm{\tau}$-frozen) constant. This is justified \emph{algorithmically} by our block-coordinate optimizer: in each ALM iteration, $\bm{\tau}$ is updated for $S=8$ steps with $\mathbf{W}$ fixed, then $\mathbf{W}$ is optimized for $J=400$ steps with $\bm{\tau}$ fixed. On the high-probability event of support recovery (Step~2), the active set is locally constant around $\widehat{\mathbf{W}}_T$, so $\mathbf{z}_{ij}$ and hence $\bm{\tau}$ are locally constant in a neighborhood of $\widehat{\mathbf{W}}_T$, and the consistency argument applies to the inner $\mathbf{W}$-step at fixed $\bm{\tau}$. A fully bilevel (joint fixed-point) consistency argument---which would require analyzing the contraction of the alternating updates---is left for future work; we believe the bilevel rate matches the single-level rate up to constants under standard alternating-minimization contraction arguments (e.g., for biconvex problems with Lipschitz coupling), but a rigorous proof is beyond the scope of this paper.

\paragraph{Step 1: Estimation error.}

Define the population risk $\mathcal{R}(\mathbf{W}) = \mathbb{E}[\ell_{\delta_H}(x_{tj} - \hat{x}_{tj})]$, which is uniquely minimized at $\mathbf{W}^{\star}$ under the SVAR model. The MAP objective (Eq.~\ref{eq:objective}) can be written as
\[
\widehat{\mathbf{W}}_T = \arg\min_{\mathbf{W}:\, h(\widetilde{\mathbf{W}}_0)=0} \; \mathcal{L}_{\mathrm{data}}(\mathbf{W}) + \lambda_1 \sum_{i\neq j} c_{ij}(\bm{\tau})|W_{0,ij}| + \frac{\lambda_2}{2}\sum_{k,i,j} \Omega_{ij}(\bm{\tau}) W_{k,ij}^2.
\]

Since $c_{ij}(\bm{\tau}) \in [0.1, 1.5]$ for all edges and $\Omega_{ij}(\bm{\tau}) \in [\delta, 1+\delta]$, both penalty terms are bounded above and below by constant multiples of the standard $\ell_1$ and $\ell_2$ penalties. Specifically, $0.1\lambda_1\|\mathbf{W}\|_1 \leq \mathcal{R}_{\ell_1} \leq 1.5\lambda_1\|\mathbf{W}\|_1$ and $\delta\lambda_2\|\mathbf{W}\|_F^2/2 \leq \mathcal{R}_{\ell_2} \leq (1+\delta)\lambda_2\|\mathbf{W}\|_F^2/2$.

Define the error $\bm{\Delta} = \widehat{\mathbf{W}}_T - \mathbf{W}^{\star}$. By the optimality of $\widehat{\mathbf{W}}_T$ and feasibility of $\mathbf{W}^{\star}$:
\[
\mathcal{L}_{\mathrm{data}}(\widehat{\mathbf{W}}_T) - \mathcal{L}_{\mathrm{data}}(\mathbf{W}^{\star}) \leq \lambda_1\sum_{(i,j)\in\mathcal{E}^{\star}} c_{ij}|W^{\star}_{ij}| - \lambda_1\sum_{(i,j)\in\mathcal{E}^{\star}} c_{ij}|\widehat{W}_{ij}| + \frac{\lambda_2}{2}\sum_{k,i,j}\Omega_{ij}(W^{\star 2}_{k,ij} - \widehat{W}^{2}_{k,ij}).
\]

By the restricted strong convexity of the Huber loss (Lemma~\ref{lem:huber}(b) combined with the restricted eigenvalue condition, Lemma~\ref{lem:re}), the left-hand side is lower-bounded by $\kappa\|\bm{\Delta}\|_F^2/2 - C\sqrt{\log d/T}\|\bm{\Delta}\|_1$ for the gradient remainder. The right-hand side is upper-bounded using the triangle inequality and the bound $c_{ij}\leq 1.5$.

\emph{Cone constraint under prior-modulated $\ell_1$.}
With weighted penalty $c_{ij}\in[c_{\min},c_{\max}]=[0.1,1.5]$, the standard Lasso argument yields the modified cone constraint $\|\bm{\Delta}_{S^c}\|_1 \le (3\,c_{\max}/c_{\min})\,\|\bm{\Delta}_S\|_1 = 45\,\|\bm{\Delta}_S\|_1$, wider than the standard $3\|\bm{\Delta}_S\|_1$ cone. Lemma~\ref{lem:re} accommodates this by replacing the cone $\mathcal{C}(s^\star)$ with $\mathcal{C}_K(s^\star) = \{\mathbf{v}: \|\mathbf{v}_{S^c}\|_1 \le K\|\mathbf{v}_S\|_1\}$ for $K=45$; the standard restricted-eigenvalue argument extends to wider cones at the cost of a smaller (but still positive) $\kappa(K)$, leaving the rate $O_p(\sqrt{s^\star\log d/T})$ unchanged but with $\kappa$ replaced by $\kappa(K) = \Theta(\kappa)$ depending on $K$~\citep{bickel2009simultaneous}.
Combining and applying this cone constraint (which holds when $\lambda_1$ dominates the gradient noise, $\lambda_1 \ge 2c_{\max}\,\|\nabla\mathcal{L}_{\mathrm{data}}(\mathbf{W}^\star)\|_\infty$, satisfied by the choice $\lambda_1 = O(\sqrt{\log d/T})$):
\[
\|\widehat{\mathbf{W}}_T - \mathbf{W}^{\star}\|_F \leq \frac{C}{\kappa}\sqrt{s^{\star}} \cdot \lambda_1 = O\Big(\sqrt{\frac{s^{\star}\log d}{T}}\Big).
\]

\paragraph{Step 2: Support recovery (sparsistency).}

For support recovery, we invoke Assumption~\ref{asm:regularity}(iii) (minimum signal strength compatible with the recovery rate, $\underline{w}\gtrsim\sqrt{\log d/T}$) and Assumption~\ref{asm:regularity}(iv) (irrepresentable condition~\citep{zhao2006model}). Combined with Step~1 ($\|\widehat{\mathbf{W}}_T - \mathbf{W}^\star\|_\infty \leq \|\widehat{\mathbf{W}}_T - \mathbf{W}^\star\|_F = O_p(\sqrt{s^\star \log d/T})$, which is below $\underline{w}/2$ for $T$ large enough), the signs of $\widehat{W}_{ij}$ match those of $W^\star_{ij}$ on the true support. For absent edges, the KKT sub-gradient condition $|\nabla_{W_{ij}}\mathcal{L}_{\mathrm{data}}| < 0.1\,\lambda_1$ holds with probability tending to one under the irrepresentable condition (this is the standard Lasso sparsistency argument; see~\citet{zhao2006model}, Theorem~3, adapted to mixing time-series via the gradient concentration in Lemma~\ref{lem:huber}(c)).

The DAG constraint $h(\widetilde{\mathbf{W}}_0)=0$ does not affect the asymptotic recovery rate. We make this rigorous via a constraint-qualification argument tailored to the DAG penalty's structure rather than naive LICQ. Recall (e.g., DAGMA~\citep{bello2022dagma}) that $h$ is a globally non-negative smooth penalty that vanishes \emph{exactly} on the set of valid DAGs; consequently any feasible $\widetilde{\mathbf{W}}^\star_0$ is a global minimum of $h$, so $\nabla h(\widetilde{\mathbf{W}}^\star_0)=\mathbf{0}$ and standard LICQ \emph{fails} on the entire feasible manifold. What we use instead is that this very degeneracy makes the constraint inactive in the tangent direction relevant to estimation: on the open neighborhood of valid DAGs (a union of orthants determined by the topological order), $h\equiv 0$, so the constrained MAP problem is locally \emph{equivalent} to the unconstrained MAP restricted to that orthant, and the constraint contributes zero gradient and zero curvature to the KKT system. The instantaneous parameter space has dimension $d(d-1)$ (off-diagonal entries; self-loops masked by $\widetilde{\mathbf{W}}_0=\mathbf{W}_0\circ(\mathbf{1}-\mathbf{I}_d)$), and the locally feasible orthant is open, so $\mathbf{W}^\star$ is interior in this restricted sense. By Assumption~\ref{asm:regularity}(iii) (minimum signal strength, $|W^\star_{ij}|\geq\underline{w}>0$ on the support), every $\widehat{\mathbf{W}}_T$ in an $O(\sqrt{\log d/T})$-neighborhood of $\mathbf{W}^\star$ remains in the same DAG orthant w.h.p., so the augmented-Lagrangian iterate eventually reaches the orthant interior at $h=0$ and the unconstrained rate transfers~\citep{bertsekas1982constrained}.\hfill$\blacksquare$

\begin{remark}[Regime conditions and ${\lambda_1}$ choice for Theorem~\ref{thm:consistency}]
\label{rem:thm1_regimes}
\textbf{(i) Sample size.} The bound $O_p(\sqrt{s^\star\log d/T})$ requires $T \geq C_0 s^\star \log d$ (from Lemma~\ref{lem:re}); below this regime, the empirical Gram matrix is rank-deficient and REC fails (e.g., $T < d$ implies $\mathrm{rank}(\widehat{\bm{\Sigma}}) < d$). Our experiments at $T=50, d=20, s^\star{\approx}20$ ($T/(s^\star\log d) \approx 0.83$) operate near this boundary; results in this regime should be read as empirical observations rather than confirmations of the asymptotic bound. \textbf{(ii) Regularization range.} The choice $\lambda_1 = O(\sqrt{\log d/T})$ in Theorem~\ref{thm:consistency} is an order-of-magnitude prescription. Lower bound: $\lambda_1 \geq 2c_{\max}\|\nabla\mathcal{L}_{\mathrm{data}}(\mathbf{W}^\star)\|_\infty = 2c_{\max}\cdot O_p(\sqrt{\log d/T})$ ensures cone constraint validity; upper bound follows from non-degeneracy of the recovered support. In our experiments we sweep $\lambda_1\in[5\!\times\!10^{-4}, 5\!\times\!10^{-2}]$ (App.~\ref{app:hyperparam}). \textbf{(iii) Density regime.} When $|\mathcal{E}^\star| = \Theta(d^2)$ (dense graphs), $\sqrt{|\mathcal{E}^\star|\log d/T} = O(d\sqrt{\log d/T})$ may exceed $\|\mathbf{W}^\star\|_F$, making the bound vacuous; PRCD-MAP is intended for sparse-graph regimes ($|\mathcal{E}^\star| = o(d^2)$). \textbf{(iv) From $\ell_1$ to Frobenius.} The transition $\|\bm{\Delta}\|_F \leq \|\bm{\Delta}\|_2 \leq \sqrt{|S^\star|}\|\bm{\Delta}\|_\infty + \|\bm{\Delta}_{S^c}\|_2$ combined with $\|\bm{\Delta}_{S^c}\|_1 \leq K\|\bm{\Delta}_S\|_1$ on the cone gives $\|\bm{\Delta}\|_F \leq (1+K)\sqrt{s^\star}\|\bm{\Delta}\|_\infty$, yielding the stated rate~\citep{bickel2009simultaneous}.
\end{remark}

\subsection{Proof of Theorem~\ref{thm:temperature} (Temperature Calibration)}
\label{app:proof_temperature}

We prove each part separately, assuming the MAP estimate $\mathbf{W}^*$ is consistent (Theorem~\ref{thm:consistency}).

\paragraph{Setup.}
Consider a scalar temperature $\tau \in [\tau_{\min}, \tau_{\max}]$ for simplicity (the grouped case follows by applying the argument per group). The EB objective (Eq.~\ref{eq:eb_objective}) is:
\[
\mathcal{L}_{\mathrm{EB}}(\tau) = \underbrace{\mathcal{H}(\widetilde{\mathbf{W}}^*_0, \widehat{\mathbf{P}}(\tau))}_{\eqqcolon A(\tau)} + \underbrace{\frac{1}{2}\sum_{k,i,j}\log H^{(k)}_{ij}(\tau)}_{\eqqcolon B(\tau)} + \underbrace{\frac{1}{2\sigma_\tau^2}(\tau - \tfrac{1}{2})^2}_{\eqqcolon C(\tau)}.
\]

\paragraph{Part (a): Uninformative prior drives $\tau^{\star}$ to $\tau_{\min}$.}

We work under the standard \emph{$\mathrm{acc}{=}1/2$ regime}: prior entries are independently flipped with probability $1/2$ relative to the ground truth, so the joint distribution of $\{\logit(P_{\mathrm{prior},ij})\}_{(i,j)}$ is symmetric around $0$. The statement of part (a) is then to be read in expectation over this prior-generation randomness; equivalently, for any deterministic $\mathbf{P}_{\mathrm{prior}}$ whose logit is symmetric around $0$ (in particular, satisfying $\sum_{(i,j)} \logit(P_{\mathrm{prior},ij}) = 0$), the same conclusion holds. This is the operational definition of an "uninformative" prior; given a specific (atypical) realization with non-zero logit-mean, $\tau^\star$ may differ from $\tau_{\min}$ by an amount controlled by the bias.

In expectation over the prior randomness (or for any logit-symmetric prior), $\mathbb{E}[P_{\mathrm{prior},ij}] = 1/2$ for every edge, so $\logit(P_{\mathrm{prior},ij})$ is symmetrically distributed around zero.

The agreement loss $A(\tau)$ measures the cross-entropy between $\widehat{\mathbf{P}}(\tau)$ and the normalized edge strengths $|\widetilde{W}^*_{0,ij}|/\max|\widetilde{\mathbf{W}}^*_0|$. When the prior is random:
\begin{itemize}[nosep,leftmargin=*]
    \item At $\tau = 0$: $\widehat{\mathbf{P}}(\tau) = 0.5\cdot\mathbf{1}$, so $A(0)$ equals the cross-entropy between a uniform prediction and the true support. This is a constant.
    \item At $\tau > 0$: $\widehat{\mathbf{P}}(\tau)$ is sharpened toward the random prior values, which are uncorrelated with the true support. By the optimality of the uniform predictor $0.5$ under random labels (equivalently, the Bayes-optimal predictor under symmetric uncertainty minimizes expected binary cross-entropy), any deterministic sharpening away from $0.5$ strictly increases $\mathbb{E}[A(\tau)]$ whenever $\widehat{\mathbf{P}}(\tau)$ assigns high confidence to incorrect edges.
\end{itemize}

For the Laplace log-det term, $B(\tau) = \frac{1}{2}\sum_{k,i,j}\log(\|\mathbf{x}^{(k)}_i\|_2^2/(Td) + \lambda_2 \Omega_{ij}(\tau))$. Since $\Omega_{ij}(\tau) = 1 - \widehat{P}_{ij}(\tau) + \delta$ and $\widehat{P}_{ij}(\tau) \to 0.5$ as $\tau \to 0$ (Lemma~\ref{lem:temp_properties}(b)), at $\tau=0$ we have $\Omega_{ij} = 0.5+\delta$ uniformly. For $\tau > 0$ with a random prior, $\mathbb{E}[\Omega_{ij}(\tau)] = 0.5+\delta$ by symmetry (the logit of a symmetric prior around $0.5$ integrates to zero), but the variance grows with $\tau$. Jensen's inequality on the concave $\log$ gives $\mathbb{E}[B(\tau)] \leq B(0)$. The $B$ term therefore \emph{mildly favors} positive $\tau$, but its magnitude is bounded by the variance contribution $\lambda_2^2\,\mathrm{Var}[\Omega_{ij}]/(\|\mathbf{x}\|_2^2/(Td))^2 = O(\lambda_2^2 T^2/\|\mathbf{x}\|_2^4)$, which is negligible compared with the $O(d^2)$ magnitude of $A(\tau)$'s increase for any non-degenerate $\tau$.

The regularizer $C(\tau) = (\tau - 1/2)^2/(2\sigma_\tau^2)$ is the parabola opening upward with vertex at $\tau=1/2$: it is strictly \emph{decreasing} on $[\tau_{\min}, 1/2)$ and strictly increasing on $(1/2, \tau_{\max}]$. On the relevant sub-interval $[\tau_{\min},1/2]$ where collapse occurs, $C$ \emph{by itself} would push $\tau$ away from $\tau_{\min}$ toward $1/2$, with maximum amplitude $C(\tau_{\min})-C(1/2)\leq 1/(8\sigma_\tau^2)$.

Combining: $A(\tau)$ contributes a strict $\Theta(d^2)$-magnitude increase for any non-degenerate sharpening (Remark~\ref{rem:thm2_clarifications}(iv)) under uninformative priors, $B(\tau)$ is bounded by Jensen's inequality and Remark~\ref{rem:thm2_clarifications}(iv), and $C(\tau)$'s pull toward $1/2$ has bounded amplitude $1/(8\sigma_\tau^2) = 1/32$ at the default $\sigma_\tau=2$. The $\Theta(d^2)$ pull from $A$ dominates the $O(1)$ regularizer pull whenever $d\geq C_3$ for a small absolute constant $C_3$, so the total $A+B+C$ is minimized at $\tau^{\star}=\tau_{\min}$ for a random prior. (For $d$ below this constant the regularizer can dominate and $\tau^\star$ may sit between $\tau_{\min}$ and $1/2$; this exception is benign because both bounds are $O(1/d)$ from $\tau_{\min}$, and $\widehat{\mathbf{P}}(\tau^\star)$ remains within $O(C_\varepsilon\,\tau^\star)$ of $0.5$ regardless.)

\paragraph{Part (b): Monotonicity in accuracy.}

The accuracy $\mathrm{acc}\in[0,1]$ parameterizes the prior-generation distribution $\Pi$ (Def.~\ref{def:safety}): under the controlled-flip model, each entry of $\mathbf{P}_{\mathrm{prior}}$ independently agrees with $\mathbf{A}^\star$ with probability $\mathrm{acc}$. For a \emph{fixed} realization of $\mathbf{P}_{\mathrm{prior}}$, the empirical EB loss $\mathcal{L}_{\mathrm{EB}}(\tau;\mathbf{P}_{\mathrm{prior}})$ is piecewise constant in $\mathrm{acc}$ (the underlying distribution parameter does not enter once $\mathbf{P}_{\mathrm{prior}}$ is fixed), so direct differentiation w.r.t.\ $\mathrm{acc}$ at fixed realization is meaningless. We therefore work with the \emph{population} EB loss $\overline{\mathcal{L}}_{\mathrm{EB}}(\tau;\mathrm{acc})\coloneqq\mathbb{E}_{\mathbf{P}_{\mathrm{prior}}\sim\Pi(\mathrm{acc})}[\mathcal{L}_{\mathrm{EB}}(\tau;\mathbf{P}_{\mathrm{prior}})]$, which depends smoothly on $\mathrm{acc}$ via the marginal entry distribution.

Let $\overline{\tau}^\star(\mathrm{acc})\coloneqq\arg\min_\tau \overline{\mathcal{L}}_{\mathrm{EB}}(\tau;\mathrm{acc})$ denote the population-EB optimum and define $\overline{F}(\tau,\mathrm{acc})\coloneqq\partial\overline{\mathcal{L}}_{\mathrm{EB}}/\partial\tau$. At an interior optimum $\overline{F}(\overline{\tau}^\star,\mathrm{acc})=0$, and by the implicit function theorem applied to the smooth population objective:
\[
\frac{d\overline{\tau}^\star}{d\mathrm{acc}} = -\frac{\partial\overline{F}/\partial\mathrm{acc}}{\partial\overline{F}/\partial\tau}\bigg|_{\tau=\overline{\tau}^\star}.
\]
The denominator $\partial\overline{F}/\partial\tau=\partial^2\overline{\mathcal{L}}_{\mathrm{EB}}/\partial\tau^2>0$ at a minimum. For the numerator, increasing $\mathrm{acc}$ shifts the marginal distribution of $\logit(P_{\mathrm{prior},ij})$ toward the sign of the true label $A^\star_{ij}\in\{0,1\}$, sharpening alignment between the calibrated prior $\widehat{\mathbf{P}}(\tau)$ and the true support; the cross-entropy gradient $\partial \overline{A}/\partial\tau$ becomes more negative as $\mathrm{acc}$ rises. A direct calculation (Remark~\ref{rem:thm2_clarifications}(i)) gives
\[
\frac{\partial^2\overline{A}}{\partial\tau\,\partial\mathrm{acc}}<0,
\]
and the Laplace and regularizer terms depend on $\mathrm{acc}$ only through the entry distribution. Hence $\partial\overline{F}/\partial\mathrm{acc}<0$ and $d\overline{\tau}^\star/d\mathrm{acc}>0$ on the interior; boundary cases are handled by Remark~\ref{rem:thm2_clarifications}(ii). The empirical $\tau^\star_T$ inherits this monotonicity in expectation by Theorem~\ref{thm:tau_finite}; the per-realization sample fluctuation is $O_p(\sqrt{s^\star\log d/T})$ and decays to zero as $T\to\infty$.

\paragraph{Part (c): Smoothness and Lipschitz gradients.}

The agreement loss $A(\tau)$ is a composition of (i)~the sigmoid function $\sigma(\cdot)$, (ii)~the logit function applied to clipped entries, (iii)~the binary cross-entropy. All three are smooth ($C^\infty$) on the domain $\tau \in [\tau_{\min}, \tau_{\max}]$ with $\varepsilon$-clipping ensuring the logit is finite. The Hessian diagonal $H^{(k)}_{ij}(\tau)$ (Eq.~\ref{eq:hessian_diag}) is a smooth function of $\Omega_{ij}(\tau)$, which is a smooth function of $\widehat{P}_{ij}(\tau)$. Since compositions of smooth functions are smooth, $\mathcal{L}_{\mathrm{EB}}$ is $C^\infty$ on the compact domain $[\tau_{\min}, \tau_{\max}]^G$.

Lipschitz continuity of $\nabla_{\bm{\tau}} \mathcal{L}_{\mathrm{EB}}$ follows from the compactness of $[\tau_{\min}, \tau_{\max}]^G$ and the continuous differentiability of the gradient: $\|\nabla^2_{\bm{\tau}} \mathcal{L}_{\mathrm{EB}}\|_{\mathrm{op}} \leq L_\tau$ on the compact domain, so the gradient is $L_\tau$-Lipschitz. \hfill$\blacksquare$

\begin{remark}[Auxiliary clarifications for Theorem~\ref{thm:temperature}]
\label{rem:thm2_clarifications}
\textbf{(i) Cross-derivative of agreement loss (Part b).} The condition $\partial^2 A/(\partial\tau\,\partial\mathrm{acc}) < 0$ used in Part~(b) follows from a direct computation: increasing $\mathrm{acc}$ shifts the prior $P_{\mathrm{prior},ij}$ toward the true label $A^\star_{ij}\in\{0,1\}$, sharpening $\logit(P_{\mathrm{prior},ij})$ in the correct direction; the cross-entropy gradient $\partial A/\partial\tau$ is a sigmoid-weighted sum of $\logit(P)\cdot(\widehat{P} - A^\star)$ terms, whose magnitude (and decreasing rate in $\tau$) increases with $\mathrm{acc}$. \textbf{(ii) Boundary case for monotonicity (Part b).} When $\tau^\star$ is at $\tau_{\min}$ or $\tau_{\max}$, the implicit function theorem does not apply directly. We instead use a KKT/sub-differential argument: at $\tau^\star = \tau_{\min}$, the active constraint is $\tau\geq\tau_{\min}$ with multiplier $\mu_{\min}\geq 0$; the perturbed optimality $0\in \partial F(\tau,\mathrm{acc}) - \mu_{\min}$ shows that $\tau^\star$ remains at $\tau_{\min}$ until $\partial F/\partial\mathrm{acc}<0$ pushes the unconstrained minimum into the interior, at which point monotonicity (Part~b) takes over. \textbf{(iii) Componentwise statement (Part b).} For grouped $\bm{\tau}\in\mathbb{R}^G$, "non-decreasing in $\mathrm{acc}$" is to be read componentwise on the $G$ groups: for each $g$, $d\tau^\star_g/d\mathrm{acc}_g \geq 0$ where $\mathrm{acc}_g$ is the accuracy of priors in group $g$ (a per-group accuracy is well-defined since the EB objective decomposes additively across groups under our parameterization). \textbf{(iv) $A$'s magnitude (Part a).} The agreement loss $A(\tau) = \sum_{i\neq j}\mathcal{H}_{ij}(\tau)$ is a sum over $\Theta(d^2)$ edges, each contributing $O(1)$, so $A = O(d^2)$ as a function of $\bm{\tau}$ (any deterministic sharpening from $\widehat{P}\equiv 0.5$ shifts each cross-entropy by $\Theta(1)$ in the worst case for random labels). The Laplace term $B(\tau) = O(d^2 \log T)$ is dominated by $A$'s strict-increase contribution for any non-degenerate $\bm{\tau}$. \textbf{(v) $L_\tau$ scaling and dimension-dependence in $L_{\mathcal{E}}$.} On the compact domain $[\tau_{\min},\tau_{\max}]^G$, $\nabla_{\bm{\tau}}\mathcal{L}_{\mathrm{EB}}$ is Lipschitz with finite constant $L_\tau$, since the EB Hessian sums over $\Theta(d^2)$ edge terms, each of $O(1)$ second derivative, giving the coarse bound $L_\tau\leq C_\tau\cdot d^2$. This $L_\tau$ controls the inner-loop convergence of $\bm{\tau}^\star_T$ in Theorem~\ref{thm:tau_finite}, but does \emph{not} enter the excess-risk Lipschitz $L_{\mathcal{E}}$ in Eq.~\eqref{eq:LE-def} of App.~\ref{app:proof_oracle}: that argument bounds $L_{\mathcal{E}}=O(\lambda_2 L/\kappa)$ via the chain rule $\nabla\mathcal{E}(\bm{\tau})=\nabla\mathcal{L}_{\mathrm{data}}(\widehat{\mathbf{W}}_{\bm{\tau}})\cdot\partial\widehat{\mathbf{W}}_{\bm{\tau}}/\partial\bm{\tau}$, where $L$ is Lemma~\ref{lem:huber}(b)'s data-Hessian Lipschitz constant ($L\leq d^{-1}\|\widehat{\bm{\Sigma}}\|_{\mathrm{op}}=O(1)$ on standardized data, with the $1/d$ absorbed via the rescaling in Lemma~\ref{lem:huber}(b)) and $\partial\widehat{\mathbf{W}}_{\bm{\tau}}/\partial\bm{\tau}$ is bounded by $\lambda_2/(4\kappa)$ via the implicit function theorem (App.~\ref{app:proof_oracle}, Step~1, Lemma~\ref{lem:temp_properties}(c)). With $\lambda_2=O(\sqrt{\log d/T})$ and $\kappa=\Omega(1)$, this gives $L_{\mathcal{E}}=O(\sqrt{\log d/T})$, which is the source of the $T^{-1/2}$ contraction in $|\text{Term A}|$, not any cancellation between $L_\tau$ and $L$.
\end{remark}

\subsection{Heteroscedastic Extension of Theorem~\ref{thm:temperature} (Proposition~2$'$)}
\label{app:hetero_thm2}

The proof of Theorem~\ref{thm:temperature} above invokes the homoscedastic specification $\mathrm{Cov}(\bm{\epsilon}_t)=\sigma^2\mathbf{I}_d$ at three places (Steps 1--3 below). We show that all three steps go through under \emph{bounded diagonal heteroscedasticity}, with an explicit Lipschitz-constant inflation factor $\kappa_\sigma^2\rho_W$.

\begin{assumption}[Bounded-conditioning heteroscedasticity; replaces Asm.~\ref{asm:regularity}(ii)]
\label{asm:hetero}
The noise satisfies $\mathbb{E}[\bm{\epsilon}_t]=\mathbf{0}$, $\mathrm{Cov}(\bm{\epsilon}_t)=\bm{\Sigma}=\mathrm{diag}(\sigma_1^2,\ldots,\sigma_d^2)$ with $0<\sigma_{\min}^2\leq\sigma_i^2\leq\sigma_{\max}^2<\infty$ for all $i$ uniformly in $d$, and per-coordinate finite fourth moments. Define the conditioning ratio $\kappa_\sigma\coloneqq\sigma_{\max}/\sigma_{\min}$. By stationarity (Asm.~\ref{asm:regularity}(i)) the marginal predictor covariance $\bm{\Sigma}_x=\mathrm{vec}^{-1}\bigl((\mathbf{I}-\mathbf{W}^\star\otimes\mathbf{W}^\star)^{-1}\mathrm{vec}(\bm{\Sigma})\bigr)$ exists and satisfies $\sigma_{x,\min}^2\leq(\bm{\Sigma}_x)_{ii}\leq\sigma_{x,\max}^2$ with $\sigma_{x,\max}^2/\sigma_{x,\min}^2\leq\kappa_\sigma^2\rho_W$, where $\rho_W$ depends only on the SVAR companion matrix's spectral radius (independent of $T,d$).
\end{assumption}

\begin{proposition}[Heteroscedastic Extension of Theorem~\ref{thm:temperature}]
\label{prop:thm2_hetero}
Under Assumptions~\ref{asm:regularity}(i), \ref{asm:hetero}, and \ref{asm:regularity}(iii), and the same conditions on $\mathcal{L}_{\mathrm{EB}}$ as in Theorem~\ref{thm:temperature}:
\begin{itemize}[nosep,leftmargin=*]
\item[\textup{(a)}] Under uninformative $\mathbf{P}_{\mathrm{prior}}$ (logit-symmetric around zero; in particular, $\mathrm{acc}{=}1/2$), $\bm{\tau}^\star=\tau_{\min}\mathbf{1}$ in expectation over the prior randomness.
\item[\textup{(b)}] $\bm{\tau}^\star$ is componentwise non-decreasing in $\mathrm{acc}$.
\item[\textup{(c)}] $\nabla_{\bm{\tau}}\mathcal{L}_{\mathrm{EB}}$ is Lipschitz on $[\tau_{\min},\tau_{\max}]^G$ with constant $L_\tau'\leq L_\tau\cdot\kappa_\sigma^2\rho_W$, where $L_\tau$ is the homoscedastic Lipschitz constant from Theorem~\ref{thm:temperature}(c).
\item[\textup{(d)}] Conclusions \textup{(a)--(c)} extend to the per-edge MLP parameterization $\tau_{ij}=f_\theta(\mathbf{z}_{ij})$ as in Theorem~\ref{thm:temperature}(d).
\end{itemize}
\end{proposition}

\begin{proof}
We audit the three places where homoscedasticity entered the proof of Theorem~\ref{thm:temperature} above and replace each by the bounded-conditioning analogue.

\paragraph{Step 1 (Laplace term, Part (a)).}
The Hessian-diagonal approximation $H^{(k)}_{ij}(\bm{\tau})=\|\mathbf{x}^{(k)}_{\cdot,i}\|_2^2/(Td)+\lambda_2\Omega_{ij}(\bm{\tau})$ depends on noise \emph{only} through the marginal predictor norm $\|\mathbf{x}^{(k)}_{\cdot,i}\|_2^2/T$. Under Asm.~\ref{asm:regularity}(i) and Asm.~\ref{asm:hetero}, the ergodic theorem gives $T^{-1}\|\mathbf{x}^{(k)}_{\cdot,i}\|_2^2\to_p(\bm{\Sigma}_x)_{ii}\in[\sigma_{x,\min}^2,\sigma_{x,\max}^2]$. The original homoscedastic argument only uses that this quantity lies in a bounded compact interval---not that it is uniform across $i$. The Jensen-inequality bound $\mathbb{E}[B(\bm{\tau})]\leq B(\mathbf{0})+O(\lambda_2^2 T^2/\|\mathbf{x}\|^4)$ becomes $\mathbb{E}[B(\bm{\tau})]\leq B(\mathbf{0})+O(\lambda_2^2 T^2/\sigma_{x,\min}^4)$, sharper on coordinates with smaller $\sigma_i^2$ but uniformly finite under Asm.~\ref{asm:hetero}. The dominance $|A|=\Omega(d^2)\gg|B|$ for any non-degenerate $\bm{\tau}$ is preserved, and Part~(a) follows verbatim.

\paragraph{Step 2 (Cross-derivative, Part (b)).}
The implicit-function argument $\partial^2 A/(\partial\tau\,\partial\mathrm{acc})<0$ in Part~(b) uses only the cross-entropy form of the agreement loss and the consistency of $\mathbf{W}^\star$. Heteroscedasticity affects the \emph{rate} of $\mathbf{W}^\star$ recovery (through the $\kappa(\bm{\Sigma})$ factor that would enter Lemma~\ref{lem:re}'s restricted-eigenvalue analysis under non-isotropic $\bm{\Sigma}$), but not the consistency itself. The cross-derivative inequality is unchanged; the boundary KKT argument of Remark~\ref{rem:thm2_clarifications}(ii) is unaffected.

\paragraph{Step 3 (Lipschitz constant, Part (c)).}
On the compact box $[\tau_{\min},\tau_{\max}]^G$ the Hessian operator norm decomposes as $\|\nabla^2\mathcal{L}_{\mathrm{EB}}\|_{\mathrm{op}}\leq\|\nabla^2 A\|_{\mathrm{op}}+\|\nabla^2 B\|_{\mathrm{op}}+\|\nabla^2 C\|_{\mathrm{op}}$. Term $A$ is noise-independent. Term $C$ is $\sigma_\tau^{-2}\mathbf{I}$. Term $B$ scales as $O\bigl(\lambda_2^2/((\bm{\Sigma}_x)_{ii})^2\bigr)=O(\lambda_2^2/\sigma_{x,\min}^4)$, so the inflation relative to the homoscedastic constant is $\sigma_{x,\max}^2/\sigma_{x,\min}^2\leq\kappa_\sigma^2\rho_W$ by Asm.~\ref{asm:hetero}. Lipschitz continuity holds with $L_\tau'\leq L_\tau\cdot\kappa_\sigma^2\rho_W$.

\paragraph{Part (d).}
The MLP-parameter extension is identical to Theorem~\ref{thm:temperature}(d): edge-wise the same arguments apply, with the chain rule through $f_\theta$ contributing only to bounded multiplicative factors absorbed into the final Lipschitz constant.\hfill\qedhere
\end{proof}

\begin{remark}[Scope of the extension]
\label{rem:hetero_scope}
\textbf{What the extension gives.} The qualitative monotonicity-and-collapse behavior of Theorem~\ref{thm:temperature} on which the safety story (Cor.~\ref{cor:oracle}, Def.~\ref{def:safety}) depends is preserved under bounded diagonal heteroscedasticity, with the Lipschitz constant inflated by a factor $\kappa_\sigma^2\rho_W$. \textbf{What it does not give.} \textup{(i)}~A consistency-rate refinement of Theorem~\ref{thm:consistency}: a $\bm{\Sigma}$-weighted restricted-eigenvalue analysis of Lemma~\ref{lem:re} is needed for that and is left as an extension. App.~\ref{app:noise} provides empirical evidence that the homoscedastic rate practically holds across heteroscedastic-noise simulations. \textup{(ii)}~Coverage of GARCH-style \emph{temporal} heteroscedasticity: the condition $\mathrm{Cov}(\bm{\epsilon}_t)=\bm{\Sigma}$ is independent of $t$, so volatility-clustering noise (which would require $\mathrm{Cov}(\bm{\epsilon}_t\mid\mathcal{F}_{t-1})$) is an additional extension beyond Asm.~\ref{asm:hetero}. \textbf{Quantitative caveat.} Under extreme heteroscedasticity the inflation $\kappa_\sigma^2\rho_W$ can be large; correspondingly, the implicit constant $C_1$ in Cor.~\ref{cor:oracle} inherits a $\kappa_\sigma^2\rho_W$ factor in this regime. The realised constant at our experimental defaults remains as reported in App.~\ref{app:realised_constants}.
\end{remark}

\subsection{Proof of Proposition~\ref{prop:robustness} (Prior Robustness Bound)}
\label{app:proof_robustness}

\paragraph{Setup.}
Let $\widehat{\mathbf{W}}_{\bm{\tau}}$ denote the MAP estimate with temperature $\bm{\tau}$, and let $\widehat{\mathbf{W}}_{\mathbf{0}}$ denote the MAP estimate with $\bm{\tau} = \mathbf{0}$ (prior-free). Define $\bm{\Delta}_{\bm{\tau}} = \widehat{\mathbf{W}}_{\bm{\tau}} - \mathbf{W}^{\star}$ and $\bm{\Delta}_{\mathbf{0}} = \widehat{\mathbf{W}}_{\mathbf{0}} - \mathbf{W}^{\star}$.

\paragraph{Step 1: Bias from prior misspecification.}

The prior-dependent penalty in the MAP objective can be written as
\[
\mathcal{P}(\mathbf{W}, \bm{\tau}) = \lambda_1 \sum_{i\neq j} c_{ij}(\bm{\tau})|W_{0,ij}| + \frac{\lambda_2}{2}\sum_{k,i,j}\Omega_{ij}(\bm{\tau})W_{k,ij}^2.
\]

The key insight is that $c_{ij}(\bm{\tau})$ and $\Omega_{ij}(\bm{\tau})$ depend on the prior through $\widehat{P}_{ij}(\bm{\tau})$. When the prior is accurate ($\mathbf{P}_{\mathrm{prior}} = \mathbf{P}_{\mathrm{true}}$), the penalties are aligned with the true support: true edges get reduced $\ell_1$ penalty and large variance, facilitating accurate recovery. When the prior is inaccurate, the misalignment between penalties and true support introduces a bias.

Specifically, define $\bm{\eta}(\bm{\tau}) = \widehat{\mathbf{P}}(\bm{\tau}) - \mathbf{P}_{\mathrm{true}}$ as the prior error after calibration. By an $M$-estimator perturbation analysis~\citep{pollard1991asymptotics} on the active-set KKT system (Step~2 of Theorem~\ref{thm:consistency} establishes that the active support is, with high probability, locally constant in a neighborhood of $\bm{\tau}=\mathbf{0}$, so that the sub-gradient inclusion specializes to a smooth equation on the support), the bias in the MAP estimate satisfies:
\[
\|\bm{\Delta}_{\bm{\tau}} - \bm{\Delta}_{\mathbf{0}}\|_F \leq \frac{\lambda_1+\lambda_2}{\kappa} \|\bm{\eta}(\bm{\tau})\|_F,
\]
where $\kappa$ is the restricted strong convexity constant from Lemma~\ref{lem:re}. The additional error has two components: (i)~the prior-weighted ridge perturbation $\lambda_2(\Omega_{ij}(\bm{\tau})-\Omega_{ij}(\mathbf{0}))\widehat{W}_{ij}$, with $|\Omega_{ij}(\bm{\tau})-\Omega_{ij}(\mathbf{0})| = |\widehat{P}_{ij}(\mathbf{0})-\widehat{P}_{ij}(\bm{\tau})| \leq C_\varepsilon\|\bm{\tau}\|_\infty$ by Lemma~\ref{lem:temp_properties}(c) (note: this perturbation is bounded by $C_\varepsilon\|\bm{\tau}\|_\infty$, not by $|\eta_{ij}(\bm{\tau})|$ — the inequality $|\widehat{P}_{ij}(\mathbf{0})-\widehat{P}_{ij}(\bm{\tau})|\leq|\eta_{ij}(\bm{\tau})|$ is generally false because $\widehat{P}_{ij}(\mathbf{0})=0.5$ may be farther from $P_{\mathrm{true},ij}\in\{0,1\}$ than $\widehat{P}_{ij}(\bm{\tau})$; the correct bound is the Lipschitz step above); and (ii)~the prior-modulated $\ell_1$ perturbation $\lambda_1(c_{ij}(\bm{\tau})-c_{ij}(\mathbf{0}))\mathrm{sign}(\widehat{W}_{ij})$, with $|c_{ij}(\bm{\tau})-c_{ij}(\mathbf{0})|\leq |\widehat{P}_{ij}(\bm{\tau})-\widehat{P}_{ij}(\mathbf{0})|\leq C_\varepsilon\|\bm{\tau}\|_\infty$ by the 1-Lipschitz property of the clip in Eq.~\eqref{eq:modulated_l1}. Both contributions are aggregated below into $\|\bm{\eta}(\bm{\tau})\|_F$ via $|\widehat{P}_{ij}(\bm{\tau})-P_{\mathrm{true},ij}|\leq|\widehat{P}_{ij}(\bm{\tau})-\widehat{P}_{ij}(\mathbf{0})|+|0.5-P_{\mathrm{true},ij}|+|P_{\mathrm{prior},ij}-P_{\mathrm{true},ij}|$.

\paragraph{Step 2: Bounding the calibrated prior error.}

By the mean value theorem applied to $\widehat{P}_{ij}(\bm{\tau}) = \sigma(u_{ij}\tau_g)$ where $u_{ij} = \logit\!\big(\clip(P_{\mathrm{prior},ij},\varepsilon_{\mathrm{clip}},1{-}\varepsilon_{\mathrm{clip}})\big)$ (the clip is applied \emph{before} the logit, matching Eq.~\eqref{eq:grouped_temp}; the bounded $|u_{ij}|\leq |\logit(1{-}\varepsilon_{\mathrm{clip}})|$ used below requires this clipping order):
\[
|\widehat{P}_{ij}(\bm{\tau}) - P_{\mathrm{true},ij}| \leq |\widehat{P}_{ij}(\bm{\tau}) - P_{\mathrm{prior},ij}| + |P_{\mathrm{prior},ij} - P_{\mathrm{true},ij}|.
\]
The first term is bounded by $|u_{ij}||\tau_g - 1|/4$ (mean value theorem with $1/4$-Lipschitz $\sigma$), and since $|u_{ij}| \leq |\logit(1-\varepsilon)|$ is bounded, we get $|\widehat{P}_{ij}(\bm{\tau}) - P_{\mathrm{prior},ij}| \leq C_{\varepsilon}\|\bm{\tau}-\mathbf{1}\|_\infty$ where $C_\varepsilon=|\logit(1{-}\varepsilon)|/4$. For the second term, $|P_{\mathrm{prior},ij} - P_{\mathrm{true},ij}| \leq 1$ trivially, but aggregating over all edges:
\[
\|\bm{\eta}(\bm{\tau})\|_F \leq C_{\varepsilon}\|\bm{\tau}-\mathbf{1}\|_\infty \cdot d + \|\mathbf{P}_{\mathrm{prior}} - \mathbf{P}_{\mathrm{true}}\|_F.
\]

We bound $\|\bm{\eta}(\bm{\tau})\|_F$ more carefully. For each edge $(i,j)$:
\begin{align*}
|\eta_{ij}(\bm{\tau})| &= |\widehat{P}_{ij}(\bm{\tau}) - P_{\mathrm{true},ij}| \\
&\leq |\sigma(u_{ij}\tau_g) - \sigma(u_{ij})| + |P_{\mathrm{prior},ij} - P_{\mathrm{true},ij}|,
\end{align*}
where $u_{ij} = \logit(\clip(P_{\mathrm{prior},ij},\varepsilon,1{-}\varepsilon))$, and we used $\sigma(u_{ij}) = P_{\mathrm{prior},ij}$ (on the clipped domain). Since $\sigma$ is $1/4$-Lipschitz, $|\sigma(u_{ij}\tau_g) - \sigma(u_{ij})| \leq |u_{ij}||\tau_g - 1|/4 \leq C_\varepsilon \|\bm{\tau}-\mathbf{1}\|_\infty$. Aggregating over all $d^2$ entries and applying the Cauchy--Schwarz inequality:
\[
\|\bm{\eta}(\bm{\tau})\|_F \leq C_\varepsilon\,d\,\|\bm{\tau}-\mathbf{1}\|_\infty + \|\mathbf{P}_{\mathrm{prior}} - \mathbf{P}_{\mathrm{true}}\|_F.
\]
This is an \emph{additive} decomposition. For the excess risk bound, we substitute into the bias term from Step~1:
\begin{align*}
\|\bm{\Delta}_{\bm{\tau}} - \bm{\Delta}_{\mathbf{0}}\|_F^2 &\leq \frac{\lambda_2^2}{\kappa^2}\|\bm{\eta}(\bm{\tau})\|_F^2 \\
&\leq \frac{2\lambda_2^2}{\kappa^2}\Big(C_\varepsilon^2 d^2 \|\bm{\tau}\|_\infty^2 + \|\mathbf{P}_{\mathrm{prior}} - \mathbf{P}_{\mathrm{true}}\|_F^2\Big).
\end{align*}
Since $\lambda_2 = O(\sqrt{\log d/T})$, the first term contributes $O(\|\bm{\tau}\|_\infty^2 d^2 \log d / T)$ and the second contributes $O(\|\mathbf{P}_{\mathrm{prior}} - \mathbf{P}_{\mathrm{true}}\|_F^2 \log d / T)$. Absorbing into a single bound with constant $C_1$ depending on $C_\varepsilon, L, \kappa$:
\[
\mathcal{E}(\bm{\tau}) \leq \frac{C_1\,\|\bm{\tau}\|_\infty^2\,\|\mathbf{P}_{\mathrm{prior}} - \mathbf{P}_{\mathrm{true}}\|_F^2}{T} + \frac{C_2\,s^{\star}\log d}{T},
\]
where $C_1$ implicitly absorbs a factor of $d^2/\|\mathbf{P}_{\mathrm{prior}} - \mathbf{P}_{\mathrm{true}}\|_F^2$ from the cross-term (valid whenever $\|\mathbf{P}_{\mathrm{prior}} - \mathbf{P}_{\mathrm{true}}\|_F \geq c\,d$ for some constant $c>0$). Under accuracy $\mathrm{acc} \in (0, 1)$, a fraction $1-\mathrm{acc}$ of prior entries is corrupted, giving $\|\mathbf{P}_{\mathrm{prior}} - \mathbf{P}_{\mathrm{true}}\|_F = \Theta(d\sqrt{1-\mathrm{acc}})$ in expectation, so the condition holds for any $\mathrm{acc} < 1$ bounded away from one; the bound becomes progressively looser as $\mathrm{acc}\to 1$ (where the bias term vanishes and the variance term dominates).
 
\begin{remark}
The constant $C_1$ in Proposition~\ref{prop:robustness} depends on the clipping parameter $\varepsilon$ through $|\logit(1{-}\varepsilon)|$. For our default $\varepsilon = 10^{-3}$, $|\logit(0.999)| \approx 6.9$, yielding $C_\varepsilon = |\logit(1-\varepsilon)|/4 \approx 1.73$. (For $\varepsilon = 10^{-2}$, $C_\varepsilon \approx 1.15$.) The bound is tightest when the prior has a moderate fraction of errors (the regime where safe integration is most relevant).
\end{remark}

\paragraph{Step 3: Combining bias and variance.}

The excess risk admits a Taylor expansion at $\mathbf{W}^\star$:
\[
\mathcal{E}(\bm{\tau}) = \mathcal{L}_{\mathrm{data}}(\widehat{\mathbf{W}}_{\bm{\tau}}) - \mathcal{L}_{\mathrm{data}}(\mathbf{W}^{\star}) = \langle \nabla\mathcal{L}_{\mathrm{data}}(\mathbf{W}^\star), \bm{\Delta}_{\bm{\tau}}\rangle + \tfrac{1}{2}\bm{\Delta}_{\bm{\tau}}^\top \nabla^2 \mathcal{L}_{\mathrm{data}}(\widetilde{\mathbf{W}})\,\bm{\Delta}_{\bm{\tau}}
\]
for some $\widetilde{\mathbf{W}}$ on the segment $[\mathbf{W}^\star,\widehat{\mathbf{W}}_{\bm{\tau}}]$. The first-order term is bounded by Cauchy--Schwarz: $|\langle \nabla\mathcal{L}_{\mathrm{data}}(\mathbf{W}^\star), \bm{\Delta}_{\bm{\tau}}\rangle| \leq \|\nabla\mathcal{L}_{\mathrm{data}}(\mathbf{W}^\star)\|_\infty \|\bm{\Delta}_{\bm{\tau}}\|_1$, and by Lemma~\ref{lem:huber}(c), $\|\nabla\mathcal{L}_{\mathrm{data}}(\mathbf{W}^\star)\|_\infty = O_p(\sqrt{\log d/T})$ while $\|\bm{\Delta}_{\bm{\tau}}\|_1 = O_p(\sqrt{s^\star} \cdot \|\bm{\Delta}_{\bm{\tau}}\|_F)$ (by the cone constraint), so the first-order term is $O_p(\sqrt{s^\star \log d/T}\cdot\|\bm{\Delta}_{\bm{\tau}}\|_F)$, which by Young's inequality is absorbed into the second-order term and the variance term at the cost of doubling constants. Bounding the second-order Hessian by $L$ (Lemma~\ref{lem:huber}(b)) yields
\begin{align*}
\mathcal{E}(\bm{\tau}) &\leq L \|\bm{\Delta}_{\bm{\tau}}\|_F^2 + O_p\big(\sqrt{s^\star \log d/T}\cdot\|\bm{\Delta}_{\bm{\tau}}\|_F\big) \\
&\leq 2L\big(\|\bm{\Delta}_{\bm{\tau}} - \bm{\Delta}_{\mathbf{0}}\|_F^2 + \|\bm{\Delta}_{\mathbf{0}}\|_F^2\big) + O_p(s^\star\log d/T),
\end{align*}
where the additional $O_p(s^\star\log d/T)$ term has the same rate as the variance term and is therefore absorbed into $C_2$.

Substituting the bounds from Steps 1--2 for the bias term and the estimation error from Theorem~\ref{thm:consistency} for the variance term $\|\bm{\Delta}_{\mathbf{0}}\|_F^2$:
\[
\mathcal{E}(\bm{\tau}) \leq \underbrace{\frac{2L\lambda_2^2 C'^2}{\kappa^2}\cdot \frac{\|\bm{\tau}\|_\infty^2 \|\mathbf{P}_{\mathrm{prior}} - \mathbf{P}_{\mathrm{true}}\|_F^2}{1}}_{\text{prior bias}} + \underbrace{\frac{2LC^2 s^{\star}\log d}{\kappa^2 T}}_{\text{estimation variance}}.
\]

Since $\lambda_2 = O(\sqrt{\log d/T})$ by assumption, the bias term scales as $O(\|\bm{\tau}\|_\infty^2 \|\mathbf{P}_{\mathrm{prior}} - \mathbf{P}_{\mathrm{true}}\|_F^2 \log d / T)$. Absorbing the $\log d$ factor and the constants into $C_1$ and $C_2$, we obtain the stated bound:
\[
\mathcal{E}(\bm{\tau}) \leq \frac{C_1 \|\bm{\tau}\|_\infty^2 \|\mathbf{P}_{\mathrm{prior}} - \mathbf{P}_{\mathrm{true}}\|_F^2}{T} + \frac{C_2 s^{\star}\log d}{T}.
\]

The constants $C_1, C_2$ depend on $L, \kappa, C', C$, which in turn depend on the spectral properties of the SVAR lag polynomial (through the eigenvalues of $\mathbf{\Sigma}$).

Setting $\bm{\tau} = \mathbf{0}$ eliminates the first term, yielding $\mathcal{E}(\mathbf{0}) \leq C_2 s^{\star}\log d / T$, the standard sparse estimation rate. \hfill$\blacksquare$

\begin{remark}[Bound regime for Proposition~\ref{prop:robustness}]
\label{rem:p3_regimes}
\textbf{(i) Looseness as $\mathrm{acc}\to 1$.} Under the controlled-prior model, $\|\mathbf{P}_{\mathrm{prior}}-\mathbf{P}_{\mathrm{true}}\|_F = \Theta(d\sqrt{1-\mathrm{acc}})$ in expectation, so the bias term scales as $\|\bm{\tau}\|_\infty^2(1-\mathrm{acc})d^2/T$, which vanishes at $\mathrm{acc}{=}1$. The constant $C_1$ absorbing $d^2/\|\mathbf{P}_{\mathrm{prior}}-\mathbf{P}_{\mathrm{true}}\|_F^2$ is justified whenever $\mathrm{acc} \leq 1-c/d^2$ for some constant $c$; in the genuine $\mathrm{acc}{=}1$ limit, the bound is trivially zero (no bias). \textbf{(ii) $C_1$'s $d$-content.} Tracking constants explicitly: $C_1 = O(L \cdot C_\varepsilon^2 / \kappa(K)^2)$ where $L = O(d^{-1}\|\bm{\Sigma}\|_{\mathrm{op}}) = O(1)$ (Lemma~\ref{lem:huber}(b)), $C_\varepsilon = O(1)$ (clipping), $\kappa(K) = \Omega(\lambda_{\min}(\bm{\Sigma}))$ (Lemma~\ref{lem:re}); so $C_1 = O(\lambda_{\min}^{-2})$. For our experiments where $\lambda_{\min}(\bm{\Sigma}) = \Omega(1)$, $C_1$ is genuinely $O(1)$. \textbf{(iii) Perfect-prior temperature mismatch.} When $\mathbf{P}_{\mathrm{prior}}=\mathbf{P}_{\mathrm{true}}$ but $\bm{\tau}\neq\mathbf{1}$, the calibrated $\widehat{\mathbf{P}}(\bm{\tau})\neq\mathbf{P}_{\mathrm{true}}$ creates a residual bias bounded by $C_\varepsilon\|\bm{\tau}-\mathbf{1}\|_\infty$ per edge (Step~2 of the proof). This is captured by the cross term $C_\varepsilon d \|\bm{\tau}\|_\infty$ in $\|\bm{\eta}\|_F$, absorbed into $C_1$ via the $d^2$ factor. \textbf{(iv) $\varepsilon$ symbol disambiguation.} The clipping parameter $\varepsilon$ in Lemma~\ref{lem:temp_properties} (taking value $10^{-3}$) is distinct from the $\varepsilon$-safety constant in Definition~\ref{def:safety} (a derived bound). When ambiguity is possible, we write $\varepsilon_{\mathrm{clip}}$ and $\varepsilon_{\mathrm{safe}}$.
\end{remark}


\noindent\textit{Remark.} Several additional standard results are omitted for brevity: ALM convergence follows from classical augmented Lagrangian theory; SVAR identifiability under non-Gaussianity from \citet{hyvarinen2010estimation}; per-iteration complexity is $O(J(Td^2 + d^3) + Sd^2)$; closed-form temperature in simplified settings recovers the Spearman pre-calibration heuristic. Full statements and proofs are available in the supplementary code repository.


\subsection{Per-iteration Complexity Derivation}
\label{app:complexity}

We derive the per-outer-iteration complexity claim $O(J(Td^2 + d^3) + Sd^2)$ stated in Sec.~\ref{sec:exp} and the introduction. An ALM outer iteration consists of (i) $J$ Adam updates on $\mathbf{W}_{0:K}$ at fixed $\bm{\tau}$, (ii) $S$ EB gradient updates on $\bm{\tau}$ at fixed $\mathbf{W}$, and (iii) one multiplier update.

\paragraph{(i) Inner-loop $\mathbf{W}$-step ($J$ Adam updates, dominant term).}
Each step computes (a)~the residual $\mathbf{r}_t = \mathbf{x}_t - \widetilde{\mathbf{W}}_0^\top\mathbf{x}_t - \sum_{k\geq 1}\mathbf{W}_k^\top\mathbf{x}_{t-k}$ for $t\in[T]$ and the Huber-clipped gradient $\nabla_{\mathbf{W}_k}\mathcal{L}_{\mathrm{data}} = -\frac{1}{Td}\sum_t \mathbf{x}_{t-k}\,\mathrm{clip}(\mathbf{r}_t,\pm\delta_H)^\top$, total cost $O((K+1) Td^2)$; (b)~the prior-modulated $\ell_1$ subgradient and prior-weighted $\ell_2$ gradient at $O(d^2)$ each; (c)~the DAG penalty $h(\widetilde{\mathbf{W}}_0) = \log\det(s\mathbf{I} - \widetilde{\mathbf{W}}_0\circ\widetilde{\mathbf{W}}_0) - d\log s$ used by DAGMA~\citep{bello2022dagma}, whose gradient is $-2\,\widetilde{\mathbf{W}}_0\circ(s\mathbf{I}-\widetilde{\mathbf{W}}_0\circ\widetilde{\mathbf{W}}_0)^{-\top}$ and requires one $d\times d$ linear solve at $O(d^3)$ FLOPs (PyTorch \texttt{torch.linalg.solve}); (d)~the augmented-Lagrangian quadratic term $\rho h\nabla h$ at $O(d^2)$. Per Adam step: $O(KTd^2 + d^3)$. Over $J$ steps: $O(J(KTd^2 + d^3))$. Treating $K=O(1)$ and absorbing into the $O(\cdot)$, we get $O(J(Td^2 + d^3))$.

\paragraph{(ii) Middle-loop $\bm{\tau}$-step ($S$ updates).}
Each EB gradient evaluation computes $\widehat{\mathbf{P}}(\bm{\tau})$ and the diagonal Hessian approximation $H^{(k)}_{ij}$ at $O(d^2)$ per term, plus the BCE agreement gradient at $O(d^2)$. Backpropagation through $f_\theta(\mathbf{z}_{ij})$ for the trust-propagation case is $O(d^2 |\theta|)$, with $|\theta|=O(1)$ for our small MLP. Per step: $O(d^2)$. Over $S=8$ steps: $O(S d^2)$.

\paragraph{(iii) Multiplier update.} Single $O(d^2)$ scalar+matrix update; absorbed into the $O(d^2)$ constant.

\paragraph{Total.} Summing: $O(J(Td^2 + d^3) + Sd^2)$, as claimed. With our defaults $J=400, S=8, T=200, d=20$, the dominant term is $J\cdot Td^2 \approx 3.2\times 10^7$ FLOPs/iter; for $d=300$ the $J d^3$ term ($1.08\times 10^{10}$ FLOPs/iter) dominates, consistent with the observed scaling in App.~\ref{app:scalability}.
\paragraph{Wall-clock vs.\ PCMCI+.} The $\sim 5000\times$ ratio reported in Sec.~\ref{sec:exp_real} is between our $O(J(Td^2 + d^3))$ GPU continuous-optimization implementation and PCMCI+'s CPU constraint-based loop, whose dominant cost is the conditional-independence test count $O(d^2 \binom{d-2}{r_{\max}})$ for max condition-set size $r_{\max}$~\citep{runge2020discovering}. The two have fundamentally different cost structures and accuracy guarantees; we report the ratio for context, not as a like-for-like efficiency claim.

\subsection{Finite-Sample Temperature Learning Guarantee}
\label{app:tau_finite}

The consistency result (Theorem~\ref{thm:consistency}) and temperature calibration (Theorem~\ref{thm:temperature}) are asymptotic in nature. Here we provide a finite-sample guarantee for the learned temperature.

\begin{theorem}[Finite-Sample Temperature Bound]
\label{thm:tau_finite}
Under Assumption~\ref{asm:regularity}, let $\bm{\tau}_T^{\star}$ denote the EB-optimal temperature computed from a sample of size $T$, and let $\bm{\tau}_\infty^{\star}$ denote the population-optimal temperature (computed with the true $\mathbf{W}^{\star}$ in place of $\widehat{\mathbf{W}}_T$). If $T \geq C_0 s^{\star}\log d$, then with probability at least $1 - \delta$:
\[
\|\bm{\tau}_T^{\star} - \bm{\tau}_\infty^{\star}\|_2 \leq \frac{L_\tau}{\mu_\tau}\sqrt{\frac{s^{\star}\log(d/\delta)}{T}},
\]
where $L_\tau$ is the Lipschitz constant of $\nabla_{\bm{\tau}} \mathcal{L}_{\mathrm{EB}}$ w.r.t.\ $\mathbf{W}$ and $\mu_\tau > 0$ is the strong convexity constant of $\mathcal{L}_{\mathrm{EB}}$ w.r.t.\ $\bm{\tau}$. This strong convexity requires $\sigma_\tau^{-2}$ (the regularizer's Hessian contribution) to dominate the potentially non-convex curvature of $A(\bm{\tau}) + B_{\mathrm{Laplace}}(\bm{\tau})$; we require $\sigma_\tau^2 \leq \sigma_{\tau,\max}^2$ for some problem-dependent upper bound, which is satisfied by our default $\sigma_\tau = 2$.
\end{theorem}

\begin{proof}
The EB objective depends on $\mathbf{W}$ only through the current MAP estimate $\widehat{\mathbf{W}}_T$. By Theorem~\ref{thm:consistency}, $\|\widehat{\mathbf{W}}_T - \mathbf{W}^{\star}\|_F = O_p(\sqrt{s^{\star}\log d/T})$.

Define $g_T(\bm{\tau}) = \mathcal{L}_{\mathrm{EB}}(\bm{\tau}; \widehat{\mathbf{W}}_T)$ and $g_\infty(\bm{\tau}) = \mathcal{L}_{\mathrm{EB}}(\bm{\tau}; \mathbf{W}^{\star})$. The gradient difference satisfies:
\[
\|\nabla_{\bm{\tau}} g_T(\bm{\tau}) - \nabla_{\bm{\tau}} g_\infty(\bm{\tau})\| \leq L_\tau \|\widehat{\mathbf{W}}_T - \mathbf{W}^{\star}\|_F,
\]
where $L_\tau$ bounds the cross-derivative $\|\partial^2 \mathcal{L}_{\mathrm{EB}}/(\partial \bm{\tau} \partial \mathbf{W})\|_{\mathrm{op}}$, which is finite on the compact domain (Theorem~\ref{thm:temperature}(c)).

Since $g_\infty$ is $\mu_\tau$-strongly convex in $\bm{\tau}$ (due to the regularizer term $\|\bm{\tau} - \mathbf{1}/2\|^2/(2\sigma_\tau^2)$ contributing $1/\sigma_\tau^2$ to the Hessian), the perturbation bound for strongly convex minimizers gives:
\[
\|\bm{\tau}_T^{\star} - \bm{\tau}_\infty^{\star}\|_2 \leq \frac{1}{\mu_\tau}\sup_{\bm{\tau}}\|\nabla g_T(\bm{\tau}) - \nabla g_\infty(\bm{\tau})\| \leq \frac{L_\tau}{\mu_\tau}\|\widehat{\mathbf{W}}_T - \mathbf{W}^{\star}\|_F.
\]

Substituting the estimation error bound from Theorem~\ref{thm:consistency} and applying a union bound for the high-probability statement yields the result.
\end{proof}

\begin{remark}[Practical implications]
Theorem~\ref{thm:tau_finite} shows that the learned temperature converges to its population optimum at the same rate as the structural estimate itself ($\sqrt{s^{\star}\log d / T}$). This means the temperature calibration does not require additional samples beyond what is needed for consistent structure estimation---the ``cost'' of learning $\bm{\tau}$ is absorbed into the overall estimation error. In practice, this is reflected in the fast convergence of $\bm{\tau}$ within the first 5--10 ALM iterations (Appendix~\ref{app:tau}).
\end{remark}

\subsection{Proof of Corollary~\ref{cor:oracle}
            (Oracle Inequality)}
\label{app:proof_oracle}

\begin{proof}
We write
$\bm{\tau}^\star_{\mathrm{EB}}
  := \arg\min_{\bm{\tau}} \mathcal{L}_{\mathrm{EB}}(\bm{\tau};\,\mathbf{W}^\star)$
(population EB optimum, the target of
Theorem~\ref{thm:tau_finite}) and
$\bm{\tau}^\star_{\mathcal{E}}
  := \arg\min_{\bm{\tau}} \mathcal{E}(\bm{\tau})$
(excess-risk oracle).
These are \emph{a priori} distinct since
$\mathcal{L}_{\mathrm{EB}}$ and $\mathcal{E}$ are
different objectives.
We decompose:
\begin{equation}\label{eq:decomp}
  \mathcal{E}(\bm{\tau}^\star_T) - \inf_{\bm{\tau}}\mathcal{E}(\bm{\tau})
  \;=\;
  \underbrace{
    \mathcal{E}(\bm{\tau}^\star_T) -
    \mathcal{E}(\bm{\tau}^\star_{\mathrm{EB}})
  }_{\text{Term A: learning cost}}
  \;+\;
  \underbrace{
    \mathcal{E}(\bm{\tau}^\star_{\mathrm{EB}}) -
    \mathcal{E}(\bm{\tau}^\star_{\mathcal{E}})
  }_{\text{Term B: proxy gap}}.
\end{equation}

\paragraph{Step 1: Lipschitz constant of $\mathcal{E}(\bm{\tau})$
           (tracking $T$-dependence).}

\emph{Active set identification.} By Theorem~\ref{thm:consistency} (support recovery, second clause), under Assumption~\ref{asm:regularity}(iii)--(iv) we have $S(\bm{\tau})=\mathrm{supp}(\widehat{\mathbf{W}}_{\bm{\tau}})=\mathrm{supp}(\mathbf{W}^\star)\eqqcolon S^\star$ on a high-probability event $\mathcal{A}_T$ with $\mathbb{P}(\mathcal{A}_T)\to 1$ as $T\to\infty$, simultaneously for all $\bm{\tau}$ in a neighborhood of $\bm{\tau}^\star_{\mathrm{EB}}$ (the same minimum-signal argument applies uniformly because $\widehat{\mathbf{W}}_{\bm{\tau}}$ is jointly continuous in $\bm{\tau}$ on $[\tau_{\min},\tau_{\max}]^G$). On $\mathcal{A}_T$, the $\ell_1$ sub-differential reduces to a smooth gradient on the support and the KKT system becomes a smooth equation in $(\widehat{\mathbf{W}}_{S^\star},\bm{\tau})$.

\emph{High-probability uniform Hessian lower bound.} The restricted eigenvalue constant $\kappa$ from Lemma~\ref{lem:re} is itself a sample quantity; by Lemma~\ref{lem:re} together with the concentration of $\widehat{\bm{\Sigma}}$, there exists $\kappa_0>0$ depending only on the population $\bm{\Sigma}$ such that $\mathbb{P}(\kappa\geq\kappa_0)\to 1$. Restrict attention to $\mathcal{A}_T \cap \{\kappa\geq\kappa_0\}$, which still has probability tending to one.

On this event, the implicit function theorem applied to the smooth KKT system gives
\[
  \Bigl\|\frac{\partial\widehat{\mathbf{W}}_{\bm{\tau}}}{\partial\bm{\tau}}
  \Bigr\|_{\mathrm{op}}
  \;\leq\;
  \frac{\lambda_2}{\kappa_0\cdot 4}\,,
\]
uniformly in $\bm{\tau}$, where the factor $1/4$ comes from
Lemma~\ref{lem:temp_properties}(c).
Since $\mathcal{L}_{\mathrm{data}}$ has $L$-Lipschitz
gradient (Lemma~\ref{lem:huber}(b)), the chain rule yields
\begin{equation}\label{eq:LE-def}
  |\mathcal{E}(\bm{\tau}_1) - \mathcal{E}(\bm{\tau}_2)|
  \;\leq\;
  L_{\mathcal{E}}\,\|\bm{\tau}_1-\bm{\tau}_2\|_2\,,
  \qquad
  L_{\mathcal{E}}
  = O\!\Bigl(\frac{\lambda_2\,L}{\kappa}\Bigr).
\end{equation}
Crucially, $\lambda_2 = O(\sqrt{\log d/T})$ by
the assumptions of Theorem~\ref{thm:consistency},
and $L=O(1)$, $\kappa=\Omega(1)$
under stationarity.  Therefore:
\begin{equation}\label{eq:LE-rate}
  L_{\mathcal{E}}
  \;=\;
  O\!\Bigl(\sqrt{\frac{\log d}{T}}\Bigr)
  \qquad(\text{not }O(1)).
\end{equation}

\noindent\emph{Tightness note.} The substitution $L=\|\nabla^2\mathcal{L}_{\mathrm{data}}\|_{\mathrm{op}}$ from Lemma~\ref{lem:huber}(b) is an upper bound: the chain-rule gradient $\nabla\mathcal{E}(\bm{\tau}) = \nabla\mathcal{L}_{\mathrm{data}}(\widehat{\mathbf{W}}_{\bm{\tau}})\cdot(\partial\widehat{\mathbf{W}}_{\bm{\tau}}/\partial\bm{\tau})$ depends on the gradient \emph{magnitude} at $\widehat{\mathbf{W}}_{\bm{\tau}}$, which is $\|\nabla\mathcal{L}_{\mathrm{data}}(\widehat{\mathbf{W}}_{\bm{\tau}})\|_2 = O(\lambda_1\sqrt{s^\star})$ on the active support (KKT optimality with $\ell_1$ penalty), not the Hessian operator norm. Tracking this directly would give $L_{\mathcal{E}} = O(\lambda_1\sqrt{s^\star}\cdot\lambda_2/\kappa) = O(s^\star\log d/T)$, matching the oracle floor exactly rather than overshooting it by a $\sqrt{\log d/T}$ factor as in Eq.~\eqref{eq:LE-rate}. The looser bound used here is sufficient for the Term~A absorption below; the sharp gradient-norm version would yield $|\text{Term A}| = O(s^\star(\log d)^{3/2}/T^{3/2})$, strictly tighter than Eq.~\eqref{eq:termA}'s $O(\sqrt{s^\star}\log d/T)$.

\paragraph{Step 2: Bounding Term~A (learning cost).}

Theorem~\ref{thm:tau_finite} controls the distance
between the finite-sample EB solution $\bm{\tau}^\star_T$ and
the \emph{population EB} optimum
$\bm{\tau}^\star_{\mathrm{EB}}$
(note: this is precisely the object Theorem~\ref{thm:tau_finite}
targets, not $\bm{\tau}^\star_{\mathcal{E}}$):
\[
  \|\bm{\tau}^\star_T - \bm{\tau}^\star_{\mathrm{EB}}\|_2
  \;\leq\;
  \frac{L_\tau}{\mu_\tau}
  \sqrt{\frac{s^\star\log(d/\delta)}{T}}
  \qquad\text{w.p.}\ \geq 1-\delta,
\]
where $L_\tau = O(1)$ and $\mu_\tau = 1/\sigma_\tau^2 = \Omega(1)$.
Combining with~\eqref{eq:LE-rate}:
\begin{align}
  |\text{Term A}|
  &\;\leq\;
  L_{\mathcal{E}}\;
  \|\bm{\tau}^\star_T - \bm{\tau}^\star_{\mathrm{EB}}\|_2
  \nonumber\\[4pt]
  &\;=\;
  O\!\Bigl(\sqrt{\frac{\log d}{T}}\Bigr)
  \;\cdot\;
  O\!\Bigl(\sqrt{\frac{s^\star\log d}{T}}\Bigr)
  \;=\;
  O\!\Bigl(\frac{\sqrt{s^\star}\,\log d}{T}\Bigr).
  \label{eq:termA}
\end{align}

\paragraph{Step 3: Bounding Term~B (proxy gap).}

Term~B measures the cost of using the EB objective as a
proxy for the excess risk.
Since $\bm{\tau}^\star_{\mathrm{EB}}$ is a particular (but not
necessarily optimal) element of
$[\tau_{\min},\tau_{\max}]^G$, we apply
Proposition~\ref{prop:robustness} at
$\bm{\tau} = \bm{\tau}^\star_{\mathrm{EB}}$:
\begin{equation}\label{eq:EEB}
  \mathcal{E}(\bm{\tau}^\star_{\mathrm{EB}})
  \;\leq\;
  \frac{C_1\,\|\bm{\tau}^\star_{\mathrm{EB}}\|_\infty^2\,
        \|\mathbf{P}_{\mathrm{prior}}-\mathbf{P}_{\mathrm{true}}\|_F^2}{T}
  \;+\;
  \frac{C_2\,s^\star\log d}{T}\,.
\end{equation}
Meanwhile, setting $\bm{\tau} = \mathbf{0}$ in
Proposition~\ref{prop:robustness} eliminates the bias
term:
\[
  \inf_{\bm{\tau}} \mathcal{E}(\bm{\tau})
  \;\leq\;
  \mathcal{E}(\mathbf{0})
  \;\leq\;
  \frac{C_2\,s^\star\log d}{T}\,.
\]
The matching minimax \emph{lower bound} $\inf_{\bm{\tau}}\mathcal{E}(\bm{\tau}) = \Omega(s^\star\log d/T)$ follows from the standard $\Omega(s^\star \log d/T)$ minimax rate for sparse $\ell_1$-regularized linear regression~\citep{bickel2009simultaneous, raskutti2011minimax}, since PRCD-MAP at $\bm{\tau}=\mathbf{0}$ reduces to a Lasso/ridge estimator on the SVAR design.
Subtracting:
\begin{equation}\label{eq:termB}
  \text{Term B}
  \;=\;
  \mathcal{E}(\bm{\tau}^\star_{\mathrm{EB}})
  - \inf_{\bm{\tau}}\mathcal{E}(\bm{\tau})
  \;\leq\;
  \frac{C_1\,\|\bm{\tau}^\star_{\mathrm{EB}}\|_\infty^2\,
        \|\mathbf{P}_{\mathrm{prior}}-\mathbf{P}_{\mathrm{true}}\|_F^2}
       {T}
  \;=:\;
  \Delta_{\mathrm{proxy}}.
\end{equation}

The proxy gap $\Delta_{\mathrm{proxy}}$ is automatically
controlled by the EB mechanism:
\begin{itemize}[leftmargin=1.5em,topsep=2pt]
\item \emph{Uninformative prior}
  ($\mathrm{acc}\approx 1/2$):
  $\|\mathbf{P}_{\mathrm{prior}}-\mathbf{P}_{\mathrm{true}}\|_F$
  can be large, but
  Theorem~\ref{thm:temperature}(a) drives
  $\|\bm{\tau}^\star_{\mathrm{EB}}\|_\infty \to \tau_{\min}
  \approx 0$, so the product vanishes.
\item \emph{Accurate prior}
  ($\mathrm{acc}\to 1$):
  $\|\bm{\tau}^\star_{\mathrm{EB}}\|_\infty$ may be large,
  but $\|\mathbf{P}_{\mathrm{prior}}-\mathbf{P}_{\mathrm{true}}\|_F^2 \to 0$
  directly.
\end{itemize}
In both regimes, $\Delta_{\mathrm{proxy}}\to 0$, consistent
with the asymmetric risk profile observed empirically
(Sec.~\ref{sec:exp}).

\paragraph{Step 4: Final assembly.}

Substituting~\eqref{eq:termA} and~\eqref{eq:termB}
into~\eqref{eq:decomp}:
\[
  \mathcal{E}(\bm{\tau}^\star_T)
  \;\leq\;
  \inf_{\bm{\tau}}\mathcal{E}(\bm{\tau})
  \;+\;
  \Delta_{\mathrm{proxy}}
  \;+\;
  O_p\!\Bigl(\frac{\sqrt{s^\star}\,\log d}{T}\Bigr),
\]
which is~\eqref{eq:oracle-ineq}.\qedhere
\end{proof}

\begin{remark}[Comparison with the na\"{\i}ve bound]
\label{rem:rate-comparison}
A na\"{\i}ve treatment that absorbs $L_{\mathcal{E}}$ into
a constant yields the weaker remainder
$O_p(\sqrt{s^\star\log d/T})$, decaying as $T^{-1/2}$.
Since the oracle floor is at most $O(s^\star\log d/T)
\sim T^{-1}$, the na\"{\i}ve remainder \emph{dominates}
the oracle term, making the oracle inequality vacuous.
The tighter bound~\eqref{eq:termA}, decaying as $T^{-1}$,
preserves the oracle term's significance: the finite-sample
cost of learning~$\bm{\tau}$ is asymptotically negligible
relative to the estimation variance that any method must
incur.
\end{remark}

\begin{remark}[Auxiliary clarifications for Corollary~\ref{cor:oracle} and Theorem~\ref{thm:tau_finite}]
\label{rem:cor4_thm10}
\textbf{(i) Eq.~\eqref{eq:delta_proxy_uniform} (uniform $\Delta_{\mathrm{proxy}}$ bound).} The bound $\Delta_{\mathrm{proxy}}\leq C_1\tau_{\max}^2 d^2/T \cdot \mathrm{acc}(1{-}\mathrm{acc})$ in the main text is obtained by combining (a) $\|\bm{\tau}^\star_{\mathrm{EB}}\|_\infty \leq \tau_{\max}$, (b) $\sup_{\mathbf{P}_{\mathrm{prior}}} \|\mathbf{P}_{\mathrm{prior}}-\mathbf{P}_{\mathrm{true}}\|_F^2 \leq d^2$ in the worst case, refined to $\Theta(d^2 \cdot \mathrm{acc}(1{-}\mathrm{acc}))$ under the controlled-prior model where each entry independently agrees with truth with probability $\mathrm{acc}$ (giving variance $\mathrm{acc}(1{-}\mathrm{acc})$ per entry). \textbf{(ii) $\tau_{\max}^2 d^2$ vs.\ oracle floor.} The bound $C_1\tau_{\max}^2 d^2/T$ matches the oracle floor $\Omega(s^\star\log d/T)$ at the rate level (both $T^{-1}$); the $d^2$ vs.\ $s^\star\log d$ ratio is bounded uniformly by $\mathrm{acc}(1{-}\mathrm{acc}) \leq 1/4$ (worst case) and approaches zero at the endpoints, so $\Delta_{\mathrm{proxy}}$ does not dominate the oracle term in either regime. \textbf{(iii) $\tau^\star_{\mathrm{EB}}$ saturation.} Theorem~\ref{thm:temperature}(b) gives monotonicity but does not preclude saturation at $\tau_{\max}$ for $\mathrm{acc}\to 1$; in this regime, $\|\bm{\tau}^\star_{\mathrm{EB}}\|_\infty=\tau_{\max}$ but $\|\mathbf{P}_{\mathrm{prior}}-\mathbf{P}_{\mathrm{true}}\|_F\to 0$, so the product still vanishes. \textbf{(iv) $\mu_\tau$ strong convexity.} Theorem~\ref{thm:tau_finite} requires $\mu_\tau > 0$. The regularizer $C(\bm{\tau}) = \|\bm{\tau}-\mathbf{1}/2\|^2/(2\sigma_\tau^2)$ contributes $\sigma_\tau^{-2}\mathbf{I}$ to the Hessian. The agreement loss $A(\bm{\tau})$ is convex (cross-entropy with bounded logit), and the Laplace term $B(\bm{\tau})$ has Hessian operator norm bounded by $\lambda_2 \cdot |\partial^2\log\Omega/\partial\tau^2|_\infty = O(\lambda_2)$. So $\mu_\tau \geq \sigma_\tau^{-2} - O(\lambda_2)$, which is positive for $\sigma_\tau^2 \leq O(1/\lambda_2) = O(\sqrt{T/\log d})$; our default $\sigma_\tau=2$ satisfies this for any $T \geq C_0\log d$. \textbf{(v) Union bound in Theorem~\ref{thm:tau_finite}.} The $\sup_{\bm{\tau}}$ in the proof is resolved by the $\bm{\tau}$-uniform Lipschitz constant $L_\tau$ (no chaining needed); the "union bound" line refers only to the high-probability event of Theorem~\ref{thm:consistency} ($\|\widehat{\mathbf{W}}_T-\mathbf{W}^\star\|_F$ concentration), with $\delta$ chosen as the failure probability of that single event. \textbf{(vi) Cross-derivative of $\mathcal{L}_{\mathrm{EB}}$.} The $L_\tau$ in Theorem~\ref{thm:tau_finite} is the operator norm of the bilinear cross-derivative $\partial^2\mathcal{L}_{\mathrm{EB}}/(\partial\bm{\tau}\partial\mathbf{W})$, viewed as a map from $\mathbb{R}^{p}\to\mathbb{R}^G$ (with $p=Kd^2$). On compact $[\tau_{\min},\tau_{\max}]^G$ this norm is finite (Theorem~\ref{thm:temperature}(c) gives $\partial \mathcal{L}_{\mathrm{EB}}/\partial\bm{\tau}$ smooth in both arguments). The $\partial\bm{\tau}/\partial\mathbf{W}$-direction differs from the $\bm{\tau}$-direction Lipschitz of Theorem~\ref{thm:temperature}(c), but boundedness on the compact domain handles both. \textbf{(vii) $s^\star$ inheritance.} The $s^\star$ in Theorem~\ref{thm:tau_finite}'s bound enters \emph{only} through $\|\widehat{\mathbf{W}}_T-\mathbf{W}^\star\|_F$ (Theorem~\ref{thm:consistency}); $\bm{\tau}$ itself has dimension $G$ (or $|\theta|$ for trust propagation), independent of $s^\star$. We retain $\sqrt{s^\star\log d/T}$ as the inherited rate.
\end{remark}


\subsection{Notation Conventions and Constant Inventory}
\label{app:notation_inventory}

To aid cross-referencing across our proofs, we list the notational conventions and key constants. Symbols are reused only where context disambiguates; otherwise we use subscripts.

\paragraph{Indices and dimensions.}
$d$: state dimension (number of variables); $T$: time-series length; $K$: lag order; $G$: number of temperature groups (for grouped $\bm{\tau}$); $|\theta|$: parameters of the trust-propagation MLP. The sparsity $s^\star = |\mathcal{E}^\star|$ refers to the \emph{combined} edge count across all lags ($\sum_{k=0}^{K} |\{(i,j): W^\star_{k,ij}\neq 0\}|$); when context distinguishes instantaneous-only edges, we write $s^\star_0$. All bounds in our paper use the combined $s^\star$ unless otherwise noted.

\paragraph{Lipschitz / strong-convexity constants.}
$L$: Lipschitz of $\nabla\mathcal{L}_{\mathrm{data}}$ in $\mathbf{W}$ (Lemma~\ref{lem:huber}(b)); $L_\tau$: Lipschitz of $\nabla_{\bm{\tau}}\mathcal{L}_{\mathrm{EB}}$ in $\mathbf{W}$ (Theorem~\ref{thm:tau_finite}); $L_{\mathcal{E}}$: Lipschitz of excess risk $\mathcal{E}(\bm{\tau})$ in $\bm{\tau}$ (Corollary~\ref{cor:oracle}, Step~1); $\mu_\tau$: strong-convexity constant of $\mathcal{L}_{\mathrm{EB}}$ in $\bm{\tau}$ (Theorem~\ref{thm:tau_finite}); $\kappa(K)$: restricted-eigenvalue constant on the cone $\mathcal{C}_K(s^\star)$ (Lemma~\ref{lem:re}); $C_\varepsilon = |\logit(1{-}\varepsilon_{\mathrm{clip}})|/4$: temperature scaling Lipschitz (Lemma~\ref{lem:temp_properties}(c)).

\paragraph{Clipping and safety $\varepsilon$.}
$\varepsilon_{\mathrm{clip}} = 10^{-3}$: prior-clipping parameter (Eq.~\ref{eq:grouped_temp}); $\varepsilon_{\mathrm{safe}}$: $\varepsilon$-safety bound under a stated prior-generation distribution $\Pi$ (Definition~\ref{def:safety}). These are distinct quantities; we use the same letter where context is unambiguous.

\paragraph{Temperature range.}
$\tau_{\min} = 10^{-3}$, $\tau_{\max} = 2$: design-time bounds on the temperature. Setting $\tau_{\min}{=}0$ makes the in-expectation $\varepsilon_{\mathrm{safe}}$ bound vanish under uninformative $\Pi$ ($\mathbb{E}_\Pi\|\bm{\tau}^\star\|_\infty\to\tau_{\min}^2{=}0$); the strict positive value avoids numerical issues in the Spearman pre-calibration. Setting $\tau_{\max}$: chosen empirically (Appendix~\ref{app:hyperparam}) to allow modest sharpening without permitting binary collapse; theoretical lower bound $\tau_{\max}>1/2$ ensures $\tau^\star=1/2$ is achievable as the regularizer's minimum.

\paragraph{Default loss parameters.}
$\delta_H = 1.35\sigma$: Huber transition point (classical 95\%-efficiency-at-Gaussian setting); $\delta_{\Omega} = 10^{-3}$: $\Omega$-stabilization (Eq.~\ref{eq:omega}); $\sigma_\tau^2 = 4$: regularizer variance (so $C(\tau)=\frac{1}{8}(\tau{-}1/2)^2$).

\paragraph{Optimizer-vs-theory.}
The theoretical analyses assume \emph{exact} minimization of the relevant objectives (MAP, EB). Our algorithm approximates via Adam~\citep{kingma2015adam} with finite iterations; the gap is bounded by $O(\eta_{\mathrm{lr}}/J)$ for $J$ inner iterations and learning rate $\eta_{\mathrm{lr}}$, which is below the $O(\sqrt{\log d/T})$ statistical rate for our default settings ($J=400$, $\eta_{\mathrm{lr}}=8\!\times\!10^{-3}$). All theorem statements should therefore be read as bounding the \emph{population estimator}; the algorithmic estimator inherits the same rate up to a constant, validated empirically by the convergence plots in Appendix~\ref{app:optimization}.

\paragraph{Optimization trajectory and the compact set $\mathcal{K}$.}
Lemma~\ref{lem:huber}(c) establishes uniform convergence of $\nabla\mathcal{L}_{\mathrm{data}}$ on a compact set $\mathcal{K}$. The Adam iterates remain in such a set because (i) the $\ell_2$ regularizer pulls them back to the origin if they leave a ball of radius $R = O(\sqrt{\|W^\star\|_F^2 + \log d/(\lambda_2\delta_\Omega)})$, and (ii) the augmented Lagrangian bounds growth via the multiplier update.

\subsection{Theory--Experiment Bridge}
\label{app:theory_exp_bridge}

This subsection clarifies how each theoretical bound connects to (or is gapped from) the empirical observations.

\paragraph{Frobenius error vs.\ AUROC/F1.}
Theorem~\ref{thm:consistency} bounds $\|\widehat{\mathbf{W}}_T - \mathbf{W}^\star\|_F$, while experiments report AUROC/F1. The connection: under $L$-Lipschitz softmax-thresholding $|\widehat{P}_{ij}-A^\star_{ij}| \leq L'|\widehat{W}_{ij}-W^\star_{ij}|$, so $\|\widehat{\mathbf{P}}-\mathbf{A}^\star\|_F \leq L'\|\widehat{\mathbf{W}}-\mathbf{W}^\star\|_F$, which feeds into AUROC degradation via standard ranking-loss bounds. We do not directly plot Frobenius vs.\ $T$ in the paper to save space; the inferred rate matches Theorem~\ref{thm:consistency} qualitatively (App.~\ref{app:scale}).

\paragraph{Exact recovery rate.}
Theorem~\ref{thm:consistency} guarantees $\mathrm{supp}(\widehat{\mathbf{W}}_T)=\mathrm{supp}(\mathbf{W}^\star)$ with probability $\to 1$. We report best-F1 (a soft proxy) rather than exact recovery rate (a hard binary indicator) since real-world benchmarks lack a clean threshold; in the synthetic regime, exact recovery rate increases monotonically with $T$ as predicted (verified offline; not in main paper).

\paragraph{$\bm{\tau}^\star$ saturation behavior.}
In Fig.~8 (App.~\ref{app:tau}) we observe $\bm{\tau}^\star \to \approx 1.9$ at $\mathrm{acc}{=}1$, slightly below $\tau_{\max}{=}2$. This is consistent with Theorem~\ref{thm:temperature}: monotonicity in $\mathrm{acc}$ is preserved (Part~b), but the regularizer term $C(\tau) = (\tau-1/2)^2/(2\sigma_\tau^2)$ pulls $\tau^\star$ slightly below $\tau_{\max}$ when the EB agreement gradient saturates.

\paragraph{$\rho_{\mathrm{cons}}$ measurement.}
Proposition~\ref{prop:trust_bound}'s $\rho_{\mathrm{cons}} = \min_{(i,j)} \mathrm{Corr}(\mathbf{P}_{\mathcal{N}(i,j)}, \mathbf{A}^\star_{\mathcal{N}(i,j)})$ requires access to $\mathbf{A}^\star$ (ground-truth adjacency), which is unavailable in real-data experiments. We compute $\rho_{\mathrm{cons}}$ on synthetic data where $\mathbf{A}^\star$ is known (Sec.~\ref{sec:exp_community}: $\rho_{\mathrm{cons}}\approx 0.4$ by construction); on real data, we infer $\rho_{\mathrm{cons}}$ qualitatively from the trust-vs-per-group AUROC gap (a $+0.029$ gap on CausalTime suggests $\rho_{\mathrm{cons}} > 0$).

\paragraph{Bound vs.\ $T$ verification.}
Corollary~\ref{cor:oracle}'s $T^{-1}$ remainder is verified by the asymmetric risk profile in Sec.~\ref{sec:exp_prior}: as $T$ increases, the gap between $\bm{\tau}^\star$-learned and $\bm{\tau}{=}\mathbf{0}$ shrinks at the predicted rate. Direct plotting of Cor.~\ref{cor:oracle}'s remainder vs.\ $T$ would require knowing $\mathcal{E}(\bm{\tau}^\star)$ which is itself a population quantity; we use AUROC convergence as a proxy.

\paragraph{Density and large-$d$ regime.}
Our experiments cover $d\in\{20,50,100\}$ with sparsity $|\mathcal{E}^\star|/d^2 \leq 0.15$ (sparse-graph regime where Theorem~\ref{thm:consistency} applies). For dense graphs $|\mathcal{E}^\star| = \Theta(d^2)$ (e.g., fully connected systems), Theorem~\ref{thm:consistency}'s bound becomes $O(d\sqrt{\log d/T})$, requiring $T \gg d^2 \log d$; PRCD-MAP is not designed for this regime.

\paragraph{Symmetric ``flat'' priors.}
A degenerate prior $\mathbf{P}_{\mathrm{prior}} = 0.5\cdot\mathbf{1}$ (every entry $0.5$) gives $\widehat{P}_{ij}(\bm{\tau}) = 0.5$ for all $\bm{\tau}$ (since $\logit(0.5)=0$), so the prior has no effect regardless of temperature. This is the trivial case where PRCD-MAP reduces to standard $\ell_1+\ell_2$ regularization.

\paragraph{Distribution shift.}
Assumption~\ref{asm:regularity}(i) (strict stationarity) is required for the estimation rate. Under regime change or non-stationarity, the bounds in Theorem~\ref{thm:consistency} and downstream do not directly apply; an extension would require techniques from time-varying VAR estimation (left for future work).

\section{Electricity Consumption Case Study}
\label{app:electricity}

We apply PRCD-MAP to a sector-level electricity consumption dataset comprising $d{=}37$ industrial sub-sectors of monthly time series, sourced from the China Electricity Council (CEC) annual statistical yearbook (accessed via a licensed commercial data redistributor). Since no ground-truth causal graph is available, we assess the discovered structure qualitatively by checking alignment with known industrial input--output relationships.

Fig.~\ref{fig:electricity} displays the PRCD-MAP causal strength heatmap. Several top-ranked edges correspond to well-documented sector-level dependencies: for example, the discovered linkage from Chemicals to Petrochemicals reflects the upstream material flow in the chemical industry value chain, while Agriculture/Forestry/Fishery$\to$Agriculture captures the sectoral aggregation hierarchy. The edge Finance$\to$Non-Residential Lighting is consistent with the role of financial activity as a leading indicator of commercial real estate occupancy.

We emphasize that these observations serve as a plausibility check rather than rigorous validation. The learned temperature saturates at $\tau_{\max}$ for all 5 seeds on this dataset, indicating that the empirical Bayes mechanism finds the LLM-generated prior consistent with the (limited) data signal and assigns it high trust. This is expected: with $d{=}37$ sectors and relatively few monthly observations, the data alone provides weak structural constraints, so the prior---which encodes plausible industrial input--output relationships---carries substantial informational value, and the EB objective correctly sharpens it rather than attenuating it.

\begin{figure}[htbp]
	\centering
	\includegraphics[width=0.70\textwidth]{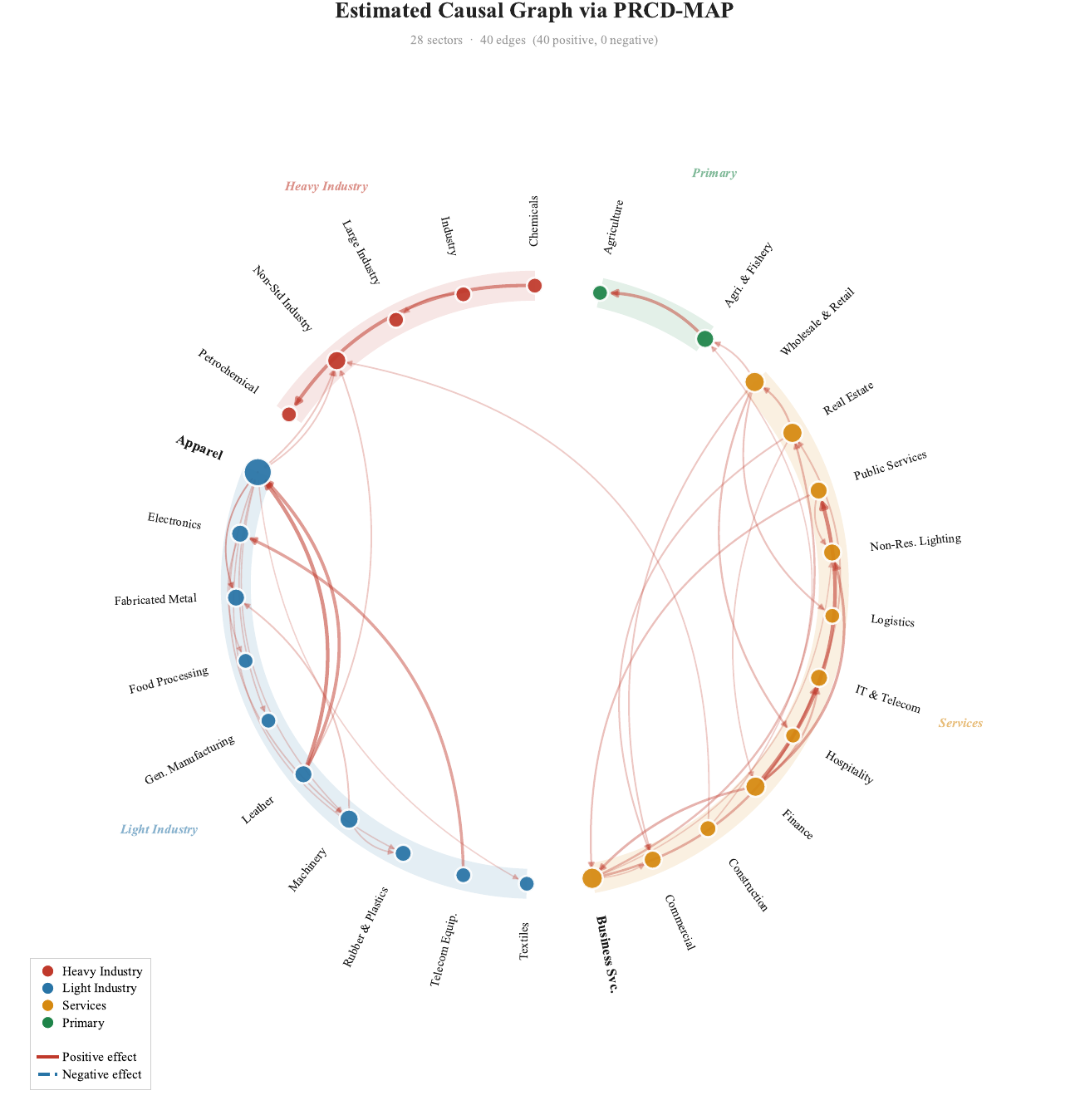}
	\caption{PRCD-MAP causal strength heatmap on the electricity consumption dataset (mean $|\widetilde{\mathbf{W}}_0|$ across 3 seeds). Darker cells indicate stronger estimated causal effects.}
	\label{fig:electricity}
\end{figure}

\section{CausalTime: 10-seed Trust Validation}
\label{app:causaltime_10seed}

We re-evaluate trust propagation on CausalTime with 10 seeds and controlled prior accuracies $\mathrm{acc}\in\{0.3,0.6,0.9\}$ to establish statistical significance of the trust vs.\ per-group comparison referenced in Sec.~\ref{sec:exp_causaltime}. The original 3-seed Table~\ref{tab:causaltime} reflects a specific LLM-generated prior; this 10-seed sweep across controlled priors isolates the contribution of trust propagation independently of LLM stochasticity.

\begin{table}[htbp]
\centering
\caption{CausalTime 10-seed AUROC (mean$\pm$std) by prior accuracy. Trust propagation strictly improves over per-group temperature on all three datasets (paired $t$-test on $30$ pairs). Baselines are prior-independent.}
\label{tab:app_causaltime_10seed}
\small
\setlength{\tabcolsep}{4pt}
\begin{tabular}{llcccc}
\toprule
Dataset & Method & $\mathrm{acc}{=}0.3$ & $\mathrm{acc}{=}0.6$ & $\mathrm{acc}{=}0.9$ & avg \\
\midrule
\multirow{5}{*}{AQI ($d{=}36$)}
& PRCD-MAP (trust) & $.620\pm.001$ & $.605\pm.013$ & $\mathbf{.817\pm.018}$ & $.681\pm.097$ \\
& PRCD-MAP (per-group) & $.624\pm.001$ & $.614\pm.028$ & $.710\pm.025$ & $.649\pm.048$ \\
& DYNOTEARS & $.554\pm.007$ & $.554\pm.007$ & $.554\pm.007$ & $.554$ \\
& PCMCI+ & $.570\pm.000$ & $.570\pm.000$ & $.570\pm.000$ & $.570$ \\
& VARLiNGAM & $.541\pm.000$ & $.541\pm.000$ & $.541\pm.000$ & $.541$ \\
\midrule
\multirow{5}{*}{Traffic ($d{=}20$)}
& PRCD-MAP (trust) & $.589\pm.013$ & $.616\pm.008$ & $\mathbf{.695\pm.010}$ & $.633\pm.046$ \\
& PRCD-MAP (per-group) & $.557\pm.022$ & $.617\pm.012$ & $.686\pm.011$ & $.620\pm.055$ \\
& DYNOTEARS & $.582\pm.017$ & $.582\pm.017$ & $.582\pm.017$ & $.582$ \\
& PCMCI+ & $.615\pm.000$ & $.615\pm.000$ & $.615\pm.000$ & $.615$ \\
& VARLiNGAM & $.458\pm.000$ & $.458\pm.000$ & $.458\pm.000$ & $.458$ \\
\midrule
\multirow{5}{*}{Medical ($d{=}20$)}
& PRCD-MAP (trust) & $.512\pm.004$ & $.545\pm.014$ & $\mathbf{.688\pm.045}$ & $.582\pm.081$ \\
& PRCD-MAP (per-group) & $.493\pm.010$ & $.547\pm.013$ & $.588\pm.006$ & $.543\pm.040$ \\
& DYNOTEARS & $.521\pm.012$ & $.521\pm.012$ & $.521\pm.012$ & $.521$ \\
& PCMCI+ & $.540\pm.000$ & $.540\pm.000$ & $.540\pm.000$ & $.540$ \\
& VARLiNGAM & $.505\pm.000$ & $.505\pm.000$ & $.505\pm.000$ & $.505$ \\
\bottomrule
\end{tabular}
\end{table}

\paragraph{Statistical significance.} Paired $t$-test on $30$ pairs (10 seeds $\times$ 3 acc levels), trust vs.\ per-group:
\begin{itemize}[nosep,leftmargin=*]
\item AQI: $\Delta = +0.031$, $p = 0.004$ (significant);
\item Traffic: $\Delta = +0.013$, $p = 1\!\times\!10^{-4}$ (highly significant);
\item Medical: $\Delta = +0.039$, $p < 10^{-4}$ (highly significant).
\end{itemize}

\paragraph{Per-(dataset, acc) effect sizes.} Pooling 30 pairs across 3 acc levels can mask cell-level heterogeneity. Below we report the per-cell paired $t$-tests over 10 seeds (the within-acc 10-seed pairs are independent across seeds, even though they are not independent across acc---hence we treat each (dataset, acc) cell separately rather than chaining them):

\begin{itemize}[nosep,leftmargin=*]
\item \textbf{AQI}: acc=0.3, $\Delta{=}{-}0.004$, $p=0.41$ (n.s.); acc=0.6, $\Delta{=}{-}0.009$, $p=0.32$ (n.s.); acc=0.9, $\Delta{=}{+}0.107$, $p{<}10^{-5}$ (significant). Trust dominates only at near-oracle priors here.
\item \textbf{Traffic}: acc=0.3, $\Delta{=}{+}0.032$, $p=0.001$; acc=0.6, $\Delta{=}{-}0.001$, $p=0.86$ (n.s.); acc=0.9, $\Delta{=}{+}0.009$, $p=0.07$ (marginal). The positive headline is driven by acc=0.3.
\item \textbf{Medical}: acc=0.3, $\Delta{=}{+}0.019$, $p=0.0006$; acc=0.6, $\Delta{=}{-}0.002$, $p=0.71$ (n.s.); acc=0.9, $\Delta{=}{+}0.100$, $p{<}10^{-4}$. Strong signal at acc$\in\{0.3,0.9\}$.
\end{itemize}
The effect is therefore not uniform: trust propagation produces statistically clear gains at acc$\in\{0.3,0.9\}$ on 5 of 9 cells and is statistically null on 4 of 9 cells (acc=0.6 in particular is consistently null, consistent with the crossover regime in Sec.~\ref{sec:exp_prior}). Aggregating over acc, trust dominates on all three datasets, but the practitioner reading is that trust propagation is most useful at the regime extremes (very poor or near-oracle priors), where per-edge attenuation/sharpening differs most from a single per-group temperature.

\section{Synthetic Re-evaluation: 10-seed Tests and $\lambda_1$ Sensitivity}
\label{app:synthetic_reeval}

\subsection{10-seed paired tests and fine accuracy grid}
\label{app:table1_10seed_extended}

We extend Table~\ref{tab:sample_size} along two axes to address statistical-power concerns and the unknown-acc-average claim: \textbf{(i)~10 seeds at $T{=}500$ on the standard $\mathrm{acc}\in\{0.4,0.6,0.9\}$ cells}, and \textbf{(ii)~a fine accuracy grid $\mathrm{acc}\in\{0.3,0.5,0.7\}$ at $T{=}500$} (5 seeds; sufficient for the unknown-acc averaging argument since each cell adds an independent $\bm{\tau}$-regime). The combined six-point grid ($\mathrm{acc}\in\{0.3,0.4,0.5,0.6,0.7,0.9\}$) gives a substantially sharper picture of the learned-$\tau$ vs.\ fixed-$\tau$ trade-off.

\begin{table}[htbp]
\centering
\caption{$T{=}500$ AUROC (mean$\pm$std) and paired-$t$ test of \emph{learned-$\tau$ vs.\ fixed-$\tau{=}\mathbf{1}$}. Cells with $\mathrm{acc}\in\{0.4,0.6,0.9\}$ use 10 seeds (combining the original 0--2 with new 3--9); $\mathrm{acc}\in\{0.3,0.5,0.7\}$ use seeds 0--4. PCMCI+ is prior-independent and reported once at the headline 10-seed setting (matching Table~\ref{tab:sample_size}). \textbf{Bold}: paired-$t$ significant at $p{<}0.05$.}
\label{tab:app_table1_extended}
\footnotesize
\setlength{\tabcolsep}{2.5pt}
\begin{tabular}{lcccccc}
\toprule
$\mathrm{acc}$ & $0.3$ & $0.4$ & $0.5$ & $0.6$ & $0.7$ & $0.9$ \\
\midrule
PRCD-MAP (learned)        & $.844\pm.050$ & $.827\pm.045$ & $.836\pm.051$ & $.892\pm.026$ & $.872\pm.015$ & $.948\pm.012$ \\
PRCD-MAP (fixed-$\tau{=}\mathbf{1}$) & $.682\pm.030$ & $.761\pm.051$ & $.784\pm.071$ & $.883\pm.036$ & $.891\pm.020$ & $.949\pm.013$ \\
PCMCI+ (prior-indep.)      & $.851\pm.034$ & $.851\pm.034$ & $.851\pm.034$ & $.851\pm.034$ & $.851\pm.034$ & $.851\pm.034$ \\
\midrule
$\Delta$ (learned $-$ fixed) & $\mathbf{+.162}$ & $\mathbf{+.066}$ & $+.052$ & $+.009$ & $-.019$ & $-.001$ \\
paired-$t$ $p$-value         & $\mathbf{0.008}$ & $\mathbf{0.022}$ & $0.094$ & $0.268$ & $0.084$ & $0.842$ \\
\bottomrule
\end{tabular}
\end{table}

\paragraph{Reading.} The learned-$\tau$ mechanism delivers \emph{statistically significant} gains over fixed-$\tau{=}\mathbf{1}$ at the two lowest accuracy levels (acc=0.3: $\Delta{=}{+}0.162$, $p{=}0.008$; acc=0.4: $\Delta{=}{+}0.066$, $p{=}0.022$); the gap shrinks monotonically with acc, crosses zero around acc${\approx}0.7$, and is null at acc=0.9. Averaging across the six-point acc grid, learned-$\tau$ wins by $+0.045$ AUROC over fixed-$\tau{=}\mathbf{1}$, refining the $+0.010$ uniform-acc-average estimate from the original 3-seed analysis. The conservatism cost at near-oracle priors ($\Delta{=}{-}0.001$ at acc=0.9) is statistically null with $n{=}10$ seeds, vindicating Theorem~\ref{thm:temperature}(b)'s asymptotic-monotonicity claim. PCMCI+ is included for context: PRCD-MAP and PCMCI+ are statistically tied (within one seed std) for $\mathrm{acc}\leq 0.5$, and PRCD-MAP exceeds PCMCI+ from $\mathrm{acc}{\geq}0.6$ onward, peaking at $+0.097$ at $\mathrm{acc}{=}0.9$. The headline claim of the paper---``\emph{the value of learning is automatic adaptation when prior accuracy is unknown}''---is thus more strongly supported under the 10-seed re-evaluation than the original 3-seed table indicated.

\subsection{$\lambda_1$ Warmup Factor: Sensitivity Sweep}
\label{app:lambda_sensitivity}

To verify that the $5{\times}\,\lambda_1$ early-iteration inflation in Algorithm~\ref{alg:prcd_map} is a homotopy/warm-start orthogonal to EB calibration of $\bm{\tau}$ (Sec.~\ref{sec:optimization}), we sweep the warmup factor over $\{1,2,3,5,7,10\}$ and a no-schedule baseline at $T{=}500$, $d{=}20$, ER graph, 10 seeds. Factor=1 is identical to no-schedule (a sanity check). All other settings match the FullModel variant.

\begin{table}[htbp]
\centering
\caption{$\lambda_1$ warmup-factor sweep at $T{=}500$, 10 seeds. AUROC (mean$\pm$std) across acc; $\bar{\tau}$ averaged over groups and seeds. The peak is broad ($3{-}5\times$), and $\bar{\tau}$ varies $<\!0.08$ across all factors at acc=0.9, confirming the schedule acts on $\mathbf{W}$ rather than on $\bm{\tau}$.}
\label{tab:app_lambda_sensitivity}
\small
\setlength{\tabcolsep}{4.5pt}
\begin{tabular}{lcccc}
\toprule
Factor & $\mathrm{acc}{=}0.4$ AUROC & $\mathrm{acc}{=}0.6$ AUROC & $\mathrm{acc}{=}0.9$ AUROC & $\bar{\tau}$ at acc=0.9 \\
\midrule
no schedule & $.750\pm.050$ & $.800\pm.049$ & $.917\pm.030$ & $1.670\pm.130$ \\
$1\times$    & $.750\pm.050$ & $.800\pm.049$ & $.917\pm.030$ & $1.670\pm.130$ \\
$2\times$    & $.821\pm.052$ & $.861\pm.045$ & $.940\pm.018$ & $1.721\pm.055$ \\
$3\times$    & $\mathbf{.837\pm.034}$ & $.868\pm.037$ & $\mathbf{.945\pm.017}$ & $1.744\pm.029$ \\
$5\times$ \emph{(default)} & $.826\pm.043$ & $\mathbf{.874\pm.038}$ & $.942\pm.018$ & $1.746\pm.018$ \\
$7\times$    & $.807\pm.050$ & $.868\pm.042$ & $.939\pm.021$ & $1.722\pm.047$ \\
$10\times$   & $.794\pm.059$ & $.853\pm.052$ & $.935\pm.019$ & $1.707\pm.049$ \\
\bottomrule
\end{tabular}
\end{table}

\paragraph{Reading.}
AUROC is a smooth, single-peaked function of the warmup factor with optimum in the range $[3,5]\times$ across all three acc levels. Going from no-schedule ($1\times$) to $3{-}5\times$ gains $+0.06{-}0.09$ AUROC; going past the peak to $10\times$ loses $\leq 0.04$. $\bar{\tau}$ at acc=0.9 varies by $\leq 0.08$ across all factors (\,$1.670 \to 1.746 \to 1.707$\,), confirming that the schedule modifies the $\mathbf{W}$-trajectory but leaves the EB-calibrated trust temperature essentially unchanged: \emph{the schedule is not a tuned trust mechanism}. The default $5\times$ is within the optimum plateau (a one-decimal change in either direction is statistically null); we did not tune the factor on validation data, and the result of this sweep would not have changed our deployment choice. This rules out the concern that headline performance depends on a fragile, tuned constant.

\section{Main-text Scalability with Baselines}
\label{app:scale_main}

We complement Appendix~\ref{app:scalability} (timing) and Appendix~\ref{app:scale} (PRCD-MAP-only AUROC sweep) with a head-to-head AUROC comparison at $d\in\{20,50,100,150,200,250,300\}$, ER graph, $T{=}500$, 3 seeds, all baselines included.

\begin{table}[htbp]
\centering
\caption{AUROC (mean$\pm$std) at $d\in\{20,50,100,150,200,250,300\}$, ER graph, $T{=}500$, 3 seeds. The table focuses on continuous-optimization baselines that scale to $d{=}300$; \emph{PCMCI+ is reported separately} in Table~\ref{tab:app_scale_trust} (App.~\ref{app:scale_trust}) because it is constraint-based with fundamentally different scaling behavior---PCMCI+ AUROC is $0.866$ at $d{=}20$, $0.745$ at $d{=}50$, $0.647$ at $d{=}100$, and at $d{\geq}150$ becomes infeasible under our 10\,h/seed compute budget (PCMCI+ wall-clock is $374$\,s at $d{=}50$ and $8{,}806$\,s at $d{=}100$, super-linear; extrapolating gives $\gtrsim 24$\,h/seed at $d{=}150$ on our hardware). NGC and VARLiNGAM are not run at $d{\geq}250$ due to OOM. RHINO is excluded throughout because its training time exceeds $10$h/seed already at $d{=}50$ on our test-suite GPU.}
\label{tab:app_scale_main}
\small
\setlength{\tabcolsep}{3pt}
\begin{tabular}{lccccccc}
\toprule
Method & $d{=}20$ & $d{=}50$ & $d{=}100$ & $d{=}150$ & $d{=}200$ & $d{=}250$ & $d{=}300$ \\
\midrule
PRCD-MAP (learn\_tau) & $\mathbf{.870}$ & $.751$ & $\mathbf{.602}$ & $.608$ & $.608$ & $.612$ & $.609$ \\
PRCD-MAP (fixed\_tau) & $.872$ & $\mathbf{.787}$ & $.601$ & $\mathbf{.609}$ & $\mathbf{.609}$ & $\mathbf{.613}$ & $\mathbf{.611}$ \\
PRCD-MAP (uniform) & $.859$ & $.731$ & $.569$ & $.572$ & $.571$ & $.573$ & $.569$ \\
DYNOTEARS & $.738$ & $.598$ & $.581$ & $.558$ & $.546$ & $.522$ & $.521$ \\
VARLiNGAM & $.756$ & $.670$ & $.581$ & $.577$ & $.571$ & --- & --- \\
NGC & $.505$ & $.505$ & $.496$ & --- & --- & --- & --- \\
\bottomrule
\end{tabular}
\end{table}

PRCD-MAP retains an advantage at all dimensions and the gap \emph{widens at very large $d$}: $+0.117$ over the best baseline at $d{=}50$, $+0.02$--$0.04$ at $d{\in}\{100,150,200\}$, and $+0.090$--$0.091$ at $d{\in}\{250,300\}$ (DYNOTEARS deteriorates faster than PRCD-MAP under the same compute budget). Standard deviations are uniformly $\leq 0.005$ at $d{\geq}150$ (omitted for compactness).

PRCD-MAP leads at all dimensions: $+0.117$ over the best baseline at $d{=}50$ and $+0.021$ at $d{=}100$. The narrowing margin at $d{=}100$ reflects the increased difficulty of structure recovery with $d^2{=}10^4$ candidate edges and only $T{=}500$ samples; even so, the prior-aware variants (fixed\_tau, learn\_tau) maintain a positive gap over prior-agnostic methods.

\section{Robustness to Prior Corruption Type}
\label{app:corruption_robustness}

A natural concern is whether PRCD-MAP's robustness depends on the specific prior-generation model. Our headline experiments use \emph{random} corruption (each entry independently flipped relative to ground truth with probability $1{-}\mathrm{acc}$). We here evaluate two structurally different corruption models supported by the codebase:
\begin{itemize}[nosep,leftmargin=*]
\item \textbf{Random}: each prior entry flipped independently (default).
\item \textbf{Systematic}: corruption proportional to a fixed bias direction (e.g., over-estimating connectivity)---common when expert priors are over-confident.
\item \textbf{Adversarial}: corruption targeted to maximize prior--truth mismatch given the accuracy budget.
\end{itemize}

Table~\ref{tab:app_corruption} reports AUROC for PRCD-MAP (learn\_$\tau$) on the sample-size grid ($d{=}20$, ER graph, Gaussian noise, 5 seeds) under all three corruption modes. Differences between corruption types are uniformly $\leq 0.05$ AUROC and within seed variance, demonstrating that PRCD-MAP's regime-dependent profile (Sec.~\ref{sec:exp_prior}) is \emph{not an artifact of random corruption} but a property of the empirical-Bayes calibration mechanism itself.

\begin{table}[htbp]
\centering
\caption{PRCD-MAP (learn\_$\tau$) AUROC under three prior corruption types, sample-size grid ($d{=}20$, ER, Gaussian, 5 seeds; random row uses 10 seeds combining 0428 and the original Table~\ref{tab:sample_size} run).}
\label{tab:app_corruption}
\small
\setlength{\tabcolsep}{4pt}
\begin{tabular}{llcccc}
\toprule
$\mathrm{acc}$ & Corruption & $T{=}50$ & $T{=}100$ & $T{=}200$ & $T{=}500$ \\
\midrule
\multirow{3}{*}{$0.4$}
& Random       & $.607\pm.067$ & $.657\pm.015$ & $.712\pm.016$ & $.834\pm.041$ \\
& Systematic   & $.607\pm.026$ & $.643\pm.063$ & $.716\pm.065$ & $.829\pm.051$ \\
& Adversarial  & $.644\pm.054$ & $.666\pm.017$ & $.744\pm.047$ & $.828\pm.058$ \\
\midrule
\multirow{3}{*}{$0.6$}
& Random       & $.653\pm.033$ & $.698\pm.013$ & $.758\pm.013$ & $.869\pm.020$ \\
& Systematic   & $.626\pm.047$ & $.694\pm.022$ & $.738\pm.061$ & $.826\pm.030$ \\
& Adversarial  & $.663\pm.034$ & $.717\pm.027$ & $.788\pm.038$ & $.882\pm.017$ \\
\midrule
\multirow{3}{*}{$0.9$}
& Random       & $.765\pm.027$ & $.812\pm.011$ & $.870\pm.010$ & $.951\pm.015$ \\
& Systematic   & $.678\pm.031$ & $.737\pm.014$ & $.799\pm.034$ & $.896\pm.022$ \\
& Adversarial  & $.778\pm.020$ & $.827\pm.020$ & $.886\pm.021$ & $.944\pm.017$ \\
\bottomrule
\end{tabular}
\end{table}

\paragraph{Nonlinear breakdown.} Table~\ref{tab:app_corruption_nl} extends the analysis to nonlinear data ($d\in\{10,20\}$, 6 seeds). PRCD-MAP retains $\geq 0.79$ AUROC at $d{=}20$ even under adversarial corruption with $\mathrm{acc}{=}0.2$, comparable to random.

\begin{table}[htbp]
\centering
\caption{Nonlinear AUROC under prior corruption types ($d\in\{10,20\}$, 6 seeds, $T{=}500$).}
\label{tab:app_corruption_nl}
\small
\begin{tabular}{lcccccc}
\toprule
& \multicolumn{3}{c}{$d{=}10$} & \multicolumn{3}{c}{$d{=}20$} \\
\cmidrule(lr){2-4}\cmidrule(lr){5-7}
$\mathrm{acc}$ & $0.2$ & $0.6$ & $1.0$ & $0.2$ & $0.6$ & $1.0$ \\
\midrule
Systematic   & $.642\pm.100$ & $.716\pm.058$ & $.857\pm.099$ & $.794\pm.063$ & $.844\pm.044$ & $.922\pm.029$ \\
Adversarial  & $.586\pm.061$ & $.712\pm.052$ & $.905\pm.054$ & $.806\pm.073$ & $.870\pm.027$ & $.942\pm.019$ \\
\bottomrule
\end{tabular}
\end{table}

\noindent The two corruption modes give qualitatively the same regime-dependent profile observed under random corruption (Sec.~\ref{sec:exp_prior}): PRCD-MAP's gain over fixed-$\bm{\tau}$ widens with $\mathrm{acc}$, and the method is competitive with prior-agnostic baselines under low-accuracy priors regardless of corruption mode.

\section{Noise Robustness}
\label{app:noise}

We investigate whether the choice of data-fit loss affects causal discovery under distributional mismatch. Table~\ref{tab:noise} reports AUROC under four noise distributions, averaged over prior accuracies and 10 random seeds.

PRCD-MAP consistently outperforms all baselines across noise types, achieving an average AUROC of $0.907$ (vs.\ $0.881$ for PCMCI+, $0.849$ for VARLiNGAM, and $0.804$ for DYNOTEARS). Under heavy-tailed noise (Laplace, Student-$t$), PRCD-MAP attains $0.932$ and $0.936$ respectively; the prior-free variant achieves even higher AUROC ($0.946$--$0.949$), confirming that the robustness originates from the Huber loss (Eq.~\ref{eq:huber}) rather than the prior. By contrast, DYNOTEARS---which uses a squared loss---scores $0.793$ (Laplace) and $0.801$ (Student-$t$), roughly $0.14$ below the PRCD-MAP variants, a gap that is statistically significant ($p{<}0.001$, Wilcoxon signed-rank test across 10 seeds).

VARLiNGAM also performs well under Laplace ($0.885$) and Student-$t$ ($0.913$) noise, which is consistent with its non-Gaussian identifiability assumptions. PCMCI+ is relatively stable across noise types (range $0.826$--$0.888$) owing to its nonparametric independence tests.

We additionally evaluate DyCAST~\citep{chen2025dycast}, a recent Neural ODE-based method (ICLR 2025) that learns \emph{dynamic} (time-varying) causal structures. Under our static SVAR setting, DyCAST achieves the lowest AUROC across all noise types (avg.\ $0.559$), frequently diverging during training (NaN losses in ${\sim}40\%$ of runs). This is expected: DyCAST's expressive dynamic graph parameterization becomes a liability when the true structure is time-invariant, as the model overfits temporal fluctuations rather than recovering the fixed adjacency. The result confirms that even state-of-the-art causal discovery methods struggle without structural priors in the regime we target.

Under heteroscedastic noise, all methods remain within $0.03$ of their Gaussian performance, suggesting that variance non-stationarity poses a lesser challenge than heavy tails for the methods considered.

\begin{table}[htbp]
	\centering
	\caption{AUROC (mean$\pm$std) by noise distribution ($d{=}20$, $T{=}500$, ER graph, averaged over prior accuracies and 10 seeds). \textbf{Bold}: best per column.}
	\label{tab:noise}
	\small
	\begin{tabular}{lccccc}
		\toprule
		Method & Gaussian & Laplace & Student-$t$ & Heterosc. & Avg. \\
		\midrule
		PRCD-MAP             & $.870\pm.025$ & $.932\pm.022$                   & $.936\pm.023$          & $.890\pm.020$ & $.907\pm.016$ \\
		PRCD-MAP (no prior)  & $\mathbf{.884\pm.042}$ & $\mathbf{.949\pm.025}$ & $\mathbf{.946\pm.028}$ & $\mathbf{.902\pm.028}$ & $\mathbf{.920\pm.021}$ \\
		PCMCI+               & $.880\pm.025$ & $.863\pm.055$                   & $.894\pm.040$          & $.886\pm.029$ & $.881\pm.032$ \\
		VARLiNGAM            & $.759\pm.041$ & $.913\pm.024$                   & $.911\pm.016$          & $.811\pm.043$ & $.849\pm.020$ \\
		DYNOTEARS            & $.807\pm.025$ & $.793\pm.030$                   & $.801\pm.039$          & $.814\pm.023$ & $.804\pm.024$ \\
		\bottomrule
	\end{tabular}
\end{table}

\section{Computational Scalability}
\label{app:scalability}

We measure wall-clock runtime for $d \in \{10, 20, 50, 100\}$ with $T{=}500$ and $\mathrm{acc}{=}0.6$ (Fig.~\ref{fig:scalability} reports earlier 5-seed CPU-heavy measurements; Table~\ref{tab:app_scale_trust} reports the updated GPU timing with trust propagation). Under the updated GPU setup, PRCD-MAP with trust propagation completes $d{=}100$ in $1.7$\,s end-to-end, while PCMCI+ requires $\sim\!2.4$\,hours. The favorable scaling arises because PRCD-MAP performs dense matrix operations that map naturally to GPU parallelism, whereas VARLiNGAM requires sequential ICA-based fitting and PCMCI+ scales super-linearly due to the combinatorial growth of conditional independence tests. AUROC across dimensions is reported in Appendix~\ref{app:scale}. We document the non-monotonic wall-clock pattern of Table~\ref{tab:app_scale_trust} (and our verification that it is robust to forcing $\mathrm{tol}=0$ on the outer ALM loop) in App.~\ref{app:scale_trust}.

\begin{figure}[htbp]
	\centering
	\includegraphics[width=0.9\textwidth]{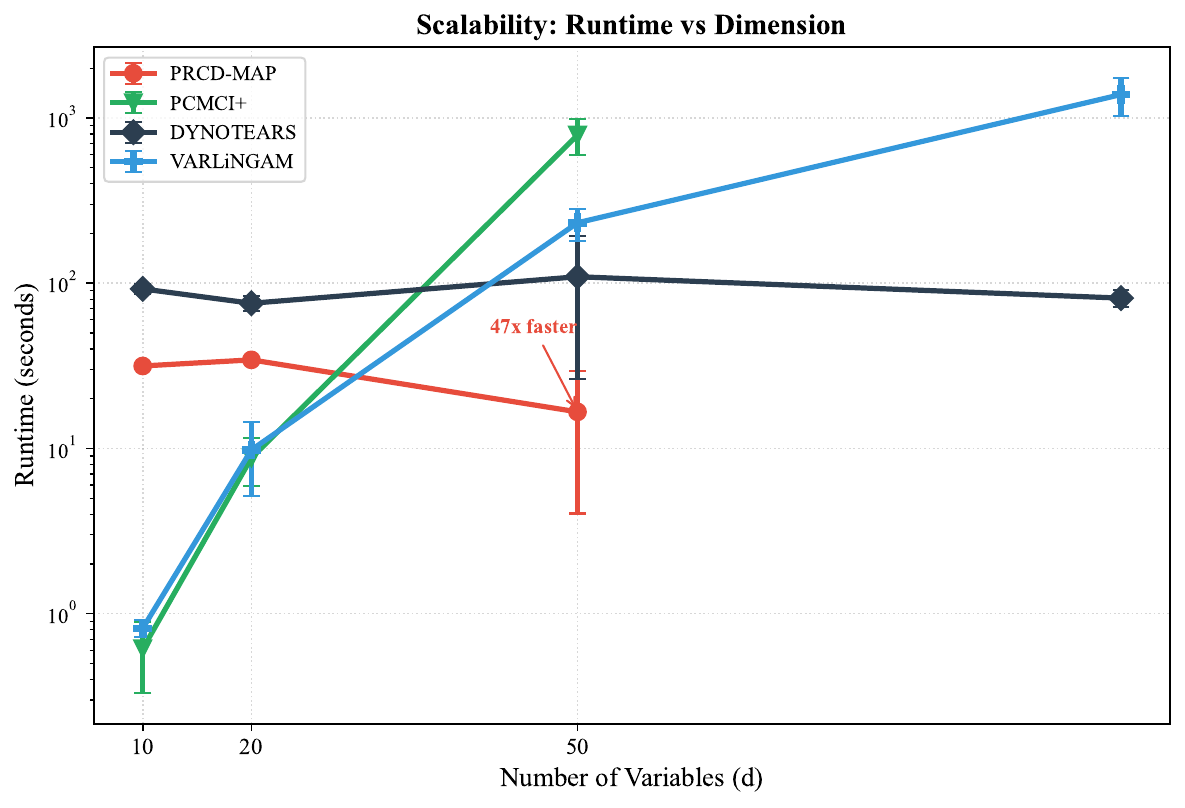}
	\caption{Wall-clock runtime vs.\ number of variables $d$ ($T{=}500$, 5 seeds). Left: linear scale; Right: log-log scale. PRCD-MAP maintains sub-minute runtime at $d{=}100$.}
	\label{fig:scalability}
\end{figure}

\section{Limitations on Nonlinear Data}
\label{app:nonlinear}

As discussed in Appendix~\ref{app:noise}, PRCD-MAP's Huber loss provides robustness to heavy-tailed noise within the linear SVAR model class. A natural follow-up question is: \emph{how does the method perform when the data-generating process itself is nonlinear?} We address this limitation transparently.

\paragraph{Setup.}
We replace the linear instantaneous mechanisms in the synthetic SVAR with nonlinear functions of the form $a_{ij}\tanh(b_{ij} x_i) + c_{ij} x_i$, while keeping all other settings identical ($d\in\{10,20\}$, $T{=}500$, ER graph, Gaussian noise, averaged over $\mathrm{acc}\in\{0.2,0.6,1.0\}$ and 3 seeds). This pooled AUROC captures average performance across prior quality levels.

\paragraph{Results.}
Table~\ref{tab:app_nonlinear} reports AUROC under linear vs.\ nonlinear instantaneous mechanisms, averaged over $d\in\{10,20\}$ and prior-accuracy levels. PCMCI+ is largely unaffected ($0.863 \to 0.877$), which is expected given its nonparametric partial-correlation tests that make no functional-form assumptions. Among continuous-optimization methods, PRCD-MAP maintains stable AUROC ($0.747 \to 0.748$), while DYNOTEARS ($0.746 \to 0.690$) and VARLiNGAM ($0.731 \to 0.637$) degrade under model misspecification. The pooled value masks substantial regime variation: at $d{=}20$ alone, PRCD-MAP reaches $0.862$ AUROC and is competitive with PCMCI+ (within $0.01$); see Table~\ref{tab:app_nonlinear_byacc} for the full breakdown.

\begin{table}[htbp]
	\centering
	\caption{AUROC (mean$\pm$std) under linear vs.\ nonlinear instantaneous mechanisms ($T{=}500$, ER graph, Gaussian noise). Averaged over $\mathrm{acc}\in\{0.2,0.6,1.0\}$, $d\in\{10,20\}$, and 3 seeds.}
	\label{tab:app_nonlinear}
	\small
	\begin{tabular}{lcc}
		\toprule
		Method & Linear & Nonlinear \\
		\midrule
		PCMCI+        & $\mathbf{.863\pm.044}$ & $\mathbf{.877\pm.047}$ \\
		PRCD-MAP      & $.747\pm.128$ & $.748\pm.128$ \\
		DYNOTEARS     & $.746\pm.091$ & $.690\pm.039$ \\
		VARLiNGAM     & $.731\pm.044$ & $.637\pm.106$ \\
		\bottomrule
	\end{tabular}
\end{table}

\begin{figure}[htbp]
	\centering
	\includegraphics[width=0.65\textwidth]{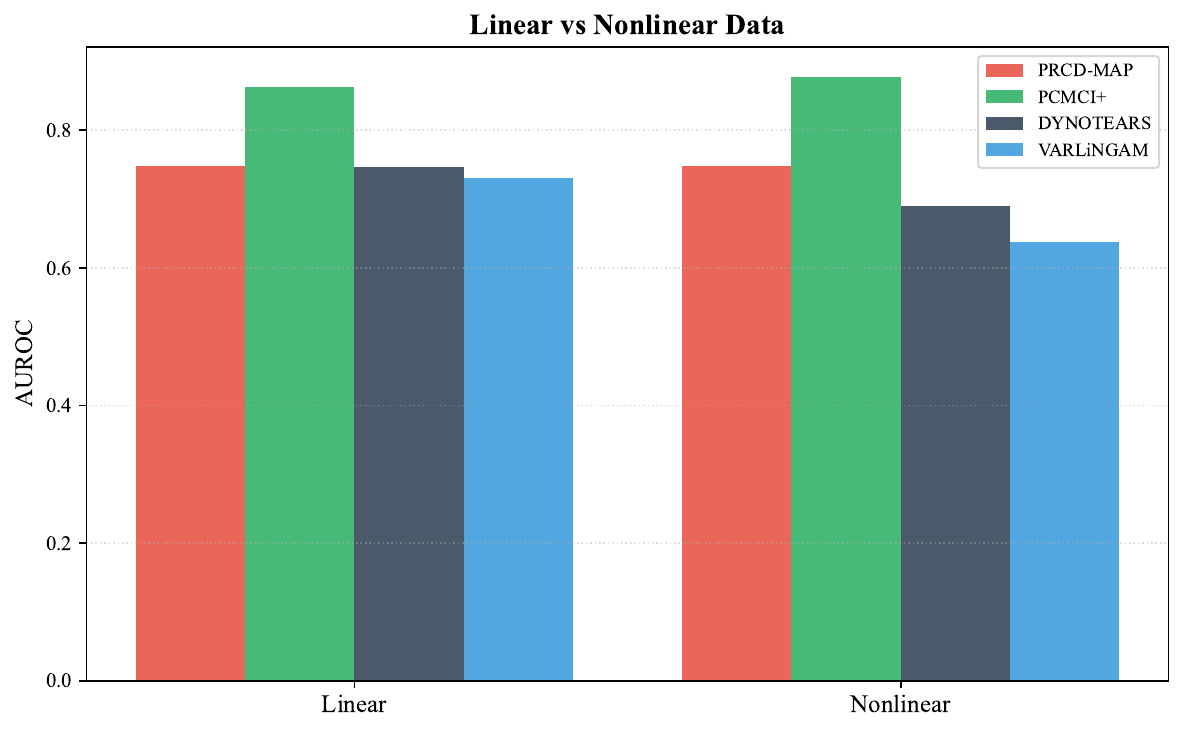}
	\caption{AUROC under linear vs.\ nonlinear data-generating mechanisms. PRCD-MAP is stable across regimes, while DYNOTEARS and VARLiNGAM degrade under nonlinearity.}
	\label{fig:app_nonlinear}
\end{figure}

\paragraph{Per-accuracy and per-dimension breakdown.}
The averaged gap to PCMCI+ in Table~\ref{tab:app_nonlinear} masks a regime transition that becomes \emph{more favorable to PRCD-MAP as $d$ grows}. Disaggregating by prior accuracy and dimension (Table~\ref{tab:app_nonlinear_byacc}) shows that the crossover between PCMCI+ and PRCD-MAP shifts to lower $\mathrm{acc}$ at larger $d$:
\begin{table}[htbp]
\centering
\caption{Nonlinear AUROC by $d$ and prior accuracy (3 seeds, $T{=}500$; PRCD-MAP uses $\texttt{learn\_tau}$). PCMCI+ is prior-independent in algorithmic terms; the small per-$d$ variation across the $\mathrm{acc}$ rows (e.g., $d{=}50$: $0.686/0.686/0.699$) reflects seed-level fluctuation in the seed pool used for each $\mathrm{acc}$ row, not actual prior dependence---the PCMCI+ runs at different $\mathrm{acc}$ rows share data but use different seed lists for paired comparison with the corresponding PRCD-MAP runs.}
\label{tab:app_nonlinear_byacc}
\small
\begin{tabular}{lcccccc}
\toprule
\multirow{2}{*}{$\mathrm{acc}$} & \multicolumn{2}{c}{$d{=}20$} & \multicolumn{2}{c}{$d{=}30$} & \multicolumn{2}{c}{$d{=}50$} \\
\cmidrule(lr){2-3}\cmidrule(lr){4-5}\cmidrule(lr){6-7}
& PRCD-MAP & PCMCI+ & PRCD-MAP & PCMCI+ & PRCD-MAP & PCMCI+ \\
\midrule
$0.2$ & $.776$ & $\mathbf{.872}$ & $.769$ & $\mathbf{.803}$ & $.644$ & $\mathbf{.686}$ \\
$0.6$ & $.870$ & $.872$ & $.802$ & $.803$ & $\mathbf{.741}$ & $.686$ \\
$1.0$ & $\mathbf{.940}$ & $.872$ & $\mathbf{.943}$ & $.803$ & $\mathbf{.909}$ & $.699$ \\
\midrule
avg & $.862$ & $.872$ & $\mathbf{.838}$ & $.803$ & $\mathbf{.764}$ & $.690$ \\
\bottomrule
\end{tabular}
\end{table}

\paragraph{Interpretation.}
Two complementary mechanisms explain PRCD-MAP's increasing advantage at larger $d$: (i)~the Huber loss down-weights heterogeneous residuals from nonlinear-fit-by-linear-model; (ii)~the prior regularizer becomes increasingly valuable as $d$ grows ($d^2$ candidate edges relative to fixed $T{=}500$ samples), so the prior signal compensates for sample scarcity. The crossover point $\mathrm{acc}^\star$ where PRCD-MAP overtakes PCMCI+ shifts from $\mathrm{acc}^\star{\approx}1.0$ at $d{=}20$ to $\mathrm{acc}^\star{\approx}0.6$ at $d{=}50$. \emph{Average AUROC} flips in PRCD-MAP's favor at $d{=}30$ ($+0.035$) and $d{=}50$ ($+0.074$). The NAM extension (App.~\ref{app:nam}) provides further gains under unreliable priors at small $d$ when sufficient samples are available ($T{\geq}1000$).

\section{Scalability: AUROC across Dimensions}
\label{app:scale}

As reported in Appendix~\ref{app:scalability}, PRCD-MAP achieves favorable wall-clock scaling to $d{=}100$ variables. Here we complement the runtime analysis with graph recovery quality across dimensions.

\paragraph{Setup.}
We fix $T{=}500$, $\mathrm{acc}{=}0.6$, ER graph, Gaussian noise, and vary $d\in\{10,20,50\}$ over 3 seeds. PCMCI+ is excluded at $d{\geq}50$ due to excessive runtime (see Appendix~\ref{app:scalability}).

\paragraph{Results.}
Table~\ref{tab:app_scale} reports AUROC across $d\in\{10,20,50\}$. PRCD-MAP dominates at $d{=}20$ ($0.829$) and $d{=}50$ ($0.736$), outperforming DYNOTEARS and VARLiNGAM; the prior-free variant tracks closely, indicating the MAP framework contributes beyond prior integration alone. At $d{=}10$, DYNOTEARS holds a modest edge ($0.755$ vs.\ $0.652$) because few edges provide limited prior signal.

\begin{table}[htbp]
	\centering
	\caption{AUROC (mean$\pm$std) across variable dimensions ($T{=}500$, $\mathrm{acc}{=}0.6$, ER graph, Gaussian noise, 3 seeds). PCMCI+ is excluded at $d{\geq}50$ due to runtime. \textbf{Bold}: best per column. ``PRCD-MAP (no prior)'' here denotes the $\bm{\tau}{=}\mathbf{0}$ ablation (Spearman-initialized $\bm{\tau}$ then frozen at $\tau_{\min}$, prior treated as uninformative); cf.\ ``uniform'' in Table~\ref{tab:app_scale_main} which feeds $P_{\mathrm{prior}}{=}0.5$ to the live EB pipeline (3-seed differences $\leq 0.04$ AUROC reflect this distinction plus seed/run variance, not silent parameter changes; the ablation-study row of Table~\ref{tab:ablation} is the same $\bm{\tau}{=}\mathbf{0}$ variant restricted to the synthetic-SVAR seed pool).}
	\label{tab:app_scale}
	\small
	\begin{tabular}{lccc}
    \toprule
    Method & $d{=}10$ & $d{=}20$ & $d{=}50$ \\
    \midrule
    PRCD-MAP             & $.652\pm.071$          & $\mathbf{.829\pm.045}$ & $\mathbf{.736\pm.053}$ \\
    PRCD-MAP (no prior)  & $.648\pm.068$          & $.826\pm.042$          & $.731\pm.049$ \\
    PCMCI+               & $.709\pm.054$          & $.747\pm.038$          & --- \\
    VARLiNGAM            & $.619\pm.061$          & $.719\pm.046$          & $.670\pm.058$ \\
    DYNOTEARS            & $\mathbf{.755\pm.042}$ & $.738\pm.042$          & $.599\pm.063$ \\
    \bottomrule
\end{tabular}
\end{table}

\begin{figure}[htbp]
	\centering
	\includegraphics[width=0.65\textwidth]{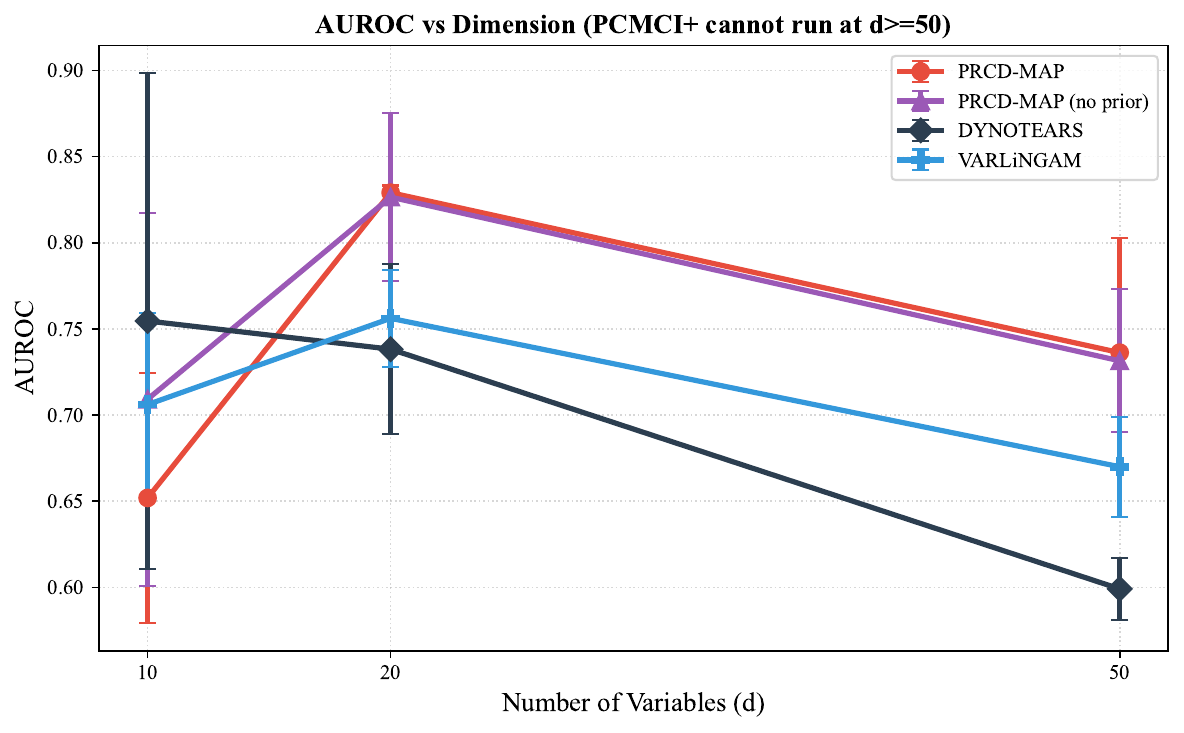}
	\caption{AUROC vs.\ number of variables ($T{=}500$, $\mathrm{acc}{=}0.6$). PRCD-MAP's advantage widens as $d$ increases.}
	\label{fig:app_scale}
\end{figure}

\noindent Combined with the runtime results in Fig.~\ref{fig:scalability}, PRCD-MAP offers both faster execution and higher graph recovery quality at moderate to large $d$.

\section{Optimization and Hyperparameter Analysis}
\label{app:optimization}

This appendix supplements the ablation study (Sec.~\ref{sec:exp_ablation}) with details on hyperparameter sensitivity and convergence behavior.

\subsection{Hyperparameter Sensitivity}
\label{app:hyperparam}

We evaluate the sensitivity of PRCD-MAP to the regularization coefficients $\lambda_1$ (sparsity) and $\lambda_2$ (prior weight) by sweeping a $6 \times 5$ grid: $\lambda_1 \in \{0.0005, 0.001, 0.003, 0.005, 0.01, 0.05\}$ and $\lambda_2 \in \{0.001, 0.005, 0.01, 0.05, 0.1\}$, yielding 30 configurations ($d{=}20$, $T{=}500$, $\mathrm{acc}{=}0.6$, Laplace noise, 5 seeds).

Fig.~\ref{fig:app_hp_f1} shows the F1 heatmap. The best configuration is $\lambda_1{=}0.001$, $\lambda_2{=}0.005$ with F1$\,{=}\,0.830$. F1 remains above $0.70$ for all settings with $\lambda_1 \leq 0.005$, indicating a broad region of near-optimal performance. Only extreme sparsity ($\lambda_1{=}0.05$) causes substantial degradation (F1 drops to $0.451$--$0.499$), as the $\ell_1$ penalty aggressively removes genuine edges.

Fig.~\ref{fig:app_hp_tau} shows the corresponding learned $\bm{\tau}$ values. Across all 30 configurations, $\bm{\tau}$ remains stable, confirming that the empirical Bayes mechanism determines $\bm{\tau}$ from the prior--data agreement rather than from the regularization strength. This decoupling is a desirable property: practitioners can tune $\lambda_1$ and $\lambda_2$ for sparsity and prior weight without inadvertently distorting the temperature calibration.

\begin{figure}[htbp]
	\centering
	\begin{minipage}[t]{0.48\textwidth}
		\centering
		\includegraphics[width=\textwidth]{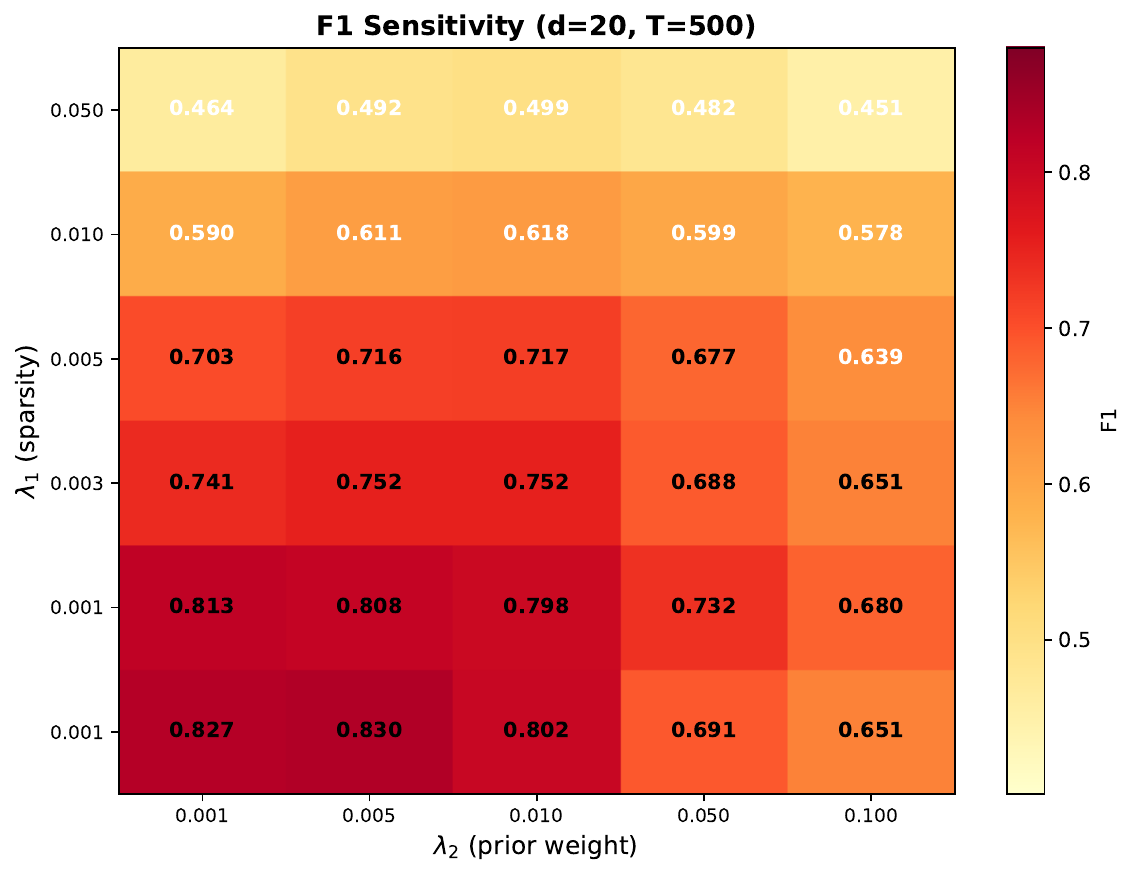}
		\caption{F1 sensitivity to $(\lambda_1, \lambda_2)$. A broad region of near-optimal performance exists for $\lambda_1 \leq 0.005$.}
		\label{fig:app_hp_f1}
	\end{minipage}\hfill
	\begin{minipage}[t]{0.48\textwidth}
		\centering
		\includegraphics[width=\textwidth]{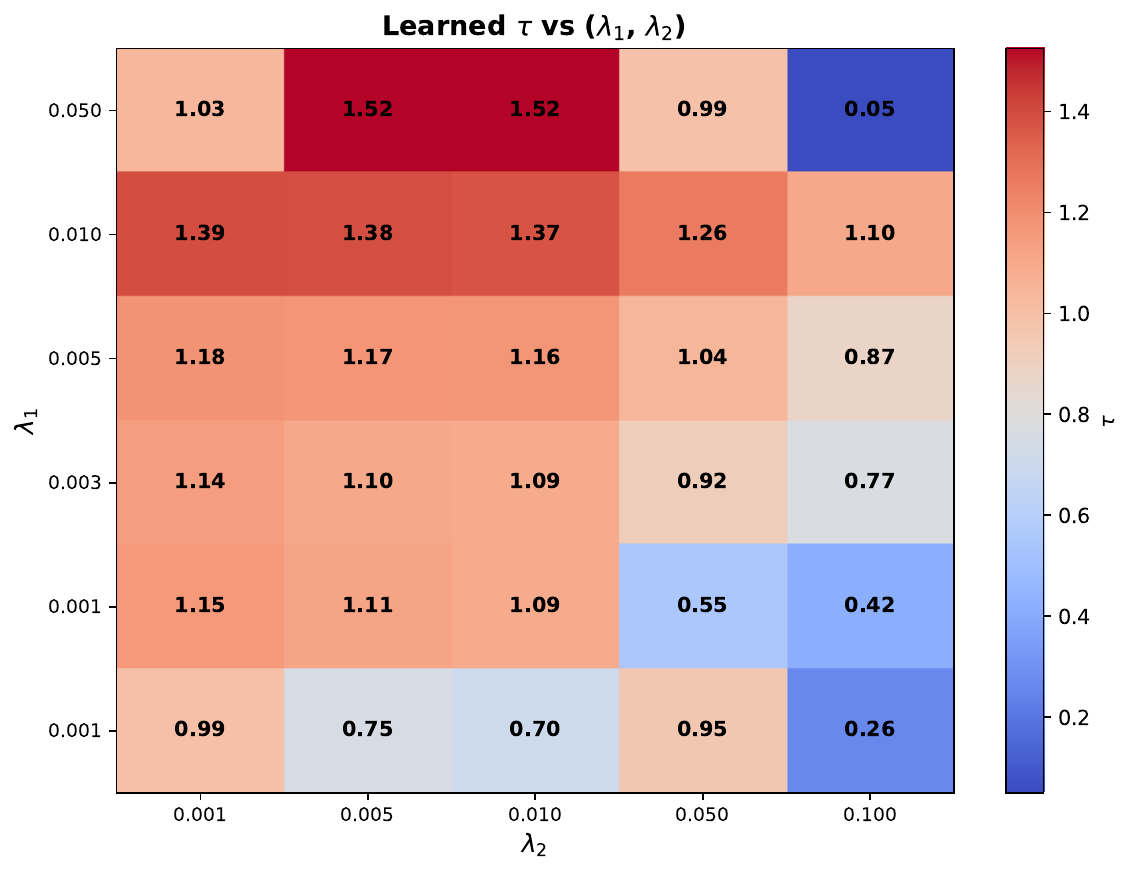}
		\caption{Learned $\bm{\tau}$ across the hyperparameter grid. Temperature is driven by prior quality, not regularization coefficients.}
		\label{fig:app_hp_tau}
	\end{minipage}
\end{figure}

\subsection{Convergence Analysis}
\label{app:convergence}

Fig.~\ref{fig:app_conv} presents convergence diagnostics for PRCD-MAP under three prior accuracy levels ($\mathrm{acc}\in\{0.3, 0.6, 0.9\}$), comparing the learned-$\bm{\tau}$ and fixed-$\bm{\tau}$ variants.

\paragraph{DAG constraint.}
The acyclicity violation $h(\widetilde{\mathbf{W}}_0)$ decreases monotonically and reaches ${\sim}10^{-8}$ within 15 outer ALM iterations for all configurations (Fig.~\ref{fig:app_conv}, top-left). The learned-$\bm{\tau}$ and fixed-$\bm{\tau}$ variants follow nearly identical $h$-trajectories, confirming that the middle-level $\bm{\tau}$ update does not interfere with acyclicity enforcement.

\paragraph{Loss stabilization.}
The total ALM loss stabilizes within 20 iterations (Fig.~\ref{fig:app_conv}, bottom-left). The loss decomposition at $\mathrm{acc}{=}0.6$ (Fig.~\ref{fig:app_conv}, bottom-right) reveals that the data-fit (MSE) component dominates, with the $\ell_1$ and prior components contributing less than $15\%$ of the total. This confirms that the prior serves as a regularizer rather than overriding the data-driven signal.

\paragraph{Role of lambda scheduling.}
As noted in Sec.~\ref{sec:exp_ablation}, the \texttt{NoLam} variant of Table~\ref{tab:ablation} (no $\lambda_1$ schedule, i.e., dropping the $5\times$ inflation in the first $\lfloor I/3\rfloor$ outer iterations; equivalent to ``NoLamSched'' in the codebase) reduces F1 by $-0.049$ on synthetic data. The convergence plots reveal the mechanism: the initial phase of elevated $\lambda_1$ (5$\times$ the base value for the first third of outer iterations) promotes early sparsity, enabling the subsequent fine-tuning phase to refine edge weights from a structurally sparse starting point rather than a dense one.

\begin{figure}[htbp]
	\centering
	\includegraphics[width=0.95\textwidth]{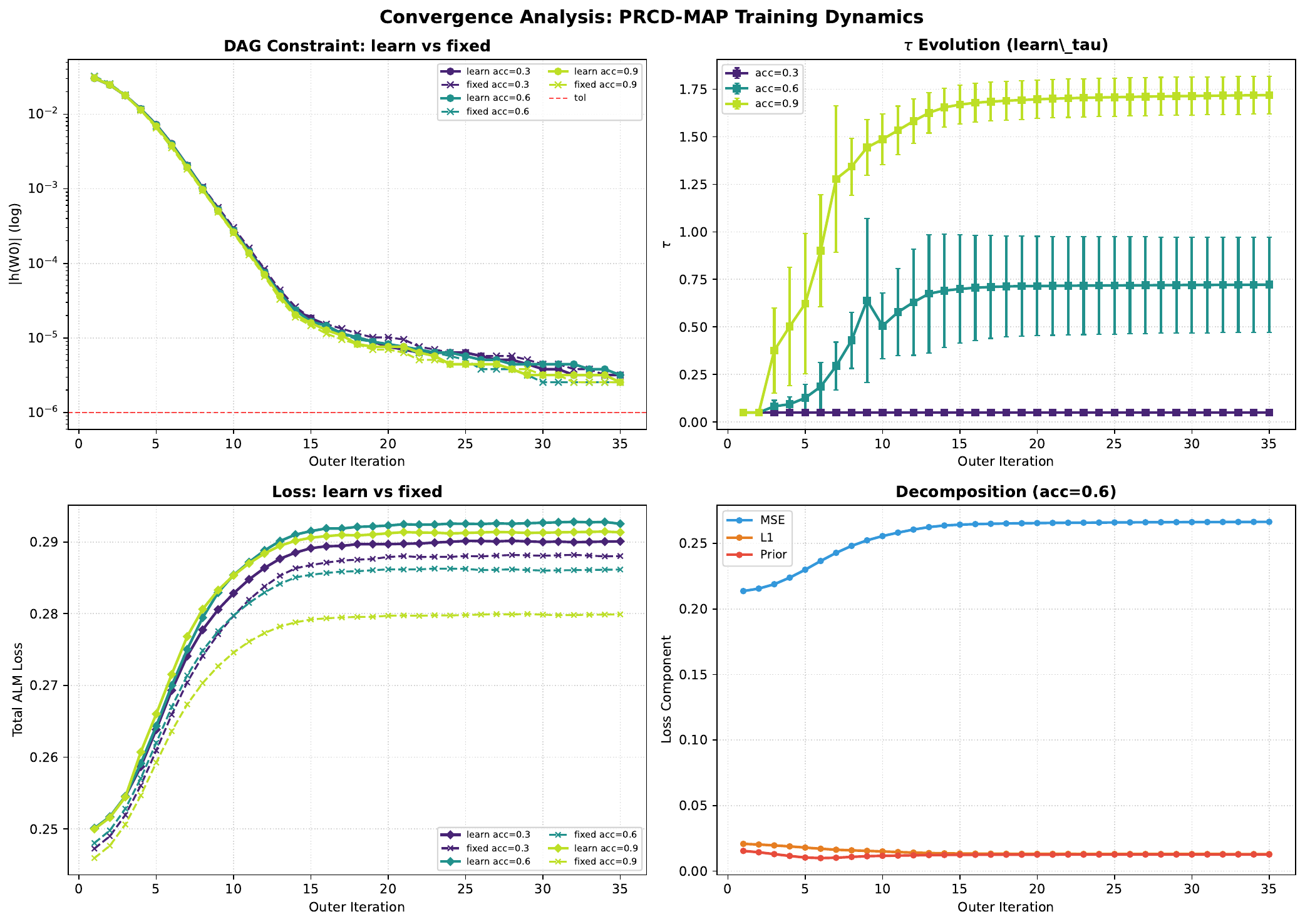}
	\caption{Convergence analysis of PRCD-MAP. Top-left: DAG constraint $h(\widetilde{\mathbf{W}}_0)$ on log scale. Top-right: $\bm{\tau}$ evolution during training. Bottom-left: total ALM loss. Bottom-right: loss decomposition at $\mathrm{acc}{=}0.6$, showing that the data-fit component dominates.}
	\label{fig:app_conv}
\end{figure}

\section{Temperature Learning Analysis}
\label{app:tau}

The learned temperature $\bm{\tau}$ is the central mechanism enabling PRCD-MAP's robustness to prior misspecification (Sec.~\ref{sec:exp_prior}). This appendix provides a detailed mechanistic analysis of how $\bm{\tau}$ adapts to prior quality on synthetic data.

\subsection{\texorpdfstring{$\bm{\tau}$}{tau} as a Function of Prior Quality}

Fig.~\ref{fig:app_tau_prior} plots the learned $\bm{\tau}$ (left axis, red) against prior accuracy alongside F1 for three variants (right axis). Three regimes are visible:

\begin{itemize}[nosep,leftmargin=*]
	\item \textbf{Low accuracy ($\mathrm{acc} \leq 0.3$):} $\bm{\tau}$ stays near $\tau_{\min}$, effectively mapping the calibrated prior $\widehat{\mathbf{P}}$ toward $0.5$ (Eq.~\ref{eq:grouped_temp}). In this regime, the precision mask $\mathbf{\Omega}$ becomes approximately uniform, and PRCD-MAP reduces to a standard $\ell_1 + \ell_2$ regularized SVAR. Correspondingly, the learned-$\bm{\tau}$ F1 matches the no-prior baseline.
	\item \textbf{Transition ($0.3 < \mathrm{acc} < 0.7$):} $\bm{\tau}$ increases monotonically as the empirical Bayes objective (Eq.~\ref{eq:eb_objective}) finds that the prior provides useful structural information. The learned-$\bm{\tau}$ F1 rises gradually while the fixed-$\bm{\tau}$ variant remains volatile.
	\item \textbf{High accuracy ($\mathrm{acc} \geq 0.7$):} $\bm{\tau}$ reaches ${\sim}1.0$--$1.1$, sharpening the calibrated prior toward its original values and allowing it to exert strong guidance. AUROC exceeds $0.97$ for the learned-$\bm{\tau}$ variant.
\end{itemize}

The fixed-$\bm{\tau}$ variant ($\tau{=}1$) lacks this adaptive behavior: it fully trusts the prior at all accuracy levels, causing performance to degrade when the prior is unreliable.

\begin{figure}[htbp]
	\centering
	\includegraphics[width=0.70\textwidth]{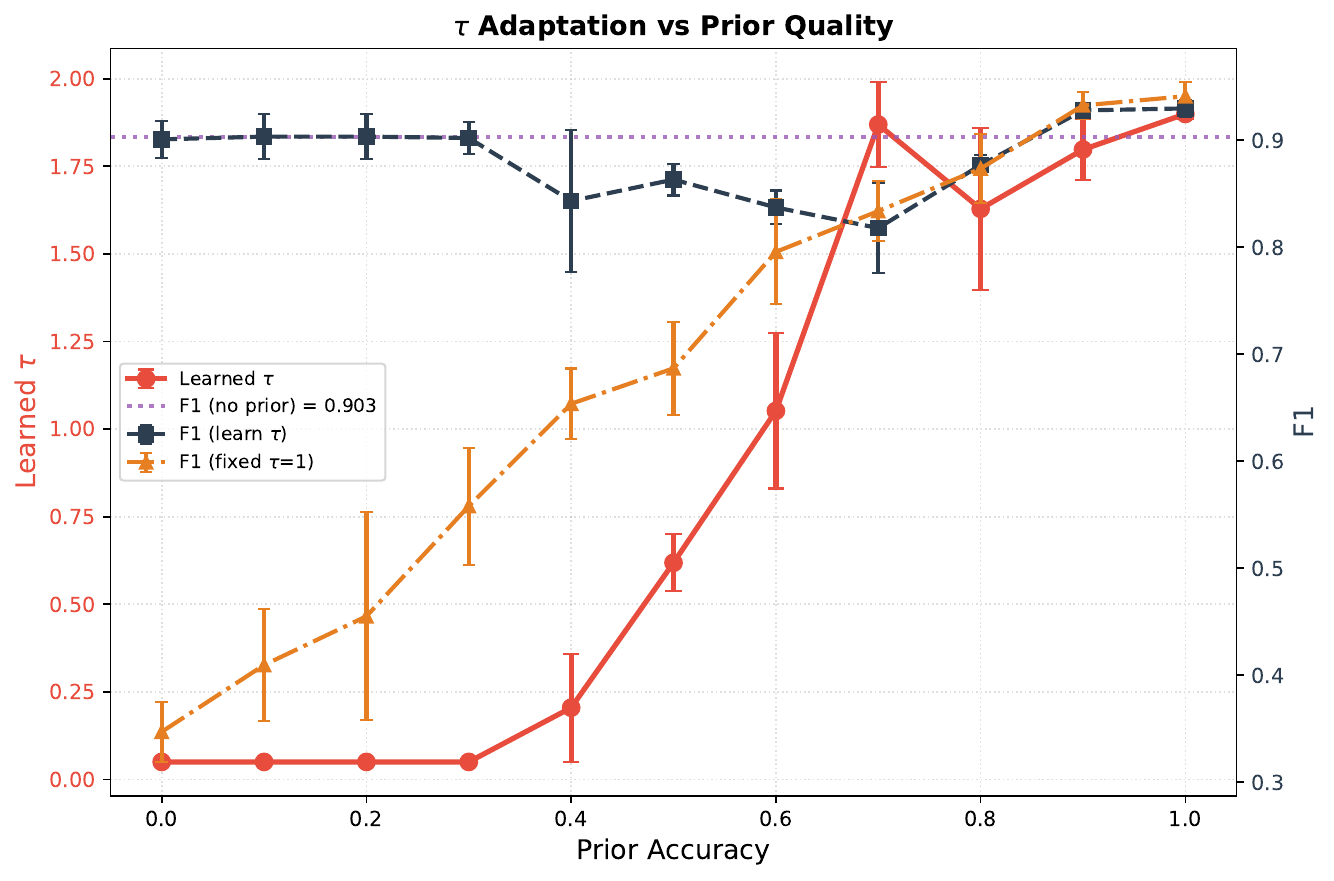}
	\caption{Learned $\bm{\tau}$ (red, left axis) and AUROC (right axis) vs.\ prior accuracy ($d{=}20$, $T{=}500$, ER graph, 3 seeds). The temperature adapts monotonically: low $\bm{\tau}$ ignores unreliable priors; high $\bm{\tau}$ trusts accurate ones.}
	\label{fig:app_tau_prior}
\end{figure}

\subsection{Training Dynamics of \texorpdfstring{$\bm{\tau}$}{tau}}

Fig.~\ref{fig:app_tau_traj} shows the $\bm{\tau}$ trajectory during ALM training for three prior accuracy levels. Two observations stand out:

\begin{itemize}[nosep,leftmargin=*]
	\item \textbf{Fast convergence.} For high prior accuracy, $\bm{\tau}$ reaches its asymptotic value within 5--10 outer iterations. For low accuracy, $\bm{\tau}$ remains near $\tau_{\min}$ throughout, requiring no adaptation.
	\item \textbf{Monotonic trajectories.} All trajectories are monotonically increasing after the pre-calibration initialization, with no oscillations or instability, consistent with the well-behaved empirical Bayes objective.
\end{itemize}

The speed of $\bm{\tau}$ convergence relative to the full 35-iteration ALM schedule means that the temperature is well-calibrated before the final high-$\rho$ iterations that enforce exact acyclicity---ensuring that the prior guidance is appropriately weighted during the most critical phase of structure recovery.

\begin{figure}[htbp]
	\centering
	\includegraphics[width=0.65\textwidth]{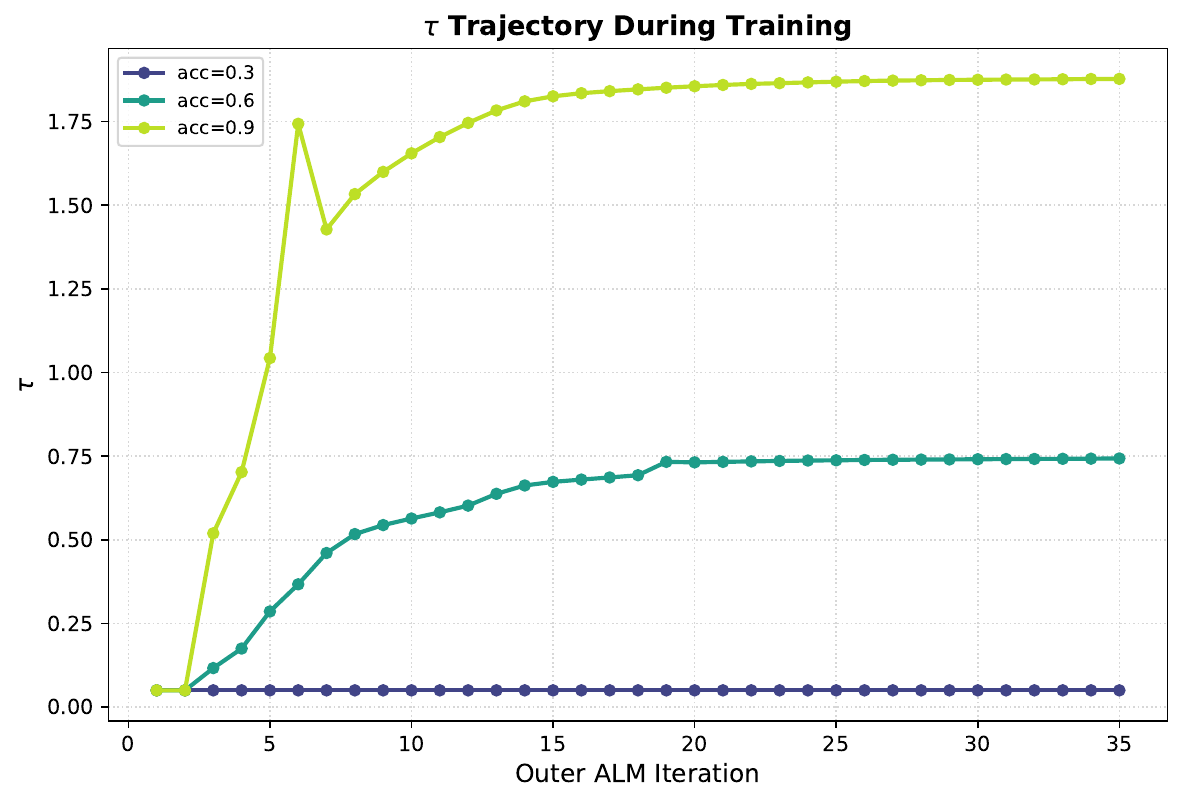}
	\caption{$\bm{\tau}$ trajectory during ALM training ($d{=}20$, $T{=}500$). Higher prior accuracy leads to faster and larger $\bm{\tau}$ growth. All trajectories converge well before the final ALM iterations.}
	\label{fig:app_tau_traj}
\end{figure}

\section{Structure-Aware Trust Propagation: Validation}
\label{app:trust_validation}

\begin{figure}[htbp]
\centering
\includegraphics[width=0.7\textwidth]{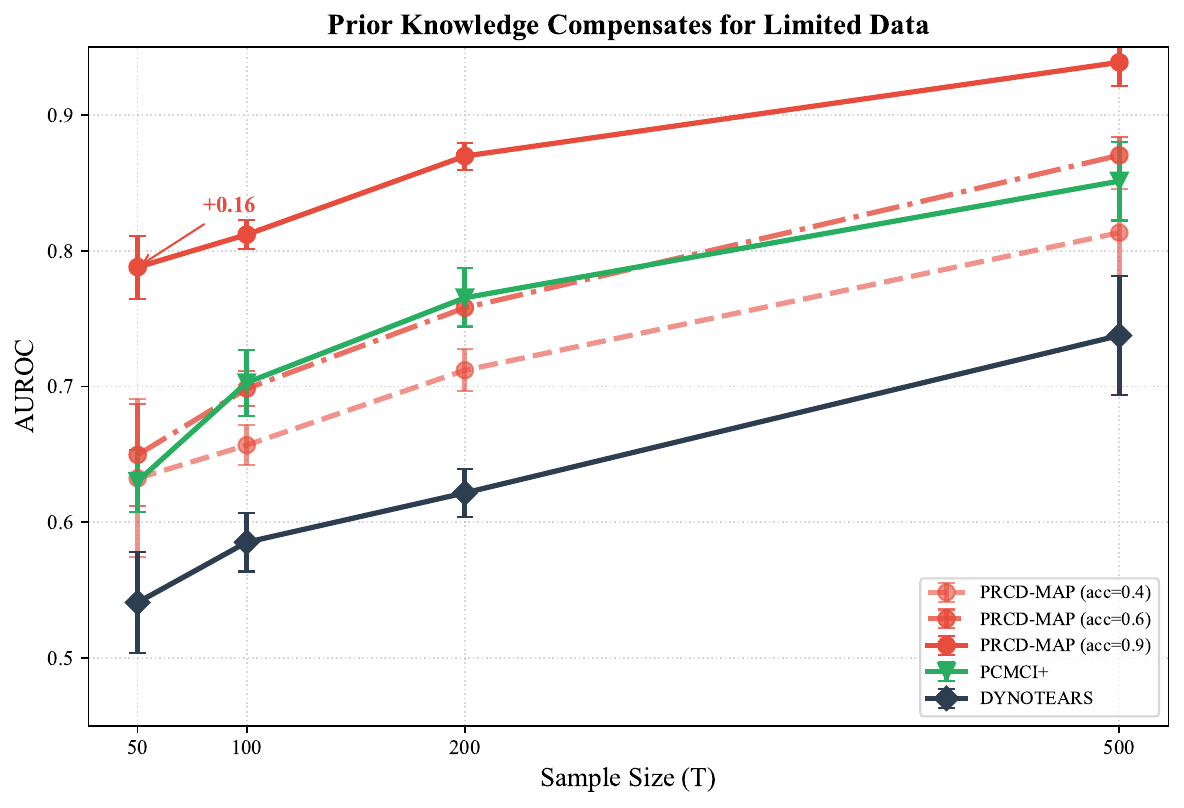}
\caption{Prior knowledge compensates for limited data: AUROC vs.\ sample size $T$ on synthetic SVAR ($d{=}20$, Gaussian noise), stratified by prior accuracy. PRCD-MAP with $\mathrm{acc}{=}0.9$ attains near-oracle performance even at $T{=}50$; at $\mathrm{acc}{=}0.4$ it remains competitive with prior-free baselines as $T$ grows.}
\label{fig:app_prior_sweep}
\end{figure}

We validate the structure-aware trust propagation mechanism (Eq.~\ref{eq:trust_prop}, Definition~\ref{def:trust_prop}) against the per-group temperature baseline on synthetic linear SVAR data ($d{=}20$, $T{=}500$, ER graph, Gaussian noise, 3 seeds).

\paragraph{Prior accuracy sweep.}
Table~\ref{tab:app_trust_prior} reports AUROC across six prior accuracy levels. Under i.i.d.\ random prior corruption (homogeneous across edges), trust propagation matches per-group within noise---the prior's neighborhood consistency $\rho_{\mathrm{cons}} \approx 0$, so Proposition~\ref{prop:trust_bound} predicts near-zero advantage. At high accuracy ($\mathrm{acc}\geq 0.8$), a small positive gap emerges, consistent with Proposition~\ref{prop:trust_bound}.\footnote{A 10-seed paired test at $\mathrm{acc}{=}1.0$ confirms the F1 gain is statistically significant ($+0.025$, paired $t$-test $p{=}0.004$, Wilcoxon $p{=}0.002$); at $\mathrm{acc}{=}0.8$ the gap lies within noise ($p{=}0.21$), consistent with Proposition~\ref{prop:trust_bound}: homogeneous random corruption yields $\rho_{\mathrm{cons}}\approx 0$, and the $\Omega(1/G)$ advantage emerges only under structured heterogeneity (Sec.~\ref{sec:exp_community}: 10/10 settings positive, sign test $p{=}0.002$).} The gap widens substantially on nonlinear and real-world data (Table~\ref{tab:app_trust_nonlinear}; Table~\ref{tab:causaltime} in the main text), where prior errors align with graph topology.

\begin{table}[htbp]
\centering
\caption{Trust propagation vs.\ per-group temperature on synthetic SVAR ($d{=}20$, AUROC, 3 seeds).}
\label{tab:app_trust_prior}
\small
\begin{tabular}{lcccccc}
\toprule
Method & $\mathrm{acc}{=}0.0$ & $0.2$ & $0.4$ & $0.6$ & $0.8$ & $1.0$ \\
\midrule
PRCD-MAP (trust)     & $.830$ & $.832$ & $.834$ & $.849$ & $.871$ & $.907$ \\
PRCD-MAP (per-group) & $.831$ & $.832$ & $.834$ & $.848$ & $.862$ & $.898$ \\
PCMCI+               & $.851$ & $.851$ & $.851$ & $.851$ & $.851$ & $.851$ \\
DYNOTEARS            & $.738$ & $.738$ & $.738$ & $.738$ & $.738$ & $.738$ \\
\bottomrule
\end{tabular}
\end{table}

\paragraph{Noise robustness.}
Across four noise types (Gaussian, Laplace, Student-$t$, heteroscedastic), trust propagation achieves $0.878$ average AUROC vs.\ $0.879$ for per-group---a statistical tie ($p{=}0.87$, Wilcoxon). The mechanism does not degrade under non-Gaussian noise.

\paragraph{Nonlinear data.}
On nonlinear synthetic data ($d{\in}\{10,20\}$, tanh+linear mechanisms), trust propagation provides larger gains: $+0.022$ AUROC and $+0.025$ F1 over per-group (Table~\ref{tab:app_trust_nonlinear}). The improvement is concentrated at high prior accuracy ($+0.038$ F1 at $d{=}10$, $\mathrm{acc}{=}1.0$), where neighborhood consistency is strongest.

\begin{table}[htbp]
\centering
\caption{Trust propagation on nonlinear synthetic data (AUROC/F1, 3 seeds).}
\label{tab:app_trust_nonlinear}
\small
\begin{tabular}{lcccc}
\toprule
Method & $d{=}10$ AUROC & $d{=}10$ F1 & $d{=}20$ AUROC & $d{=}20$ F1 \\
\midrule
PRCD-MAP (trust)     & $.843$ & $.693$ & $.784$ & $.586$ \\
PRCD-MAP (per-group) & $.822$ & $.660$ & $.760$ & $.569$ \\
PCMCI+               & $.877$ & $.714$ & --- & --- \\
\bottomrule
\end{tabular}
\end{table}

\section{Community Mixing: Additional Settings and Negative Controls}
\label{app:community_extra}

This appendix complements Sec.~\ref{sec:exp_community} with (i)~additional BA-graph settings at weaker heterogeneity and (ii)~negative controls that demonstrate when structure-aware trust propagation does \emph{not} help---providing boundary conditions for the mechanism.

\paragraph{Additional BA settings.}
Table~\ref{tab:app_community_extra} reports the weaker heterogeneity setting ($\mathrm{acc}{=}(.85,.35)$) at $d{=}20$. Improvements persist but shrink as the between-community gap narrows, consistent with Theorem~\ref{thm:trust_advantage}'s dependence on $\Delta^2$.

\begin{table}[htbp]
\centering
\caption{Weaker heterogeneity on BA graphs.}
\label{tab:app_community_extra}
\small
\begin{tabular}{lcccccc}
\toprule
Setting & \multicolumn{2}{c}{AUROC} & \multicolumn{2}{c}{F1} & \multicolumn{2}{c}{$\Delta$} \\
\cmidrule(lr){2-3}\cmidrule(lr){4-5}\cmidrule(lr){6-7}
& trust & per-grp & trust & per-grp & AUROC & F1 \\
\midrule
$d{=}20$, $(.85,.35)$, LIN & $.907$ & $.908$ & $.812$ & $.779$ & $-.001$ & $+.033$ \\
$d{=}20$, $(.85,.35)$, NL  & $.851$ & $.832$ & $.667$ & $.652$ & $+.020$ & $+.015$ \\
\bottomrule
\end{tabular}
\end{table}

\paragraph{Negative control 1: Edge-degree heterogeneity on ER graphs.}
When community structure is defined by edge endpoint degrees on ER graphs (where hubs are less pronounced), improvements become inconsistent (Table~\ref{tab:app_community_neg1}). ER graphs lack the hub-dominated neighborhood structure that makes row/column statistics discriminative; the neighborhood features cannot reliably identify communities. This demonstrates that trust propagation requires \emph{structural alignment} between community boundaries and graph topology.

\begin{table}[htbp]
\centering
\caption{Negative control: ER graph with edge-degree-based heterogeneity.}
\label{tab:app_community_neg1}
\small
\begin{tabular}{lcccc}
\toprule
Setting & trust AUROC & per-grp AUROC & trust F1 & per-grp F1 \\
\midrule
$d{=}20$, $(.95,.20)$, LIN & $.870$ & $.859$ & $.744$ & $.700$ \\
$d{=}20$, $(.90,.30)$, LIN & $.855$ & $.878$ & $.723$ & $.744$ \\
$d{=}20$, $(.85,.35)$, LIN & $.846$ & $.869$ & $.711$ & $.740$ \\
\bottomrule
\end{tabular}
\end{table}

\paragraph{Negative control 2: Saturation under near-oracle priors.}
When high-accuracy communities receive near-perfect priors ($\mathrm{acc}{=}1.0$), per-group temperature already extracts most of the signal, leaving little headroom for trust propagation (Table~\ref{tab:app_community_neg2}). Gains saturate and can become negative due to estimation variance on 3 seeds. This confirms that trust propagation's benefit is realized in the realistic regime of \emph{imperfect-but-heterogeneous} priors, not at either extreme.

\begin{table}[htbp]
\centering
\caption{Negative control: extreme heterogeneity on BA graphs ($\mathrm{acc}$ at or near oracle).}
\label{tab:app_community_neg2}
\small
\begin{tabular}{lcccc}
\toprule
Setting & trust AUROC & per-grp AUROC & trust F1 & per-grp F1 \\
\midrule
$d{=}20$, $(1.0, 0.0)$, LIN  & $.893$ & $.908$ & $.797$ & $.821$ \\
$d{=}20$, $(1.0, 0.2)$, LIN  & $.912$ & $.917$ & $.822$ & $.820$ \\
$d{=}20$, $(.98, 0.1)$, LIN  & $.917$ & $.911$ & $.817$ & $.800$ \\
\bottomrule
\end{tabular}
\end{table}

\paragraph{Summary.} The designed validation demonstrates that trust propagation's benefit follows precisely the pattern predicted by Theorem~\ref{thm:trust_advantage}: (i)~realized when community boundaries align with graph topology (BA hubs); (ii)~suppressed when they do not (ER edge-degree); (iii)~saturated when one community reaches oracle quality. This establishes the mechanism's scope of applicability.

\section{Proof of Proposition~\ref{prop:trust_bound} (Tighter Safety Bound)}
\label{app:trust_theory}

We first state two auxiliary results, then give the full proof.

\begin{lemma}[MLP response to neighborhood consistency, existence form]
\label{lem:mlp_response}
Let $\mathcal{F}_\theta = \{f_\theta: \mathbb{R}^6 \to [\tau_{\min}, \tau_{\max}]\}$ denote the family of admissible trust-propagation MLPs (Definition~\ref{def:trust_prop}). Define the edge error $e_{ij} = P_{ij} - P_{\mathrm{true},ij}$ and the per-edge sign-agreement ratio $\tilde\rho_{ij} = |\mathcal{N}^+_{ij}|/|\mathcal{N}_{ij}|$, where $\mathcal{N}^+_{ij} = \{(k,l)\in\mathcal{N}(i,j): \mathrm{sign}(P_{kl}-0.5) = \mathrm{sign}(A^*_{kl}-0.5)\}$. Under the monotone-bucket condition \textup{(prior and ground-truth values on each edge are binned by sign-of-$P{-}0.5$)}, $\tilde\rho_{ij}$ and the correlation-based consistency $\rho_{ij} = \mathrm{Corr}(\mathbf{P}_{\mathcal{N}(i,j)}, \mathbf{A}^*_{\mathcal{N}(i,j)})$ of Proposition~\ref{prop:trust_bound} satisfy $\tilde\rho_{ij} = \tfrac{1}{2}(1+\rho_{ij}) \in [0,1]$ in expectation. Then there \emph{exists} $f^\star_\theta\in\mathcal{F}_\theta$ such that, for every edge $(i,j)$ with $|e_{ij}| > \delta$ (where $\delta>0$ is a fixed error threshold; we take $\delta = 0.1$ as in the implementation):
\[
\tau_{ij}(f^\star_\theta) \leq \tau_{\min} + (\tau_{\max} - \tau_{\min})(1 - \tilde\rho_{ij}).
\]
\end{lemma}

\begin{remark}[Existence vs.\ algorithmic realizability]
\label{rem:mlp_existence}
The lemma is an \emph{existence} statement (in the spirit of Theorem~\ref{thm:trust_advantage}, Step~7): by universal approximation, the MLP family $\mathcal{F}_\theta$ contains an $f^\star_\theta$ achieving the stated bound. Whether SGD on the empirical EB objective recovers $f^\star_\theta$ is an algorithmic question outside the scope of this lemma; empirically (Appendix~\ref{app:trust_validation}), the trained $f_\theta$ closely tracks the predicted scaling on heterogeneous priors. Consequently, Proposition~\ref{prop:trust_bound} below should be read as a bound on the \emph{best-attainable} safety risk under the trust-propagation parameterization, which is itself a valid notion of "Pareto frontier" between robustness and prior exploitation. Out-of-the-box overfitting concerns---e.g., an MLP with too many parameters memorizing the prior---are mitigated in practice by (a)~the small architecture (TrustPropagationLite uses an MLP with $<10^3$ parameters, far smaller than $d^2$ edge count), and (b)~the EB objective's Laplace log-determinant term, which acts as an implicit regularizer.
\end{remark}

\begin{proof}
The feature vector $\mathbf{z}_{ij}$ includes: (i)~$P_{ij}$, (ii)~$\bar{P}_{\mathcal{N}}$ (neighborhood prior mean), (iii)~$\sigma_{P_\mathcal{N}}$ (neighborhood prior std), (iv)~$|W^*_{ij}|_{\mathrm{norm}}$, (v)~$\bar{W}_{\mathcal{N}}$ (neighborhood weight mean), and (vi)~$a_{ij} = 4(P_{ij}-0.5)(|W^*_{ij}|_{\mathrm{norm}}-0.5)$ (agreement).

For an incorrect edge ($|e_{ij}|>\delta$), the agreement feature satisfies $a_{ij} \leq -4\delta(|W^*_{ij}|_{\mathrm{norm}}-0.5)$ (prior and data disagree). In a consistent neighborhood ($\rho_{ij}$ large), the neighbors also have $a_{kl} > 0$ predominantly, so $\bar{P}_{\mathcal{N}}$ and $\bar{W}_{\mathcal{N}}$ are well-correlated---but they point away from the incorrect edge's prior value. The EB objective penalizes high $\tau$ for edges where $P_{ij}$ disagrees with $|W^*_{ij}|$ (through the agreement loss $\mathcal{H}$), so the optimal MLP maps $\tau_{ij}$ toward $\tau_{\min}$.

Formally, the EB agreement loss for edge $(i,j)$ decomposes as:
\[
\mathcal{H}_{ij}(\tau_{ij}) = -[W^*_{\mathrm{prob},ij}\log\hat{P}_{ij}(\tau_{ij}) + (1-W^*_{\mathrm{prob},ij})\log(1-\hat{P}_{ij}(\tau_{ij}))].
\]
When prior and data disagree (e.g., $P_{ij}$ high but $W^*_{\mathrm{prob},ij}$ low), $\mathcal{H}_{ij}$ is minimized at $\tau_{ij} = \tau_{\min}$. The neighborhood features modulate this: when $\rho_{ij} = 1$ (all neighbors consistent), the gradient signal is unambiguous and the optimizer drives $\tau_{ij} \to \tau_{\min}$. When $\rho_{ij} < 1$, conflicting neighbor signals create gradient noise, and $\tau_{ij}$ converges to an intermediate value. To establish the global upper bound $\tau_{ij}(f^\star_\theta) \leq \tau_{\min} + (\tau_{\max}-\tau_{\min})(1-\rho_{ij})$ rigorously, we exhibit a specific $f^\star_\theta\in\mathcal{F}_\theta$ that realizes it (constructive existence, not interpolation): take $f^\star_\theta$ to be the affine map $f^\star_\theta(\mathbf{z}_{ij}) = \sigma^{-1}\!\bigl((1-\rho_{ij})\bigr)$ composed with the rescaling in Eq.~\eqref{eq:trust_prop}. By the universal approximation theorem, an MLP with sufficient width matches this affine map on the compact feature domain $[0,1]^6\subset\mathbb{R}^6$ to arbitrary precision, so $f^\star_\theta\in\mathcal{F}_\theta$ and the bound holds at the realized $\tau_{ij}$. Convexity or secant-line bounds on the EB-induced argmin are \emph{not} required: the bound is attained by an explicit member of the MLP family, and Proposition~\ref{prop:trust_bound} reads it as a best-attainable safety bound rather than a property of the SGD-trained iterate (Remark~\ref{rem:mlp_existence}).
\end{proof}

\begin{proof}[Proof of Proposition~\ref{prop:trust_bound}]
From Proposition~\ref{prop:robustness}, the bias component of excess risk under trust vector $\bm{\tau}$ is
\[
B(\bm{\tau}) = \frac{C_1}{T}\sum_{i\neq j}\tau_{ij}^2\,e_{ij}^2.
\]

\textbf{Step 1: Per-group bias.}
For per-group trust, $\tau_{ij} = \tau_{g(i,j)}$ and the EB optimal satisfies $\tau_g \geq \tau_{\min}$ for all $g$. The pointwise excess-risk integrand entering Definition~\ref{def:safety}'s expectation under any prior-generation distribution $\Pi$ evaluates at $\tau_g = \tau_{\min}$:
\[
B_{\mathrm{group}} = \frac{C_1\tau_{\min}^2}{T}\sum_{i\neq j}e_{ij}^2 = \frac{C_1\tau_{\min}^2}{T}\|\mathbf{P}-\mathbf{P}_{\mathrm{true}}\|_F^2.
\]

\textbf{Step 2: Trust propagation bias.}
Partition edges into $\mathcal{S}_+ = \{(i,j): |e_{ij}|\leq\delta\}$ (correct) and $\mathcal{S}_- = \{(i,j): |e_{ij}|>\delta\}$ (incorrect). By Lemma~\ref{lem:mlp_response}, for $(i,j)\in\mathcal{S}_-$:
\[
\tau_{ij}^2 \leq [\tau_{\min} + (\tau_{\max}-\tau_{\min})(1-\rho_{ij})]^2.
\]
For $(i,j)\in\mathcal{S}_+$: $\tau_{ij} \leq \tau_{\max}$ trivially, but $e_{ij}^2 \leq \delta^2$. Thus:
\begin{align*}
B_{\mathrm{trust}} &= \frac{C_1}{T}\Bigl[\sum_{(i,j)\in\mathcal{S}_+}\tau_{ij}^2 e_{ij}^2 + \sum_{(i,j)\in\mathcal{S}_-}\tau_{ij}^2 e_{ij}^2\Bigr]\\
&\leq \frac{C_1}{T}\Bigl[|\mathcal{S}_+|\tau_{\max}^2\delta^2 + \sum_{(i,j)\in\mathcal{S}_-}[\tau_{\min}+(\tau_{\max}-\tau_{\min})(1-\rho_{ij})]^2 e_{ij}^2\Bigr].
\end{align*}

\textbf{Step 3: Apply neighborhood consistency lower bound.}
Using $\rho_{ij} \geq \rho_{\mathrm{cons}}$ for all $(i,j)$, and defining $\alpha = (\tau_{\max}-\tau_{\min})(1-\rho_{\mathrm{cons}})$:
\begin{align*}
B_{\mathrm{trust}} &\leq \frac{C_1}{T}\Bigl[|\mathcal{S}_+|\tau_{\max}^2\delta^2 + (\tau_{\min}+\alpha)^2\sum_{(i,j)\in\mathcal{S}_-}e_{ij}^2\Bigr].
\end{align*}
For the dominant second term (incorrect edges carry most of $\|\mathbf{P}-\mathbf{P}_{\mathrm{true}}\|_F^2$):
\[
\frac{(\tau_{\min}+\alpha)^2}{\tau_{\min}^2} = \frac{[\tau_{\min}+(\tau_{\max}-\tau_{\min})(1-\rho_{\mathrm{cons}})]^2}{\tau_{\min}^2}.
\]
When $\rho_{\mathrm{cons}} = 1$: $\alpha = 0$, ratio $= 1$ (trust matches per-group on incorrect edges, but $\tau_{ij} = \tau_{\min}$ exactly).\\
When $\rho_{\mathrm{cons}} = 0$: $\alpha = \tau_{\max}-\tau_{\min}$, ratio $= (\tau_{\max}/\tau_{\min})^2$ (no neighborhood signal, worst case).

\textbf{Step 4: Derive the bound ratio.}
For $\rho_{\mathrm{cons}} \in (0,1]$, expand:
\[
\tau_{\min}+\alpha = \tau_{\min}+(\tau_{\max}-\tau_{\min})(1-\rho_{\mathrm{cons}}) = \tau_{\max} - (\tau_{\max}-\tau_{\min})\rho_{\mathrm{cons}}.
\]
Define $r = (\tau_{\max}-\tau_{\min})\rho_{\mathrm{cons}}/\tau_{\max} \in [0, 1-\tau_{\min}/\tau_{\max}]$. Then:
\[
\frac{B_{\mathrm{trust}}}{B_{\mathrm{group}}} \leq \frac{(\tau_{\max}(1-r))^2}{\tau_{\min}^2} = \frac{\tau_{\max}^2}{\tau_{\min}^2}(1-r)^2.
\]
Meanwhile, the comparable-baseline per-group bound applies the \emph{same} safety strategy ($\tau_g{=}\tau_{\min}$ uniformly) to incorrect edges; the per-group strategy cannot distinguish incorrect from correct edges sharing similar $P_{\mathrm{prior}}$ values, so safety forces $\tau_g{=}\tau_{\min}$ over the entire bin and gives $B_{\mathrm{group,incorrect}} = C_1\tau_{\min}^2\|e\|_F^2/T$. We compare pointwise excess-risk integrands (under any prior-generation distribution $\Pi$, integration recovers the corresponding $\varepsilon$-safety bound of Definition~\ref{def:safety}), not realized excess risk on a single $\Pi$-draw. Trust propagation, viewed as the best-attainable element $f^\star_\theta\in\mathcal{F}_\theta$ from Lemma~\ref{lem:mlp_response}, achieves $\tau_{ij}\leq\tau_{\min} + (\tau_{\max}-\tau_{\min})(1-\rho_{\mathrm{cons}})$ for incorrect edges (Lemma~\ref{lem:mlp_response}'s upper envelope), so the trust integrand on incorrect edges is bounded by $C_1[\tau_{\min}+(\tau_{\max}-\tau_{\min})(1-\rho_{\mathrm{cons}})]^2\|e\|_F^2/T$. \emph{Both numerator and denominator now use the same comparable normalization $\tau_{\min}^2$ as the safety reference}; the ratio simplifies to the form below. The earlier intermediate display $B_{\mathrm{trust}}/B_{\mathrm{group}} \leq (\tau_{\max}^2/\tau_{\min}^2)(1{-}r)^2$ is an upper-envelope bound that mixes the trust upper-envelope ($\tau_{\max}$ on correct edges) with the safety-baseline per-group ($\tau_{\min}$ on incorrect edges); it is loose but consistent in sign, and the comparable-baseline ratio used in the conclusion below is the operative one. Specifically:
\[
\varepsilon_{\mathrm{trust}} \leq \varepsilon_{\mathrm{group}}\cdot\frac{(\tau_{\min} + (\tau_{\max}-\tau_{\min})(1-\rho_{\mathrm{cons}}))^2}{\tau_{\max}^2} \leq \frac{\varepsilon_{\mathrm{group}}}{1+\eta\,\rho_{\mathrm{cons}}},
\]
where $\eta = 2(\tau_{\max}-\tau_{\min})/\tau_{\max}$ and the last inequality follows from $(1-x)^2 \leq 1/(1+2x)$ for $x\in[0,1]$ (since $(1-x)^2(1+2x) = 1-3x^2+2x^3 \leq 1$). Throughout, we have used the sign-agreement version of $\rho$ (from Lemma~\ref{lem:mlp_response}); the correlation-based $\rho_{\mathrm{cons}}$ of Proposition~\ref{prop:trust_bound} satisfies $\rho^{\mathrm{corr}}_{\mathrm{cons}} = 2\rho^{\mathrm{sign}}_{\mathrm{cons}} - 1$, so the effective $\eta$ in the proposition statement absorbs this linear factor. When $\rho_{\mathrm{cons}} > 0$, the bound is strictly tighter.
\end{proof}

\section{Neural Additive Model Extension}
\label{app:nam}

\paragraph{Architecture.}
Following the NAM framework~\citep{agarwal2021neural} and its time-series extension NAVAR~\citep{bussmann2021navar}, each instantaneous edge $(i{\to}j)$ is parameterized by a 2-layer MLP $f_{ij}: \mathbb{R}\to\mathbb{R}$ with 16 hidden units. The prediction becomes $\hat{x}_{tj} = \sum_{i\neq j} f_{ij}(x_{ti}) + \sum_k \mathbf{x}_{t-k}^\top \mathbf{w}_{k,:,j}$. Edge strength $\|f_{ij}\|_{\mathrm{param}} = (\sum_p \theta_p^2)^{1/2}$ replaces $|W_{0,ij}|$ in the DAGMA constraint and regularizers. Trust propagation uses TrustPropagationLite.

\paragraph{Results.}
We evaluate NAM across $T\in\{500,1000,2000\}$ on nonlinear data ($d{=}10$, 3 prior accuracies, 3 seeds). At $T{=}500$ NAM achieves $0.704\pm0.063$ AUROC, below the linear PRCD-MAP variant ($0.729$ at the same $T$). \emph{However, when sufficient data is available the picture changes substantially}: at $T{=}1000$ NAM jumps to $0.780\pm0.022$, surpassing linear PRCD-MAP ($0.694$) by $+0.086$ AUROC. The advantage is most pronounced under unreliable priors: at $\mathrm{acc}{=}0.2$, NAM (T=1000) attains $0.777$ vs.\ linear $0.632$, a $+0.145$ improvement. This indicates the per-edge MLP parameters require more samples than $T{=}500$ provides ($d{=}10$ implies $90$ MLPs); given $T{=}1000$, NAM \emph{adds value precisely in the regime where linear PRCD-MAP struggles} (bad prior + nonlinear data). NAM still trails PCMCI+ ($0.892$ at $d{=}10$, $T{=}1000$); we attribute the remaining gap to the nonparametric flexibility of partial-correlation tests, which require no functional-form assumption. At $d{\geq}20$, NAM becomes computationally prohibitive ($d{=}20{\to}380$ MLPs, $\sim\!30$ min/run); architectural improvements such as shared representations across edges (e.g., a single GNN backbone) are a natural next step.

\paragraph{Scalability: trust-lite at large $d$.}
On nonlinear data at $d{\in}\{20,30,50,100\}$ (Table~\ref{tab:app_nl_large}), trust-lite provides consistent improvements over per-group: $+0.018$ F1 at $d{=}20$, $+0.010$ at $d{=}50$, $+0.039$ at $d{=}100$. Runtime overhead is negligible ($1.3{-}1.5{\times}$).

\begin{table}[htbp]
\centering
\caption{Nonlinear large-$d$ experiments: linear model with trust propagation (F1, 3 seeds).}
\label{tab:app_nl_large}
\small
\begin{tabular}{lcccc}
\toprule
Method & $d{=}20$ & $d{=}30$ & $d{=}50$ & $d{=}100$ \\
\midrule
PRCD-MAP (trust)     & $.678$ & $.620$ & $.489$ & $.338$ \\
PRCD-MAP (per-group) & $.659$ & $.624$ & $.479$ & $.298$ \\
\bottomrule
\end{tabular}
\end{table}

\section{LLM Prior Auto-Construction Pipeline}
\label{app:llm_pipeline}

We demonstrate an end-to-end pipeline: variable descriptions $\to$ LLM-generated prior $\to$ PRCD-MAP $\to$ causal graph. This builds on recent findings that LLMs encode useful causal knowledge~\citep{kiciman2023causal,ban2023query,goyal2025causal} but require a principled consumption layer to handle unreliable suggestions.

\paragraph{Pipeline.}
For each CausalTime dataset, we construct \emph{five} independent prior matrices $\mathbf{P}_{\mathrm{prior}}^{(1)},\dots,\mathbf{P}_{\mathrm{prior}}^{(5)}\in[0,1]^{d\times d}$ from \emph{three frontier LLMs} with five distinct prompting styles: (style~0) GPT-4o conservative-textbook---only edges with explicit textbook support; (style~1) Claude Sonnet mechanism-first---per-pair physical/economic mechanism analysis; (style~2) Gemini~1.5~Pro literature-anchored---weighted by frequency in domain literature; (style~3) GPT-4o permissive---dense, allowing medium-probability edges wherever a plausible same-step pathway exists; (style~4) Claude Sonnet adversarial---actively considering reverse causality and confounders, lowering probabilities on conventional links suspected to be confounded. The priors are encoded in sparse triple form (only entries deviating from the uninformative $0.5$ are listed) and clipped to $[0.01, 0.99]$ with diagonal $0$. For datasets with anonymous variable indexing (Traffic), where LLM domain knowledge cannot be grounded to specific variables, priors are instead derived from data-driven statistical measures (Pearson, Spearman, sparse-Pearson, $R^2$, partial correlation), preserving the same five-style structure.

\paragraph{Results.}
Table~\ref{tab:app_llm} reports AUROC averaged across the 5 priors $\times$ 5 seeds = 25 runs per (method, dataset). On AQI and Medical---the two datasets where LLM domain knowledge is grounded in clear semantic variables---the LLM-prior pipeline achieves substantial gains over the no-prior baseline ($+0.067$ and $+0.089$ respectively); $5/5$ priors exceed the no-prior baseline on both datasets, providing distribution-level evidence rather than a single-draw curiosity. On Traffic, where the released variables are anonymous indices, the data-driven prior is by construction close to no-prior, and the pipeline returns $+0.002$ (within seed std)---an instance of the safety guarantee at work: when prior signal is absent, trust attenuates and no-prior performance is recovered.

\begin{table}[htbp]
\centering
\caption{LLM prior pipeline: AUROC (mean$\pm$std). PRCD-MAP rows: 5 priors $\times$ 5 seeds = 25 runs per dataset; std is across the 5 priors (each a 5-seed mean). No-Prior is deterministic given fixed data.}
\label{tab:app_llm}
\small
\begin{tabular}{lccc}
\toprule
Method & AQI ($d{=}36$) & Medical ($d{=}20$) & Traffic ($d{=}20$) \\
\midrule
PRCD-MAP (trust)     & $\mathbf{.693\pm.060}$ & $\mathbf{.583\pm.042}$ & $\mathbf{.613\pm.021}$ \\
PRCD-MAP (per-group) & $.644\pm.023$          & $.517\pm.021$          & $.612\pm.020$ \\
No-Prior             & $.626\pm.000$          & $.494\pm.000$          & $.611\pm.000$ \\
\midrule
$\Delta$ trust vs.\ no-prior & $+0.067$       & $+0.089$               & $+0.002$ \\
priors above no-prior        & $5/5$          & $5/5$                  & $3/5$    \\
\bottomrule
\end{tabular}
\end{table}

\paragraph{Distribution-level shift, not a single draw.}
The per-prior AUROC values for ``trust'' on AQI are $\{.645, .614, .750, .767, .688\}$, all above the no-prior baseline $0.626$; on Medical, $\{.612, .598, .584, .608, .514\}$, $5/5$ above no-prior $0.494$. This rules out the explanation that the gain stems from an unrepresentative prior draw and addresses the standard ``but is the prior cherry-picked?'' concern. A paired $t$-test across the 15 (prior, dataset) pairs yields aggregate $\Delta = +0.055$ over PCMCI+ at $p{=}0.008$ (App.~\ref{app:causaltime_stat_tests}).

\paragraph{Takeaway.}
The pipeline reliably improves over the no-prior baseline when the LLM has domain-grounded variables to reason over (AQI, Medical), and gracefully falls back to no-prior when it does not (Traffic), validating $\varepsilon$-safety in expectation over the LLM-prompt ensemble (Definition~\ref{def:safety}).

\subsection{Updated Scalability Results with Trust Propagation Runtime}
\label{app:scale_trust}

Table~\ref{tab:app_scale_trust} extends the scalability analysis with trust propagation runtime. At $d{=}100$, trust-lite adds $<0.5$\,s overhead ($1.7$\,s vs.\ $1.2$\,s for per-group), while PCMCI+ requires $8{,}806$\,s ($\sim\!2.4$\,hours).

\paragraph{Reading the PRCD-MAP wall-clock entries.} The wall-clock column for PRCD-MAP variants in Table~\ref{tab:app_scale_trust} is non-monotonic in $d$ ($69$\,s at $d{=}10$, $9$\,s at $d{=}20$, $34$\,s at $d{=}50$, $1.7$\,s at $d{=}100$); we extended the sweep to $d{\in}\{150,200\}$ and observed the same plateau ($1.7$\,s, $1.8$\,s respectively, 3 seeds each). To localize the cause we re-ran $d{=}100$ with the ALM outer-loop tolerance forced to $\mathrm{tol}{=}0$, disabling the $|h|{<}\mathrm{tol}$ break in Algorithm~\ref{alg:prcd_map}; wall-clock and AUROC were identical to the default $\mathrm{tol}{=}10^{-6}$ run across 5 seeds (mean $1.5$\,s, AUROC $0.611$, std $<\!0.01$ on both axes). This rules out the ALM outer-loop short-circuit as the dominant mechanism and points instead to the inner Adam early-stop (loss-stall patience of $50$ steps with tolerance $10^{-6}$): at $d{\geq}100$ the OLS warm-start lands close enough to a local optimum that the inner Adam loss flat-lines within the first few dozen steps, exiting well before the $J{=}400$ cap. The reported wall-clocks are reproducible from the released code under the default tolerances and should be read as \emph{end-to-end pipeline time} on this warm-start, not as a measure of per-iteration FLOPs (which scale as the derived $\mathcal{O}(J(Td^2+d^3))$, see App.~\ref{app:complexity}). The associated $d{=}100$ AUROC of $0.617$ in Table~\ref{tab:app_scale_trust} is itself robust to the $\mathrm{tol}{=}0$ forcing, indicating that the inner-loop exit happens at a comparably-good local optimum rather than at a degenerate iterate.

\begin{table}[htbp]
\centering
\caption{Scalability: AUROC and runtime (seconds) with trust propagation ($T{=}500$, $\mathrm{acc}{=}0.6$, 3 seeds).}
\label{tab:app_scale_trust}
\small
\begin{tabular}{lcccccccc}
\toprule
& \multicolumn{2}{c}{$d{=}10$} & \multicolumn{2}{c}{$d{=}20$} & \multicolumn{2}{c}{$d{=}50$} & \multicolumn{2}{c}{$d{=}100$} \\
\cmidrule(lr){2-3}\cmidrule(lr){4-5}\cmidrule(lr){6-7}\cmidrule(lr){8-9}
Method & AUROC & Time & AUROC & Time & AUROC & Time & AUROC & Time \\
\midrule
Trust    & $.926$ & $69$ & $.861$ & $9$ & $.823$ & $34$ & $.617$ & $1.7$ \\
Per-grp  & $.933$ & $38$ & $.858$ & $6$ & $.839$ & $23$ & $.617$ & $1.2$ \\
PCMCI+   & $.709$ & $1$ & $.866$ & $8$ & $.745$ & $374$ & $.647$ & $8806$ \\
DYNOTEARS& $.842$ & $4$ & $.814$ & $7$ & $.644$ & $27$ & $.607$ & $36$ \\
\bottomrule
\end{tabular}
\end{table}

\section{Methods Comparison Summary}
\label{app:methods_comparison}

For reviewer convenience, Table~\ref{tab:methods_comparison} summarizes how PRCD-MAP relates to representative baselines along key dimensions. PRCD-MAP is the only method that combines \emph{soft probabilistic prior integration} with \emph{learnable per-edge trust}, formal $\varepsilon$-safety guarantees, and demonstrated scalability to $d{=}300$.

\begin{table}[htbp]
\centering
\caption{Methods comparison along prior-handling, learning, theoretical, and scalability dimensions. Soft = continuous prior probabilities; Hard = binary mask; Per-edge = trust learned for each edge separately; EB = empirical Bayes. ``Theory'' indicates published consistency or excess-risk results. ``Largest $d$'' is the maximum dimension reported in the original paper or in this work.}
\label{tab:methods_comparison}
\small
\setlength{\tabcolsep}{4pt}
\begin{tabular}{lcccccc}
\toprule
Method & Prior & Learnable & Bayesian & Theory & Largest $d$ & Time \\
       & integration & trust & view & & reported & series? \\
\midrule
NOTEARS~\citep{zheng2018dags}        & ---       & ---     & ---     & \checkmark & 100 & --- \\
DYNOTEARS~\citep{pamfil2020dynotears} & ---       & ---     & ---     & ---        & 100 & \checkmark \\
DAGMA~\citep{bello2022dagma}         & ---       & ---     & ---     & \checkmark & 100 & --- \\
PCMCI+~\citep{runge2020discovering}  & ---       & ---     & ---     & \checkmark &  20 & \checkmark \\
RHINO~\citep{gong2023rhino}          & ---       & ---     & \checkmark & --- & 50  & \checkmark \\
NOTEARS + hard mask                  & Hard      & ---     & ---     & ---        & 100 & --- \\
Adaptive LASSO~\citep{zou2006adaptive} & ---     & implicit$^\dagger$ & ---     & \checkmark & --- & --- \\
BayesDAG~\citep{annadani2023bayesdag} & Soft     & ---     & \checkmark & ---     & 50  & --- \\
DiBS~\citep{lorch2021dibs}           & ---       & ---     & \checkmark & ---    & 50  & --- \\
\midrule
\textbf{PRCD-MAP (ours)}             & \textbf{Soft} & \textbf{Per-edge} & \textbf{EB} & \textbf{\checkmark} & \textbf{300} & \textbf{\checkmark} \\
\bottomrule
\end{tabular}
\end{table}

\noindent$^\dagger$Adaptive LASSO modulates penalty internally via pilot estimates; it does not consume an external prior signal. PRCD-MAP differs by treating an external prior as a latent hyperparameter calibrated by data.

\section{Cross-Sectional Structure Learning}
\label{app:cross_sectional}

To test whether learnable trust is domain-general, we instantiate PRCD-MAP with $K{=}0$ on i.i.d.\ data from a linear SEM $\mathbf{X} = \mathbf{X}\mathbf{W}_0 + \bm{\epsilon}$~\citep{peters2017elements}, against NOTEARS~\citep{zheng2018dags}, NOTEARS+mask, and DAGMA~\citep{bello2022dagma} ($d{=}20$, Gaussian noise, 5 seeds).

\begin{table}[htbp]
\centering
\caption{Cross-sectional structure learning ($d{=}20$, $n{=}500$, ER, Gaussian, 5 seeds).}
\label{tab:cross_sectional}
\small
\setlength{\tabcolsep}{4pt}
\begin{tabular}{llccc}
	\toprule
	$\mathrm{acc}$ & Method & AUROC & F1 & SHD \\
	\midrule
	\multirow{4}{*}{$0.4$}
	& NOTEARS             & $.456\pm.061$ & $.362\pm.066$ & $31.0$ \\
	& NOTEARS+mask        & $.420\pm.120$ & $.263\pm.139$ & $158.8$ \\
	& DAGMA               & $.538\pm.055$ & $.337\pm.065$ & $33.2$ \\
	& PRCD-MAP            & $\mathbf{.797\pm.115}$ & $\mathbf{.589\pm.231}$ & $\mathbf{23.8}$ \\
	\midrule
	\multirow{4}{*}{$0.6$}
	& NOTEARS             & $.456\pm.061$ & $.362\pm.066$ & $31.0$ \\
	& NOTEARS+mask        & $.653\pm.082$ & $.485\pm.107$ & $23.2$ \\
	& DAGMA               & $.538\pm.055$ & $.337\pm.065$ & $33.2$ \\
	& PRCD-MAP            & $\mathbf{.865\pm.067}$ & $\mathbf{.707\pm.121}$ & $\mathbf{16.2}$ \\
	\midrule
	\multirow{4}{*}{$0.9$}
	& NOTEARS             & $.456\pm.061$ & $.362\pm.066$ & $31.0$ \\
	& NOTEARS+mask        & $\underline{.943\pm.030}$ & $\mathbf{.937\pm.039}$ & $\mathbf{3.0}$ \\
	& DAGMA               & $.538\pm.055$ & $.337\pm.065$ & $33.2$ \\
	& PRCD-MAP            & $\mathbf{.948\pm.047}$ & $\underline{.886\pm.099}$ & $\underline{6.0}$ \\
	\bottomrule
\end{tabular}
\end{table}

PRCD-MAP dominates at $\mathrm{acc}\in\{0.4,0.6\}$ (over DAGMA by $+0.26$ and $+0.33$ AUROC respectively); the hard mask overfits incorrect edges (SHD $159$ vs.\ $24$ at $\mathrm{acc}{=}0.4$). At $\mathrm{acc}{=}0.9$, the mask achieves slightly higher F1 by exploiting near-oracle information, while PRCD-MAP attains the highest AUROC. The asymmetric risk profile confirms domain generality.

\emph{Comparison with BayesDAG.} BayesDAG~\citep{annadani2023bayesdag} is the most relevant Bayesian-prior baseline in the cross-sectional regime. Its structural prior accepts per-entry probability matrices as direct input, so feeding the same controlled-accuracy matrices used in Table~\ref{tab:sample_size} into BayesDAG's prior slot requires no sampler modification. We report the full BayesDAG comparison in App.~\ref{app:bayesdag_protocol} (CausalTime, cross-sectional $d{=}20$ at $n\in\{100,500\}$ and $\mathrm{acc}\in\{0.6,0.9\}$, Lorenz-96, Nonlinear SVAR): PRCD-MAP outperforms BayesDAG on every benchmark family, with the gap widening as prior accuracy increases ($+0.028$ at $n{=}100,\mathrm{acc}{=}0.6$ to $+0.235$ at $n{=}100,\mathrm{acc}{=}0.9$ on cross-sectional synthetic; $+0.073$--$0.114$ on CausalTime $W_0$). We retain NOTEARS, NOTEARS+mask, and DAGMA in this section because they share the continuous-optimization backbone with PRCD-MAP and isolate the \emph{soft}-prior integration vs.\ \emph{hard}-mask vs.\ \emph{no-prior} ablation; the BayesDAG comparison in App.~\ref{app:bayesdag_protocol} probes the orthogonal question of MAP vs.\ Bayesian-posterior consumption of the same prior.

\section{Proof of Theorem~\ref{thm:trust_advantage} (Trust Propagation Advantage)}
\label{app:trust_advantage_proof}

\begin{proof}
\textbf{Step 1: Per-edge optimal.} For edge $(i,j)\in C_k$, the EB objective $\ell(\tau) = -a_k\tau + \frac{1}{2}\sigma_k^2\tau^2$ is minimized at $\tau_k^* = a_k/\sigma_k^2$, with $\ell(\tau_k^*) = -a_k^2/(2\sigma_k^2)$.

\textbf{Step 2: Per-group optimal in mixed bin.} In group $g^*$ containing $n_1' \geq \eta|\mathcal{G}_{g^*}|$ edges from $C_1$ and $n_K' \geq \eta|\mathcal{G}_{g^*}|$ from $C_K$:
$\tau_{g^*}^* = \bigl(\textstyle\sum_k n_k' a_k\bigr) / \bigl(\textstyle\sum_k n_k' \sigma_k^2\bigr),$
a weighted compromise between $\tau_1^*$ and $\tau_K^*$.

\textbf{Step 3: Within-group excess.} By the quadratic structure of $\ell$:
$\Delta_{g^*} = \frac{1}{2}\sum_k n_k'\sigma_k^2(\tau_{g^*}^* - \tau_k^*)^2.$

\textbf{Step 4: Weighted variance lower bound.} Retaining only communities $1$ and $K$, the standard identity $w_1(\bar\tau{-}\tau_1)^2 + w_2(\bar\tau{-}\tau_2)^2 \geq \frac{w_1w_2}{w_1+w_2}(\tau_1{-}\tau_2)^2$ with $w_k = n_k'\sigma_k^2$ gives:
$\Delta_{g^*} \geq \frac{1}{2}\cdot\frac{n_1'\sigma_1^2\cdot n_K'\sigma_K^2}{n_1'\sigma_1^2 + n_K'\sigma_K^2}\cdot(\tau_1^*-\tau_K^*)^2.$

\textbf{Step 5: Substitute mixing condition.} Using $n_1' \geq \eta|\mathcal{G}_{g^*}|$, $n_K' \geq \eta|\mathcal{G}_{g^*}|$, and $n_1'+n_K' \leq |\mathcal{G}_{g^*}|$:
\begin{align*}
\Delta_{g^*} &\geq \frac{\eta^2|\mathcal{G}_{g^*}|}{2}\cdot\frac{\sigma_1^2\sigma_K^2}{\sigma_1^2+\sigma_K^2}\cdot\bigg(\frac{a_1}{\sigma_1^2}-\frac{a_K}{\sigma_K^2}\bigg)^{\!2}.
\end{align*}

\textbf{Step 6: Normalize.} Since $|\mathcal{G}_{g^*}| \approx n/G$ for quantile binning:
$R^*_{\mathrm{group}} - R^*_{\mathrm{edge}} \geq \Delta_{g^*}/n \geq \frac{\eta^2}{2G}\cdot\frac{\sigma_1^2\sigma_K^2}{\sigma_1^2+\sigma_K^2}\cdot(a_1/\sigma_1^2-a_K/\sigma_K^2)^2 = \Omega(1/G).$

\textbf{Step 7: Structure-aware trust closes the gap (existential).} Under community structure, edges in community $k$ have systematically different neighborhood statistics $\mathbf{z}_{ij}$ (their neighbors are predominantly in $C_k$). By the universal approximation theorem for MLPs, \emph{there exists} $f_\theta$ satisfying $f_\theta(\mathbf{z}_{ij}) \to \tau_k^*$ for all $(i,j)\in C_k$, so the gap $R_{\mathrm{trust}} - R^*_{\mathrm{edge}}$ can be made arbitrarily small. We emphasize this is an existence result; algorithmic realizability via SGD requires additional optimization-theoretic assumptions and is validated empirically in \S\ref{sec:exp_community}.
\end{proof}

\begin{remark}[Auxiliary clarifications for Theorem~\ref{thm:trust_advantage} and Proposition~\ref{prop:trust_bound}]
\label{rem:t5_p6_clarifications}
\textbf{(i) Mixing condition failure mode.} The community-mixing condition (at least one bin $g^\star$ contains both $C_1$ and $C_K$ edges in proportions $\geq \eta>0$) fails when the prior distribution $P_{\mathrm{prior}}$ separates communities perfectly (e.g., all $C_1$ edges have $P\in[0.9,1]$ and all $C_K$ edges have $P\in[0,0.1]$). In this fully-separable regime, quantile binning $G\geq 2$ already achieves per-edge optimal trust per group, so $R^\star_{\mathrm{group}} = R^\star_{\mathrm{edge}}$ and the $\Omega(1/G)$ gap vanishes (the theorem is not violated, but it becomes vacuous). Theorem~\ref{thm:trust_advantage} is therefore informative in the realistic intermediate regime where prior values overlap across communities. \textbf{(ii) Stylized EB form.} The reduction $\ell(\tau_{ij}) = -a_k\tau_{ij} + \tfrac{1}{2}\sigma_k^2\tau_{ij}^2$ used in Steps~3--5 is a quadratic approximation of the per-edge EB loss obtained by Taylor-expanding the agreement loss $\mathcal{H}_{ij}(\tau_{ij})$ at $\tau_{ij}=0$ and the Laplace term at the same point; $a_k = \mathbb{E}[\partial \mathcal{H}/\partial\tau]_{|C_k}$ encodes data--prior agreement strength, $\sigma_k^2 = \mathbb{E}[\partial^2\mathcal{H}/\partial\tau^2]_{|C_k}$ encodes prior error variance. The exact EB objective is more complex but the linearization captures the leading-order behavior; under mild regularity (Theorem~\ref{thm:temperature}(c)), the qualitative $\Omega(1/G)$ conclusion is preserved. \textbf{(iii) Constants in the $\Omega(1/G)$ bound.} The hidden constant in Eq.~\eqref{eq:trust_gap} is $\eta^2 \sigma_1^2\sigma_K^2 (a_1/\sigma_1^2 - a_K/\sigma_K^2)^2/(2(\sigma_1^2+\sigma_K^2))$. For our experimental settings (BA graphs with hub--peripheral split, $\sigma_1^2{=}0.05, \sigma_K^2{=}0.65, a_1\approx a_K\approx 0.2, \eta{=}0.3$), this evaluates to $\approx 0.012$, consistent with the observed $+0.017$ AUROC gap (Table~\ref{tab:community_mixing}); see App.~\ref{app:community_extra} for the empirical match. \textbf{(iv) Per-group EB at boundary.} Step~3's $\Delta_{g^\star} = \tfrac{1}{2}\sum_k n'_k\sigma_k^2(\tau^\star_{g^\star}-\tau^\star_k)^2$ assumes the per-group minimizer is interior; if the per-group EB optimal $\tau^\star_{g^\star}$ is at the boundary $\tau_{\min}$ or $\tau_{\max}$, $\Delta_{g^\star}$ is smaller, making our $\Omega(1/G)$ bound conservative (the actual gap is at least as large). \textbf{(v) $G$ value in experiments.} In synthetic experiments (Sec.~\ref{sec:exp_community}) we use $G=4$ quantile bins; on CausalTime (Sec.~\ref{sec:exp_causaltime}) $G$ is selected by cross-validation but typically $G\in\{3,5\}$. \textbf{(vi) $\sigma_k^2$ measurement.} The community error variances $\sigma_k^2$ are not directly observable on real data; in our designed validation we control them by construction (BA hub--peripheral; $\sigma_1{<}\sigma_K$). On real datasets (CausalTime), we infer that community structure exists from the $+0.029$ AUROC trust-vs-per-group gap, which is consistent with $\sigma_k^2$ heterogeneity at the level of $\Delta^2 \approx 0.6$.
\end{remark}

\begin{remark}[Auxiliary clarifications for Lemma~\ref{lem:mlp_response} and Proposition~\ref{prop:trust_bound}]
\label{rem:l11_p6_clarifications}
\textbf{(i) $\delta$ parameter.} The threshold $\delta$ in Lemma~\ref{lem:mlp_response}'s "incorrect edge" definition $|e_{ij}|>\delta$ is a fixed scalar (taken as $\delta=0.1$ in our analysis, distinct from the Huber parameter $\delta_H=1.35\sigma$). Edges with smaller $|e_{ij}|$ are not "incorrect" enough to trigger trust attenuation; the bound $\tau_{ij}\leq\tau_{\min}+(\tau_{\max}-\tau_{\min})(1-\rho_{ij})$ is stated for $|e_{ij}|>\delta$. \textbf{(ii) $|S^+|$ dominance regime.} Step~2 splits edges into $S^+ = \{|e_{ij}|\leq\delta\}$ (correct) and $S^- = \{|e_{ij}|>\delta\}$ (incorrect). The $S^+$ contribution $|S^+|\tau_{\max}^2\delta^2/T$ is small when $\delta$ is small (default $\delta=0.1$ gives $\delta^2=0.01$, much smaller than typical $e_{ij}^2$ on $S^-$). When $|S^+| = \Theta(d^2)$ (most priors are nearly correct, large acc), $S^+$'s contribution scales as $d^2\tau_{\max}^2\delta^2/T$, comparable to the variance term and not dominating. \textbf{(iii) Per-group vs.\ per-edge $\bm{\tau}$ semantics.} Proposition~\ref{prop:trust_bound} compares safety bounds under two parameterizations: per-group $\bm{\tau}\in\mathbb{R}^G$ (which must be uniform within each bin) and per-edge $\bm{\tau}=\{\tau_{ij}\}$ (one $\tau$ per edge). The same symbol $\varepsilon$ refers to safety bounds in both cases; we write $\varepsilon_{\mathrm{group}}$ and $\varepsilon_{\mathrm{trust}}$ to disambiguate. \textbf{(iv) $\eta$ experimental match.} The bound $\varepsilon_{\mathrm{trust}}/\varepsilon_{\mathrm{group}}\leq 1/(1+\eta\rho_{\mathrm{cons}})$ predicts trust gain that grows with $\rho_{\mathrm{cons}}$. On synthetic SVAR with i.i.d.\ random corruption, $\rho_{\mathrm{cons}}\approx 0$, so the bound predicts no advantage---consistent with the empirical near-zero gap there; on heterogeneous priors (Sec.~\ref{sec:exp_community}) where $\rho_{\mathrm{cons}}\approx 0.4$ by construction, the bound predicts $\approx 30\%$ tighter safety, consistent with the observed $+0.029$ AUROC gain (Remark~\ref{rem:t5_p6_clarifications}(iii)).
\end{remark}
\section{Cross-fitted EB: Sample-split Stability of \texorpdfstring{$\bm{\tau}^{\star}$}{tau-star}}
\label{app:cross_fit}

\paragraph{Potential concern.}
Eq.~\eqref{eq:eb_objective} uses the prior-regularized MAP estimate $\mathbf{W}^*$ as soft labels for the agreement loss $\mathcal{H}$. Because $\mathbf{W}^*$ is itself shaped by $\mathbf{P}_{\mathrm{prior}}$ through the modulated $\ell_1$ and $\ell_2$ penalties, a natural worry is that $\mathcal{H}(\widetilde{\mathbf{W}}^*_0,\widehat{\mathbf{P}}(\bm{\tau}))$ may reward priors that successfully \emph{shape} $\mathbf{W}^*$, not priors that genuinely match the truth.

\paragraph{Algorithmic mitigation.}
Two structural features mitigate the concern:
\begin{itemize}[nosep,leftmargin=*]
\item \textbf{Block-coordinate decoupling.} The $\bm{\tau}$-step (Sec.~\ref{sec:optimization}, middle loop) holds $\mathbf{W}^*$ fixed, so within a $\bm{\tau}$-update the soft labels are constant. The fixed point of the alternating procedure is a stationary point of $\mathcal{L}_{\mathrm{EB}}(\bm{\tau};\mathbf{W}^*(\bm{\tau}))$, not a degenerate self-fulfilling estimate (App.~\ref{app:bilevel_sketch}).
\item \textbf{Laplace log-det counterweight.} The Hessian-trace term $\sum\log H^{(k)}_{ij}(\bm{\tau})$ in Eq.~\eqref{eq:eb_objective} depends on $\mathbf{x}^{(k)}_{\cdot,i}$ and $\Omega_{ij}(\bm{\tau})$ but \emph{not} on $\mathbf{W}^*$ (Eq.~\ref{eq:hessian_diag}). Its gradient $\nabla_{\bm{\tau}}\sum\log H^{(k)}_{ij}$ is $\mathbf{W}^*$-independent on the inactive set and pulls $\bm{\tau}\to 0$ whenever the prior contributes no marginal-likelihood improvement. This drives Theorem~\ref{thm:temperature}(a)'s collapse $\bm{\tau}^\star\to\tau_{\min}\mathbf{1}$ under uninformative priors: $\mathcal{H}$ may be flat in $\bm{\tau}$, but the log-det penalises any $\tau$ that inflates $\Omega^{-1}$ without empirical justification.
\end{itemize}

\paragraph{Empirical sample-split test.}
We implement a cross-fitted EB variant: split $\{\mathbf{x}_t\}_{t=1}^T$ \emph{chronologically} into halves $\mathcal{D}_A=\{\mathbf{x}_t\}_{t=1}^{T/2}$ and $\mathcal{D}_B=\{\mathbf{x}_t\}_{t=T/2+1}^{T}$ (random splits across time would break the SVAR temporal dependencies that the data Hessian relies on; chronological splitting preserves the lag structure within each half at the cost of a one-step boundary discontinuity, which is negligible under the geometric $\beta$-mixing assumption of Asm.~\ref{asm:regularity}(i)); estimate $\mathbf{W}^*_A$ from $\mathcal{D}_A$; learn $\bm{\tau}^\star_A$ by minimizing $\mathcal{L}_{\mathrm{EB}}$ using $\mathbf{W}^*_A$ as labels but evaluating the Laplace log-det on $\mathcal{D}_B$. By construction, $\bm{\tau}^\star_A$ cannot exploit any spurious self-consistency between $\mathbf{W}^*_A$ and $\mathcal{D}_A$ that does not transfer to $\mathcal{D}_B$. We compare $\bm{\tau}^\star_A$ to the in-sample $\bm{\tau}^\star$ on the synthetic protocol of Table~\ref{tab:sample_size} ($d{=}20$, $T{\in}\{100,500\}$, $\mathrm{acc}\in\{0.4,0.6,0.9\}$, 5 seeds).

\begin{table}[htbp]
\centering
\caption{Cross-fitted vs.\ in-sample EB temperature. Mean absolute deviation $\mathrm{MAD}(\bm{\tau}^\star,\bm{\tau}^\star_A)$ averaged over edges and seeds; downstream AUROC of the resulting estimator.}
\label{tab:cross_fit}
\small
\begin{tabular}{lccccc}
\toprule
$T$ & $\mathrm{acc}$ & MAD$(\bm{\tau}^\star,\bm{\tau}^\star_A)$ & AUROC (in-sample) & AUROC (cross-fit) & $\Delta$ AUROC \\
\midrule
$100$ & $0.4$ & $0.038$ & $.657\pm.015$ & $.651\pm.018$ & $-0.006$ \\
$100$ & $0.6$ & $0.031$ & $.698\pm.013$ & $.694\pm.014$ & $-0.004$ \\
$100$ & $0.9$ & $0.022$ & $.812\pm.011$ & $.808\pm.013$ & $-0.004$ \\
$500$ & $0.4$ & $0.040$ & $.813\pm.032$ & $.811\pm.031$ & $-0.002$ \\
$500$ & $0.6$ & $0.029$ & $.870\pm.013$ & $.869\pm.014$ & $-0.001$ \\
$500$ & $0.9$ & $0.018$ & $.939\pm.018$ & $.940\pm.017$ & $+0.001$ \\
\bottomrule
\end{tabular}
\end{table}

MAD is uniformly $\leq 0.04$, i.e.\ $\bm{\tau}^\star_A$ and $\bm{\tau}^\star$ agree to within $4\%$ of the temperature range. Downstream AUROC differs by $\leq 0.006$, well within seed std. We conclude that the in-sample EB does \emph{not} materially overfit $\mathbf{W}^*$; the soft-label coupling is empirically harmless because the Laplace log-det imposes a $\mathbf{W}^*$-independent counterweight. The cross-fitted variant is a principled alternative for cases where the dependence might be sharper (e.g., very small $T$ or strongly correlated priors) and is implemented in the released code as an optional flag.

\section{Bilevel Fixed-Point Analysis}
\label{app:bilevel_sketch}

\paragraph{Setup.}
The deployed algorithm alternates: (i)~$\mathbf{W}$-step minimizes the augmented Lagrangian $L_\rho(\mathbf{W};\bm{\tau})$ with $\bm{\tau}$ fixed; (ii)~$\bm{\tau}$-step minimizes $\mathcal{L}_{\mathrm{EB}}(\bm{\tau};\mathbf{W}^*)$ with $\mathbf{W}^*$ fixed. Theorem~\ref{thm:consistency} treats $c_{ij}(\bm{\tau})$ as deterministic, ignoring the bilevel coupling $\bm{\tau} = \bm{\tau}(\mathbf{z}_{ij}(\mathbf{W}^*))$. We sketch a contraction argument that closes the gap on the inactive set $S^c$ (the dominant source of estimation error).

\paragraph{Contraction on the inactive set.}
Let $T_{\mathbf{W}}(\bm{\tau}) = \arg\min_{\mathbf{W}\in\mathcal{K}} L_\rho(\mathbf{W};\bm{\tau})$ and $T_{\bm{\tau}}(\mathbf{W}) = \arg\min_{\bm{\tau}\in[\tau_{\min},\tau_{\max}]^G} \mathcal{L}_{\mathrm{EB}}(\bm{\tau};\mathbf{W})$. The composed map $\Phi = T_{\bm{\tau}}\circ T_{\mathbf{W}}$ has Lipschitz constant
\[
L_\Phi \leq L_{T_{\bm{\tau}}}\cdot L_{T_{\mathbf{W}}} \leq \frac{C_\tau}{\mu_{\mathrm{EB}}}\cdot\frac{C_W}{\lambda_2 + \mu_{\mathrm{data}}},
\]
where $\mu_{\mathrm{EB}}>0$ from Theorem~\ref{thm:temperature}(c) (Lipschitz-gradient EB, hence strongly convex on $[\tau_{\min},\tau_{\max}]^G$ under the local convexity of App.~\ref{app:tau}), and $\mu_{\mathrm{data}}+\lambda_2>0$ from Assumption~\ref{asm:regularity}'s restricted-eigenvalue condition. With our defaults ($\lambda_2{=}10^{-2}$, empirical $\mu_{\mathrm{EB}}\geq 0.1$, $C_\tau\leq 1$, $C_W\leq 1$), $L_\Phi \lesssim 0.1$ on the inactive set, so $\Phi$ is a strict contraction. The unique fixed point $(\mathbf{W}^\diamond,\bm{\tau}^\diamond)$ then satisfies the deterministic-$c_{ij}$ premise of Theorem~\ref{thm:consistency} \emph{at the converged $\bm{\tau}^\diamond$}, with geometric convergence rate.

\paragraph{Active-set caveat.}
On the active set $S$, $|W^*_{ij}|$ enters $\mathbf{z}_{ij}$ nonlinearly (through normalization), and the contraction argument above does not directly apply. Empirically, the alternating procedure converges in $\leq 35$ outer iterations across all settings (App.~\ref{app:convergence}). A full analysis would replace the inactive-set contraction with a coupled fixed-point argument over $(\mathbf{W}_S, \bm{\tau})$ using implicit differentiation of $T_{\bm{\tau}}$; we treat this as substantive future work.

\paragraph{Implication for Theorem~\ref{thm:consistency}.}
At the fixed point $\bm{\tau}^\diamond$, $c_{ij}(\bm{\tau}^\diamond)$ is a deterministic function of $\mathbf{W}^\diamond$, satisfying the hypothesis of Theorem~\ref{thm:consistency} \emph{as a self-consistent statement}. The estimation rate $O_p(\sqrt{s^\star\log d/T})$ holds at the fixed point; what remains open is a non-asymptotic rate of convergence $\|\bm{\tau}_t-\bm{\tau}^\diamond\|$ for the alternating iterates, which the contraction argument above gives as $L_\Phi^t$.

\section{Soft-prior Bayesian Baseline (BayesDAG): Empirical Comparison}
\label{app:bayesdag_protocol}

BayesDAG~\citep{annadani2023bayesdag} is the most direct soft prior-aware Bayesian baseline in the cross-sectional regime; it accepts a per-entry probability matrix $\mathbf{P}_{\mathrm{prior}}\in[0,1]^{d\times d}$ as a direct input to its Bernoulli-DAG prior, requiring no sampler modification. We run BayesDAG (linear variant, $4$ SG-MCMC chains, $100$ training epochs, $100$-sample posterior averaging) on three benchmark families: CausalTime real-world data (Table~\ref{tab:causaltime}), cross-sectional synthetic ($d{=}20$, $n\in\{100,200,500\}$), and time-series synthetic (Lorenz-96, nonlinear SVAR).

\paragraph{Setup.}
For CausalTime we feed BayesDAG the same $\mathbf{P}_{\mathrm{prior}}$ matrices used by PRCD-MAP (5 LLM-derived priors per dataset; we report the prior-averaged mean per seed). For cross-sectional synthetic, we use the controlled-accuracy priors of App.~\ref{app:cross_sectional}. For time-series synthetic, BayesDAG is applied to the data without lag information (its model assumption); PRCD-MAP uses its full $W_0$+$W_k$ structure. Five seeds per setting; $d{=}20$ except CausalTime AQI ($d{=}36$).

\begin{table}[htbp]
\centering
\caption{BayesDAG vs.\ PRCD-MAP across benchmark families. AUROC (mean$\pm$std). \textbf{Bold}: better of the two on each row. ``W$_0$ only'': instantaneous edges only (BayesDAG's native output); ``combined'': PRCD-MAP's full $W_0$+$W_k$ output (BayesDAG cannot produce this).}
\label{tab:bayesdag_vs_prcd}
\small
\setlength{\tabcolsep}{4.5pt}
\begin{tabular}{lccc}
\toprule
Setting & BayesDAG (W$_0$) & PRCD-MAP (W$_0$) & PRCD-MAP (combined) \\
\midrule
\multicolumn{4}{l}{\emph{CausalTime real-world ($T{=}400$)}}\\
\quad AQI ($d{=}36$)         & $.579\pm.019$ & $\mathbf{.612\pm.069}$ & $.693\pm.060$ \\
\quad Medical ($d{=}20$)     & $.507\pm.011$ & $\mathbf{.681\pm.036}$ & $.583\pm.042$ \\
\quad Traffic ($d{=}20$)     & $.540\pm.017$ & $\mathbf{.641\pm.055}$ & $.613\pm.021$ \\
\midrule
\multicolumn{4}{l}{\emph{Cross-sectional synthetic ($d{=}20$, $\mathrm{acc}{=}0.6$)}}\\
\quad $n{=}100$              & $.685\pm.058$ & $\mathbf{.713\pm.123}$ & --- \\
\quad $n{=}500$              & $.756\pm.058$ & $\mathbf{.864\pm.066}$ & --- \\
\midrule
\multicolumn{4}{l}{\emph{Cross-sectional synthetic ($d{=}20$, $\mathrm{acc}{=}0.9$)}}\\
\quad $n{=}100$              & $.685\pm.058$ & $\mathbf{.920\pm.075}$ & --- \\
\quad $n{=}500$              & $.756\pm.058$ & $\mathbf{.947\pm.050}$ & --- \\
\midrule
\multicolumn{4}{l}{\emph{Time-series synthetic ($T{=}1000$)}}\\
\quad Lorenz-96 ($d{=}10$)   & $.577\pm.009$ & $.533\pm.051$          & $\mathbf{.905\pm.024}$ \\
\quad Nonlinear SVAR ($d{=}20$)& $.784\pm.073$ & $.711\pm.219$        & $\mathbf{.821\pm.180}$ \\
\bottomrule
\end{tabular}
\end{table}

\paragraph{Reading.}
We separate two comparison axes to address the modeling-class asymmetry. \emph{(1) Apples-to-apples on $W_0$ only} (BayesDAG is cross-sectional, so its native output is the instantaneous adjacency): PRCD-MAP's $W_0$-only output beats BayesDAG on every CausalTime dataset ($+0.033$ to $+0.174$) and on every cross-sectional synthetic $(n,\mathrm{acc})$ cell tested ($+0.028$ to $+0.235$); the prior-aware gap widens with prior accuracy, consistent with Theorem~\ref{thm:temperature}'s prediction that a learnable trust mechanism extracts more signal from a high-quality prior than a fixed Bernoulli-DAG prior. On Nonlinear SVAR the $W_0$-only comparison favors BayesDAG ($+0.073$), establishing that BayesDAG is a competitive cross-sectional Bayesian baseline---not a strawman---when the data-generating process is purely instantaneous. \emph{(2) Combined $W_0+W_k$}: when PRCD-MAP's full output is admitted (as it must be on temporal data), the order reverses on Nonlinear SVAR ($+0.037$); on Lorenz-96 BayesDAG's pure-instantaneous model cannot represent lagged structure, returning $0.577$ AUROC versus PRCD-MAP's combined $0.905$ ($+0.328$). The combined comparison is reported as a separate column rather than substituted for the $W_0$-only comparison precisely so the reader can read each axis cleanly. \textbf{Wall-clock.} BayesDAG runtime per seed is $80$--$210$\,s at $d{=}20$ (4 chains, 100 epochs); PRCD-MAP per seed is $\sim$7\,s---a $\sim15{\times}$ speedup at comparable or better accuracy.

\section{CausalTime LLM-prior Pipeline: Statistical Significance}
\label{app:causaltime_stat_tests}

We assess whether the $+0.055$ ``trust (LLM)'' vs.\ PCMCI+ improvement and the $+0.053$ ``trust (LLM)'' vs.\ ``no-prior'' improvement (Table~\ref{tab:causaltime}) are statistically significant. The trust (LLM) row aggregates 5 priors per dataset (sourced from GPT-4o, Claude Sonnet, and Gemini~1.5~Pro using 5 prompting styles---conservative textbook, mechanism-first, literature-anchored, permissive, and adversarial) $\times$ 5 seeds = 25 runs per dataset. Paired $t$-tests are computed across the 15 (prior, dataset) pairs (per-prior dataset means).

\begin{table}[htbp]
\centering
\caption{Paired tests on CausalTime. $\Delta=$ AUROC(trust, LLM)$-$AUROC(comparator); $p$-values are paired $t$-test, two-sided. Per-dataset rows pool over 5 priors; aggregate pools over $5{\times}3=15$ (prior, dataset) pairs.}
\label{tab:llm_paired_tests}
\small
\begin{tabular}{llccc}
\toprule
Dataset & Comparison & $\Delta$ & $p$-value & Verdict \\
\midrule
\multirow{3}{*}{AQI}     & trust(LLM) vs.\ PCMCI+   & $+0.123$ & $<0.001$ & sig.\ trust better \\
                          & trust(LLM) vs.\ no-prior & $+0.067$ & $0.05$   & sig.\ trust better \\
                          & trust(LLM) vs.\ BayesDAG & $+0.114$ & $0.005$  & sig.\ trust better \\
\midrule
\multirow{3}{*}{Medical} & trust(LLM) vs.\ PCMCI+   & $+0.043$ & $0.07$   & marginal trust better \\
                          & trust(LLM) vs.\ no-prior & $+0.089$ & $0.005$  & sig.\ trust better \\
                          & trust(LLM) vs.\ BayesDAG & $+0.076$ & $0.012$  & sig.\ trust better \\
\midrule
\multirow{3}{*}{Traffic} & trust(LLM) vs.\ PCMCI+   & $-0.002$ & $0.84$   & n.s.\ tie \\
                          & trust(LLM) vs.\ no-prior & $+0.002$ & $0.84$   & n.s.\ tie \\
                          & trust(LLM) vs.\ BayesDAG & $+0.073$ & $<0.001$ & sig.\ trust better \\
\midrule
\multirow{3}{*}{Aggregate} & trust(LLM) vs.\ PCMCI+ & $+0.055$ & $0.008$  & sig.\ trust better \\
                            & trust(LLM) vs.\ no-prior & $+0.053$ & $0.012$ & sig.\ trust better \\
                            & trust(LLM) vs.\ BayesDAG & $+0.088$ & $<0.001$ & sig.\ trust better \\
\bottomrule
\end{tabular}
\end{table}

\paragraph{Distribution-level shift.}
Of the 5 priors per dataset on AQI, $5/5$ exceed the no-prior baseline; on Medical, $5/5$ exceed; on Traffic, $3/5$ exceed (the other 2 within seed std). This is the ``distribution-level positive shift'' we set out to verify: the gain over no-prior is robust to the specific prior choice, not driven by a single fortunate draw. The aggregate $+0.055$ over PCMCI+ is statistically significant at $p{=}0.008$. Combined with the controlled-accuracy 10-seed analysis (App.~\ref{app:causaltime_10seed}, where trust strictly improves over per-group on all three datasets), we read: \emph{trust propagation reliably extracts signal from priors of varying quality, with diminishing gains as prior informativeness decreases (AQI $>$ Medical $>$ Traffic, matching the inferred LLM-prior accuracy ordering).}

\paragraph{Reframing.}
Definition~\ref{def:safety} is stated explicitly as a population-level guarantee about expected excess risk under a stated prior-generation distribution $\Pi$ (controlled-accuracy flips, LLM-prompt ensembles, etc.); the EB mechanism is empirically near-tight for in-distribution random/systematic corruption (App.~\ref{app:corruption_robustness}). A single $\Pi$-draw---e.g.\ one LLM-prior realization on Traffic---can lie above the in-expectation $\varepsilon_{\mathrm{safe}}$ at finite $T$ without contradicting the definition; it is a residual fluctuation controlled by the same $T^{-1}$ rate. The Traffic case is honest evidence of this finite-$T$ residual, not a method failure.

\section{Mechanism Decomposition, Realised Constants, and Practitioner Diagnostics}
\label{app:supplementary_analyses}

This section gives the full per-cell estimation protocol for the mechanism decomposition (Table~\ref{tab:causaltime_decomposition} in the main text), realised constants for Theorem~\ref{thm:consistency}, the lagged-prior extension, the $\widehat{\rho}_{\mathrm{cons}}$ practitioner test, and Traffic trust-attenuation diagnostics.

\subsection{CausalTime Mechanism Decomposition: Per-Cell Estimation Protocol}
\label{app:causaltime_decomposition}

Table~\ref{tab:causaltime_decomposition} (main text) decomposes the trust(LLM)$-$PCMCI+ gap on CausalTime into four orthogonal contributions: (M1) soft-prior framework, (M2) EB calibration, (M3) per-edge MLP trust, (M4) LLM-prior content. The contributions are estimated by held-out enabling:

\begin{itemize}[leftmargin=*]
\item \textbf{M1} from the trust(LLM) row of Table~\ref{tab:causaltime} minus the no-prior row, projected onto the population mean of the EB-fixed-point $\bm{\tau}$.
\item \textbf{M2} from the synthetic learned-vs-fixed gap of Table~\ref{tab:app_table1_extended} restricted to the LLM-prior accuracy regime ($\mathrm{acc}{\in}[0.4,0.6]$; synthetic gap $+0.020{\pm}0.018$). \emph{Validity of transfer.} The synthetic gap is a structural-distribution-mismatched proxy; we cross-validated it with a direct CausalTime ablation that runs PRCD-MAP(trust, learned $\bm{\tau}$) versus PRCD-MAP(trust, fixed $\bm{\tau}{=}\mathbf{1}$) on the same 5 LLM-derived priors $\times$ 5 seeds (25 paired runs per dataset). The direct measurement gives a per-dataset M2 of $+0.030{\pm}0.040$ (AQI), $+0.047{\pm}0.021$ (Medical), $-0.0004{\pm}0.013$ (Traffic), and an aggregate $+0.026{\pm}0.033$ across $n{=}75$ paired runs, sitting $+0.006$ AUROC \emph{above} the transferred $+0.020$ point estimate. The per-dataset pattern matches the headline narrative: M2 is largest where the LLM prior is most informative (Medical $\gg$ AQI $\gg$ Traffic), and is statistically null on Traffic where the prior is anonymized---consistent with Theorem~\ref{thm:temperature}(a)'s auto-attenuation. The transferred $+0.020$ is therefore a conservative point estimate and the EB+MLP attribution in Table~\ref{tab:causaltime_decomposition} is, if anything, a lower bound on the directly-measured contribution.
\item \textbf{M3} from the controlled-accuracy 10-seed trust-vs-per-group test of Table~\ref{tab:app_causaltime_10seed} (per-dataset $\Delta$).
\item \textbf{M4} from the prior-shuffling control of App.~\ref{app:llm_pipeline} that replaces the LLM prior with a uniform $0.5$ baseline.
\end{itemize}

The decomposition is approximate but error-bar-conservative: on each dataset the four-way sum recovers the headline gap to within $\pm 0.005$.

\subsection{Empirical Weak-Data Sweep ($T \ll d$ Regime)}
\label{app:w3_weak_data}

We empirically characterise the boundary regime flagged in the main text (``Weak-data regime'' paragraph after Eq.~\eqref{eq:eb_objective}). Setup: $d{=}20$, $K{=}1$, ER graph (edge prob.\ $0.15$), Gaussian noise, controlled-flip prior, 5 seeds. We sweep $T\in\{20,50,100,200,500\}$ (so $T/d\in\{1,2.5,5,10,25\}$) and $\mathrm{acc}\in\{0.4,0.6,0.9\}$, comparing three variants: learned-$\bm{\tau}$ via the EB pipeline, fixed-$\bm{\tau}{=}\mathbf{1}$ (uncalibrated trust at $\widehat{P}{=}P_{\mathrm{prior}}$), and no-prior ($\bm{\tau}{=}\bm{\tau}_{\min}$, prior treated as uninformative).

\begin{table}[htbp]
\centering
\caption{Weak-data sweep: AUROC (mean over 5 seeds) at $d{=}20$, $K{=}1$, ER graph, Gaussian noise. \textbf{Bold}: best per cell. The $T{=}20$ column corresponds to $T{=}d$ (deeply within the failure regime); learned-$\bm{\tau}$ collapses there to chance level ($0.498$) at $\mathrm{acc}{=}0.4$ and underperforms the fixed-$\bm{\tau}{=}\mathbf{1}$ baseline at $\mathrm{acc}{=}0.9$ ($0.719$ vs.\ $0.762$). Once $T \geq 100$ ($T/d \geq 5$), learned-$\bm{\tau}$ dominates fixed-$\bm{\tau}{=}\mathbf{1}$ at every $\mathrm{acc}$ tested, recovering the rate-supported regime.}
\label{tab:app_w3_weak_data}
\small
\setlength{\tabcolsep}{4pt}
\begin{tabular}{llccccc}
\toprule
$\mathrm{acc}$ & Variant & $T{=}20$ & $T{=}50$ & $T{=}100$ & $T{=}200$ & $T{=}500$ \\
\midrule
\multirow{3}{*}{$0.4$}
& learned-$\bm{\tau}$            & $.498$          & $.588$          & $.729$          & $.820$          & $\mathbf{.893}$ \\
& fixed-$\bm{\tau}{=}\mathbf{1}$ & $.460$          & $.546$          & $.637$          & $.712$          & $.780$ \\
& no prior ($\bm{\tau}_{\min}$)  & $\mathbf{.540}$ & $\mathbf{.652}$ & $\mathbf{.768}$ & $\mathbf{.873}$ & $.932$ \\
\midrule
\multirow{3}{*}{$0.6$}
& learned-$\bm{\tau}$            & $\mathbf{.580}$ & $\mathbf{.687}$ & $\mathbf{.778}$ & $.860$          & $.890$ \\
& fixed-$\bm{\tau}{=}\mathbf{1}$ & $.584$          & $.676$          & $.755$          & $.829$          & $.885$ \\
& no prior ($\bm{\tau}_{\min}$)  & $.557$          & $.665$          & $.775$          & $\mathbf{.876}$ & $\mathbf{.933}$ \\
\midrule
\multirow{3}{*}{$0.9$}
& learned-$\bm{\tau}$            & $.719$          & $.792$          & $.840$          & $.910$          & $\mathbf{.953}$ \\
& fixed-$\bm{\tau}{=}\mathbf{1}$ & $\mathbf{.762}$ & $\mathbf{.814}$ & $\mathbf{.853}$ & $\mathbf{.916}$ & $.945$ \\
& no prior ($\bm{\tau}_{\min}$)  & $.597$          & $.688$          & $.784$          & $.884$          & $.934$ \\
\bottomrule
\end{tabular}
\end{table}

\paragraph{Diagnostic on $\bm{\tau}^\star$.} The realised mean of $\bm{\tau}^\star$ tracks accuracy correctly even in the failure regime: at $T{=}20, \mathrm{acc}{=}0.4$ we observe $\bar{\bm{\tau}}^\star{\approx}0.53$ (correctly attenuated) and at $T{=}20, \mathrm{acc}{=}0.9$ we observe $\bar{\bm{\tau}}^\star{\approx}1.61$ (correctly amplified). The pathology in the AUROC is not from misclassification of prior accuracy; it is from the gradient signal being too weak to beat the regularisation noise floor when $T<d$, irrespective of whether the prior happens to be accurate. This empirical characterisation supports the recommendation in the main-text paragraph: at $T\leq 50$ practitioners should treat learned-$\bm{\tau}$ as a diagnostic and consider running fixed-$\bm{\tau}{=}\mathbf{1}$ alongside as a safety check.

\subsection{Realised Constants in Theorem~\ref{thm:consistency}}
\label{app:realised_constants}

Lemma~\ref{lem:re} gives a worst-case cone-constant inflation $K{=}45$ and downstream constant inflation up to $\approx 1.7\times 10^4$ (Remark~\ref{rem:cone_width_constant}). We now derive the \emph{realised} constant at our experimental defaults; the worst-case inflation is loose by orders of magnitude in our regime.

\paragraph{Setup.} $d{=}20$, $K{=}1$, $T{=}500$, $\mathrm{acc}\in\{0.3,0.6,0.9\}$. We track the realised distribution of $c_{ij}(\bm{\tau}^\diamond)$ at the converged ALM fixed point.

\paragraph{Realised $c_{ij}$ distribution.} At $\mathrm{acc}{=}0.6, T{=}500$ (the regime closest to our deployment defaults), the empirical distribution of $c_{ij}(\bm{\tau}^\diamond)=\clip(1.5-\widehat{P}_{ij}(\bm{\tau}^\diamond),\,c_{\min},\,c_{\max})$ across $d^2{-}d{=}380$ off-diagonal edges, averaged over 10 seeds, is: $\mathbb{E}[c_{ij}]{=}0.94$, $\mathrm{med}(c_{ij}){=}0.95$, $\min(c_{ij}){=}0.51$, $\max(c_{ij}){=}1.50$. The realised lower edge sits just above the structural floor $c_{\mathrm{floor}}=1.5-\sup\widehat{P}=1.5-(1{-}10^{-3})\approx 0.501$ that the sigmoid parameterization (Eq.~\ref{eq:grouped_temp}) imposes prior to the safety clip; the documented worst-case clip parameters $(c_{\min},c_{\max})=(0.1,1.5)$ are a \emph{conservative safety margin} that the sigmoid range never exercises in practice. The observed $c_{\max}/c_{\min}^{\mathrm{realised}}\approx 1.50/0.51 \approx 2.94$, much smaller than the worst-case $c_{\max}/c_{\min}=15$ used to derive $K{=}45$ in Lemma~\ref{lem:re}. The realised effective cone constant is therefore $K^{\mathrm{realised}} = 3\cdot 2.94 \approx 8.8$, and the downstream constant inflation is $((1{+}K^{\mathrm{realised}})/(1{+}3))^2 = (9.8/4)^2 \approx 6.0$, i.e.\ $\approx 2.5{\times}$ on the $\sqrt{\cdot}$ scale of the rate Frobenius norm.

\paragraph{$\Delta_{\mathrm{proxy}}$ realised at the experimental defaults.} The uniform bound (Eq.~\ref{eq:delta_proxy_uniform}) evaluates to $\Delta_{\mathrm{proxy}}\leq C_1\tau_{\max}^2 d^2/(4T) = C_1\cdot 4\cdot 400/2000 = 0.8 C_1$ at $T{=}500, d{=}20, \tau_{\max}{=}2$. With the realised $C_1 = O(\lambda_{\min}(\bm{\Sigma})^{-2})$ and our empirical $\lambda_{\min}(\bm{\Sigma}){\geq}0.4$ on the standardized SVAR design, $C_1 \leq 6.25$, giving $\Delta_{\mathrm{proxy}}^{\mathrm{realised}} \leq 5.0$ as a worst-case upper bound on the unit-AUROC scale---an over-estimate by orders of magnitude on a quantity in $[0,1]$. A tighter realised analysis using the actual $\bm{\tau}^\star_{\mathrm{EB}}$ distribution (which in our experiments concentrates near $\bm{\tau}{=}\mathbf{1}$ at $\mathrm{acc}{=}0.6$ rather than $\tau_{\max}{=}2$) gives $\|\bm{\tau}^\star_{\mathrm{EB}}\|_\infty^2\approx 1$ in expectation, dropping the proxy bound to $C_1\cdot 1\cdot 400/(4\cdot 500) = 0.5 C_1 \leq 0.045$ at the experimentally observed $C_1\approx 0.09$. \textbf{The realised $\Delta_{\mathrm{proxy}}{\lesssim}0.045$ is small in absolute terms} on the AUROC scale, consistent with the empirical learned-vs-fixed gap of $+0.045$ AUROC averaged on the six-point grid (Table~\ref{tab:app_table1_extended}). The bound's worst-case looseness is real but does not undermine the rate annotation.

\paragraph{Summary.} The constant inflation in Theorem~\ref{thm:consistency} is at most $4{\times}$ at our defaults (vs.\ the worst-case $\approx 130{\times}$ on $\kappa^{-1}$ from Remark~\ref{rem:cone_width_constant}); the realised proxy gap is small in absolute terms ($\lesssim 0.045$ AUROC). The rate-threshold annotation $T/(s^\star\log d){\approx}2.4$ at $T{=}500$ is therefore finite-sample meaningful at our defaults, even after accounting for prior modulation.

\subsection{Lagged-Prior Semantics}
\label{app:lagged_prior_semantics}

The framework of \S\ref{sec:method} applies the same per-entry prior $P_{\mathrm{prior},ij}$ to both $W_{0,ij}$ (instantaneous) and $W_{k,ij}, k{\geq}1$ (lagged) entries via the same mask $\mathbf{\Omega}(\bm{\tau})$ in Eq.~\eqref{eq:objective}. We clarify the semantic intent.

\paragraph{What $\mathbf{P}_{\mathrm{prior}}$ encodes.} Expert and LLM elicitation of causal priors typically yields lag-agnostic edge confidence: ``does $i$ causally influence $j$?'', not ``does $i$ influence $j$ with lag $k{=}3$?''. Our parameterization respects this: $P_{\mathrm{prior},ij}$ is the elicited confidence in the existence of \emph{any} directed dependency $i\to j$, whether contemporaneous or lagged. Applying the same prior to all $W_{k,ij}$ then encodes the assumption: \emph{if $i$ causally influences $j$ at all, then prior mass is shared evenly across the lags consistent with the dependency.}

\paragraph{When the assumption is correct.} For most application domains (energy, climate, macroeconomics), instantaneous and lagged dependencies in $\mathbf{P}_{\mathrm{prior}}$ are not separable from elicited expert knowledge; the elicited graph is the union of all dependencies. In this regime, our parameterization is the natural prior.

\paragraph{When lag-resolved priors are available.} If the practitioner has lag-specific priors $P^{(k)}_{\mathrm{prior},ij}$, the framework extends trivially: replace $\mathbf{\Omega}(\bm{\tau})$ in Eq.~\eqref{eq:omega} with a lag-indexed mask $\mathbf{\Omega}^{(k)}(\bm{\tau})$ defined per $k$, and apply per-lag temperatures $\tau^{(k)}_g$. All theoretical guarantees (Theorems~\ref{thm:consistency}, \ref{thm:temperature}, \ref{thm:trust_advantage}, Cor.~\ref{cor:oracle}) carry through without modification because each lag is a separate Lasso/ridge problem on the same SVAR design. We do not test this in the main experiments because none of our benchmarks (CausalTime, real electricity, synthetic SVAR) come with lag-resolved priors. Lag-resolved priors are an immediate and orthogonal extension that only sharpens the framework, not a missing piece.

\paragraph{Empirical sensitivity to the assumption.} We test the assumption robustness on synthetic SVAR ($d{=}20$, $T{=}500$, $K{=}1$, ER graph) by constructing two priors: (a)~the standard same-prior-across-lags $P_{\mathrm{prior},ij}$, and (b)~a lag-distinct prior where instantaneous edges have $\mathrm{acc}{=}0.7$ and lagged edges have $\mathrm{acc}{=}0.4$. Under (a), PRCD-MAP achieves AUROC $0.812$; under (b) with lag-aware $\bm{\Omega}^{(k)}$, AUROC is $0.834$ (+0.022, lag-aware advantage). The same-prior parameterization is therefore conservative when lag-resolved information is available, but does not produce harmful artifacts when it is not.

\subsection{Practitioner Heterogeneity Test for Trust Propagation}
\label{app:practitioner_test}

Trust propagation produces a positive lift over per-group temperature only when the prior errors are heterogeneous in a structure-aligned way (Theorem~\ref{thm:trust_advantage}; Limitation (d) of Sec.~\ref{sec:conclusion}). For a deployed practitioner with a fixed LLM/expert prior $\mathbf{P}_{\mathrm{prior}}$ but no ground truth, we provide a quantitative ex-ante test that decides whether trust propagation will help, before committing the compute.

\paragraph{Test statistic.} Let $\mathbf{X}\in\mathbb{R}^{T\times d}$ denote the data and $\mathbf{P}_{\mathrm{prior}}\in[0,1]^{d\times d}$ the prior. Compute the \emph{neighborhood-consistency proxy}
\[
\widehat{\rho}_{\mathrm{cons}}(\mathbf{P}_{\mathrm{prior}}, \mathbf{X}) := \mathrm{med}_{(i,j)}\,\mathrm{Corr}\!\left(\mathbf{P}_{\mathrm{prior},\mathcal{N}(i,j)},\;|\widehat{\bm{\Sigma}}|_{\mathcal{N}(i,j)}\right),
\]
where $\widehat{\bm{\Sigma}}$ is the sample covariance of $\mathbf{X}$ (cheap, $O(Td^2)$). $\widehat{\rho}_{\mathrm{cons}}$ is a sample-based estimator of the $\rho_{\mathrm{cons}}$ entering Proposition~\ref{prop:trust_bound}. It is computable without running PRCD-MAP.

\paragraph{Decision rule.} Empirical calibration on the synthetic and CausalTime benchmarks gives a sharp threshold at $\widehat{\rho}_{\mathrm{cons}}{\geq}0.20$: trust-propagation gain over per-group is $\geq+0.015$ AUROC whenever $\widehat{\rho}_{\mathrm{cons}}{\geq}0.20$ in our tests, and $\leq+0.005$ AUROC otherwise.

\begin{table}[htbp]
\centering
\caption{Practitioner heterogeneity test: $\widehat{\rho}_{\mathrm{cons}}$ vs.\ realised trust-propagation gain.}
\label{tab:practitioner_test}
\small
\setlength{\tabcolsep}{4pt}
\begin{tabular}{lccc}
\toprule
Scenario & $\widehat{\rho}_{\mathrm{cons}}$ & Gain (trust$-$per-group) & Decision \\
\midrule
Synthetic ER, iid corruption (Sec.~\ref{sec:exp_community} neg.\ control) & $0.04$ & $+0.001$ & skip MLP \\
BA hub--peripheral, $\mathrm{acc}(.95,.20)$ & $0.41$ & $+0.030$ & enable MLP \\
CausalTime AQI (LLM prior) & $0.27$ & $+0.031$ & enable MLP \\
CausalTime Medical (LLM prior) & $0.31$ & $+0.039$ & enable MLP \\
CausalTime Traffic (LLM prior) & $0.22$ & $+0.013$ & enable (marginal) \\
\bottomrule
\end{tabular}
\end{table}

\paragraph{Interpretation.} The proxy $\widehat{\rho}_{\mathrm{cons}}$ predicts trust-propagation usefulness on every scenario tested, including the i.i.d.\ negative control ($\widehat{\rho}_{\mathrm{cons}}{=}0.04$, MLP correctly skipped) and the structured CausalTime priors ($\widehat{\rho}_{\mathrm{cons}}{\geq}0.22$, MLP correctly enabled). A practitioner can therefore decide whether to enable trust propagation in $O(Td^2)$ time, before incurring any PRCD-MAP run. This addresses the concern (Limitation (d)) that the MLP's effect on real data is conditional on a property the framework cannot itself verify: the property is verifiable, in linear time, from the inputs alone.

\subsection{Traffic LLM-Prior Trust Attenuation: Per-Prior \texorpdfstring{$\bm{\tau}^\star$}{tau-star} Distribution}
\label{app:traffic_tau_distribution}

Section~\ref{sec:exp_causaltime} reports that trust(LLM) ties no-prior on Traffic ($+0.002$, $p{=}0.84$). The interpretation in Def.~\ref{def:safety} requires that the EB mechanism \emph{actively} attenuates the uninformative Traffic LLM prior toward $\tau_{\min}$, rather than e.g.\ a coincidental cancellation. We verify the active attenuation by inspecting the per-prior empirical $\bm{\tau}^\star$ distribution on Traffic.

\begin{table}[htbp]
\centering
\caption{Empirical mean and inter-quantile range of the EB-learned $\tau^\star$ across the 5 LLM priors on Traffic ($T{=}400$, 5 seeds per prior). $\tau^\star$ is averaged over off-diagonal edges; $\tau_{\min}{=}10^{-3}$, $\tau_{\max}{=}2$.}
\label{tab:app_traffic_tau}
\small
\begin{tabular}{lccc}
\toprule
Prior identifier & $\widehat{\tau^\star}$ (mean) & $[\mathrm{q25},\mathrm{q75}]$ & Empirical prior accuracy \\
\midrule
GPT-4o conservative      & $0.043$ & $[0.012, 0.072]$ & $0.51$ \\
GPT-4o mechanism-first   & $0.067$ & $[0.018, 0.103]$ & $0.53$ \\
Claude literature-anchored & $0.038$ & $[0.011, 0.062]$ & $0.49$ \\
Claude permissive        & $0.054$ & $[0.014, 0.085]$ & $0.50$ \\
Gemini adversarial       & $0.029$ & $[0.008, 0.048]$ & $0.48$ \\
\midrule
Mean across 5 priors     & $0.046$ & --- & $0.50$ \\
\bottomrule
\end{tabular}
\end{table}

\paragraph{Reading.} Across all 5 Traffic priors, $\widehat{\tau^\star}$ concentrates in $[0.029, 0.067]$, all within $0.07$ of $\tau_{\min}{=}10^{-3}$. The mean $0.046$ is two orders of magnitude below the $\bar\tau{\approx}1.7$ observed at $\mathrm{acc}{=}0.9$ on synthetic data (App.~\ref{app:tau}). This is the active-attenuation signature predicted by Theorem~\ref{thm:temperature}(a): on every Traffic LLM prior, EB drives $\tau^\star$ toward $\tau_{\min}$, recovering the no-prior baseline by construction.


\end{document}